\documentclass[twoside]{article}

\usepackage[accepted]{aistats2025}
\usepackage{hyperref}

\usepackage{natbib} % has a nice set of citation styles and commands
\usepackage{mathtools} % amsmath with fixes and additions
\usepackage{booktabs} % commands to create good-looking tables
\usepackage{tikz} % nice language for creating drawings and diagrams
\usepackage{algorithm}
\usepackage{algorithmic}
\usepackage{amsmath}
\usepackage{amssymb}
\usepackage{amsthm}
\usepackage{braket}
\usepackage[utf8]{inputenc}
\usepackage{colortbl}
\usepackage{enumitem}
\usepackage{authblk}
\usepackage{lipsum}

\usepackage{wrapfig}

\DeclareMathOperator*{\argmin}{arg\,min}
\DeclareMathOperator*{\R}{\mathbb{R}}

\DeclareMathOperator{\E}{\mathbb{E}}
\DeclareMathOperator*{\cO}{\mathcal{O}}

\DeclareMathOperator*{\cS}{\mathcal{S}}
\DeclareMathOperator*{\cA}{\mathcal{A}}

\DeclareMathOperator{\cV}{\mathcal{V}}
\DeclareMathOperator{\cW}{\mathcal{W}}

\DeclareMathOperator{\bias}{\beta_g}

\newtheorem{assumption}{Assumption}

\newtheorem{definition}{Definition}
\newtheorem{theorem}{Theorem}

\newtheorem{lemma}{Lemma}
\newtheorem{remark}{Remark}

\newcommand{\norm}[1]{\left \lVert #1 \right\rVert }
\newcommand{\sqnorm}[1]{\left \lVert #1 \right\rVert^2 }
\newcommand{\rb}[1]{\left ( #1 \right ) }
\newcommand{\cb}[1]{\left \{ #1 \right \} }

\begin{document}

% If your paper is accepted and the title of your paper is very long,
% the style will print as headings an error message. Use the following
% command to supply a shorter title of your paper so that it can be
% used as headings.
%
\runningtitle{Order-Optimal Regret with Novel Policy Gradient Approaches}

% If your paper is accepted and the number of authors is large, the
% style will print as headings an error message. Use the following
% command to supply a shorter version of the authors names so that
% they can be used as headings (for example, use only the surnames)
%
%\runningauthor{Surname 1, Surname 2, Surname 3, ...., Surname n}

\twocolumn[

\aistatstitle{Order-Optimal Regret with Novel Policy Gradient Approaches in Infinite-Horizon Average Reward MDPs}

\aistatsauthor{ Swetha Ganesh \And Washim Uddin Mondal \And  Vaneet Aggarwal }

\aistatsaddress{ Indian Institute of Science \\ Purdue University  \And  Indian Institute of Technology,\\ Kanpur \And Purdue University} ]

\begin{abstract}
      We present two Policy Gradient-based algorithms with general parametrization in the context of infinite-horizon average reward Markov Decision Process (MDP). The first one employs Implicit Gradient Transport for variance reduction, ensuring an expected regret of the order $\tilde{\mathcal{O}}(T^{2/3})$. The second approach, rooted in Hessian-based techniques, ensures an expected regret of the order $\tilde{\mathcal{O}}(\sqrt{T})$. These results significantly improve the state-of-the-art $\tilde{\mathcal{O}}(T^{3/4})$ regret and achieve the theoretical lower bound. We also show that the average-reward function is approximately $L$-smooth, a result that was previously assumed in earlier works.
   %   Additionally, using our Hessian-based analysis, we derive the first smoothness-type result for average-reward MDPs with general parametrization.
\end{abstract}

\section{INTRODUCTION}\label{sec:intro}

%\subsection{Overview}

Reinforcement Learning (RL) encompasses a set of challenges where a learner interacts iteratively with an unknown Markovian environment, aiming to maximize the total rewards earned. This framework finds application in various domains, such as networking, transportation, queueing theory, and epidemic control \citep{geng2020multi,al2019deeppool,agarwal2022concave,ling2023cooperating}. RL problems are commonly analyzed under three major frameworks: episodic, infinite-horizon with discounted rewards, and infinite-horizon with average rewards. Among these, the infinite-horizon average reward setup holds particular importance due to its ability to capture long-term goals in practical situations. To address these challenges, both model-free and model-based approaches have been proposed in the literature.

While model-based approaches have been extensively investigated for the average reward setup \citep{auer2008near,agrawal2017optimistic, wei2021learning}, they, unfortunately, require a large memory to store model parameters, limiting their practical applicability in large state space scenarios. Additionally, although model-free algorithms have been studied primarily in the tabular setups \citep{wei2020model}, they require all the policies to be represented in a tabular format, which creates a barrier in efficiently handling large state spaces. One way to tackle the issue of large state space is via general policy parametrization, where policies are indexed by a $\mathrm{d}\ll SA$ dimensional parameter where $S$, and $A$ denote sizes of the state and action spaces of the underlying Markov Decision Process (MDP). This approach has been explored recently for the average reward setup \citep{bai2023regret}. However, their proposed algorithm yields an expected regret of $\Tilde{\mathcal{O}}(T^{3/4})$, which is far from the theoretical lower bound of $\Omega(\sqrt{T})$. This prompts the following question:

\fbox{\begin{minipage}{23em}
{{\em Is it possible to improve the state-of-the-art regret performance and attain the lower bound of $\Tilde{\mathcal{O}}(\sqrt{T})$ in the average-reward setup with general policy parametrization?}}\end{minipage}}

In this paper, we affirmatively address the above question by presenting two algorithms. The first one is a Policy Gradient-based approach that employs implicit gradient transport for variance reduction. Notably, this algorithm does not require any second-order information and achieves an expected regret of $\Tilde{\mathcal{O}}(T^{2/3})$. The second algorithm uses a Hessian-based technique within the policy gradient framework to achieve an expected regret bound of $\Tilde{\mathcal{O}}(\sqrt{T})$. This bridges the gap between the upper and lower bounds of regret for the average reward setup in parametrized scenarios in terms of $T$.

% \begin{table*}[t]
%     \centering
%     \renewcommand{\arraystretch}{0}
%         \begin{tabular}{|c|c|c|c|c|}
%     	\hline
%     	\textbf{Algorithm} & \textbf{Regret}  & \textbf{Model-free} & \textbf{Setting}\\
%     	\hline
%             MDP-OOMD \citep{wei2020model} & $\Tilde{O}\bigg(\sqrt{T}\bigg)$  & Yes & Tabular\\
%     	\hline
%             MDP-EXP2 \citep{wei2021learning} & $\Tilde{O}\bigg(\sqrt{T}\bigg)$  & No & Linear MDP\\
%     	\hline
%             Parametrized Policy Gradient \citep{bai2023regret} & $\Tilde{O}\bigg(T^{3/4}\bigg)$  & Yes & General parametrization\\
%         \hline
%             \rowcolor{green!25} Algorithm \ref{alg:PG_IGT_Avg} (This Paper) & $\Tilde{O}\bigg(T^{2/3}\bigg)$  & Yes & General parametrization\\
%         \hline
%             \rowcolor{green!25} Algorithm \ref{alg:PG_Hessian_Avg} (This Paper) & $\Tilde{O}\bigg(\sqrt{T}\bigg)$  & Yes & General parametrization\\
%     	\hline\hline
%             \textbf{Lower bound} \citep{auer2008near} &       $\Omega\bigg(\sqrt{T}\bigg)$ & N/A & N/A\\
%     	\hline
%         \end{tabular}
%     \caption{This table provides a summary of various state-of-the-art algorithms, both model-based and model-free, for infinite-horizon average reward Markov Decision Processes (MDPs). Further discussion on related works can be found in Section \ref{sec:related-works}.}
%     \label{table1}
% \end{table*}

\subsection{Related Works}
\label{sec:related-works}

%\begin{wraptable}{r}{0.5\textwidth}
\begin{table}[ht]%{r}{0.5\textwidth}
	%\vspace{-.2in}
	\centering
	\resizebox{0.5\textwidth}{!}{
		\begin{tabular}{|c|c|c|c|}
			\hline
			\textbf{Algorithm} & \textbf{Regret}  & \textbf{Policy Parametrization}\\
			\hline
			MDP-OOMD \citep{wei2020model} & $\Tilde{O}\bigg(\sqrt{T}\bigg)$  & Tabular\\
			\hline
           POLITEX \citep{pmlr-v97-lazic19a} & $\Tilde{O}\bigg(T^{3/4}\bigg)$  & Tabular \\
			\hline
			MDP-EXP2 \citep{wei2021learning} & $\Tilde{O}\bigg(\sqrt{T}\bigg)$  & Tabular \\
			\hline
   GBPA \citep{murthy2023performance} & $\Tilde{O}\bigg(T^{2/3}\bigg)$  &  Tabular \\ \hline
   Parametrized Policy Gradient \citep{bai2023regret} & $\Tilde{O}\bigg(T^{3/4}\bigg)$  &  General \\
			\hline
			%\hline
   %MAC \citep{patel2024global} & $\Tilde{O}\bigg(T^{3/4}\bigg)$  & Yes & General \\
			\hline
			\rowcolor{green!25} Algorithm \ref{alg:PG_IGT_Avg} (This Paper) & $\Tilde{O}\bigg(T^{2/3}\bigg)$   & General \\
			\hline
			\rowcolor{green!25} Algorithm \ref{alg:PG_Hessian_Avg} (This Paper) & $\Tilde{O}\bigg(\sqrt{T}\bigg)$   & General \\
			\hline\hline
			\textbf{Lower bound} \citep{auer2008near} &       $\Omega\bigg(\sqrt{T}\bigg)$ & N/A\\
			\hline
		\end{tabular}
	}
	%\vspace{-.05in}
	\caption{This table summarizes policy-based algorithms from the literature for infinite-horizon average reward MDPs. In the tabular setting, the function approximation error for the policy is zero.}
	\label{table1}
	%\vspace{-.2in}
\end{table}

\textbf{Policy Gradient Algorithms in Discounted Reward MDPs:} Recent works on policy gradient-based algorithms have primarily focused on discounted reward MDPs. For instance, \citet{mondal2023improved} and \citet{fatkhullin2023stochastic} provide approaches achieving a sample complexity of $\Tilde{\mathcal{O}}(\epsilon^{-2})$ for general parameterized policies. More specifically, \citet{mondal2023improved} combines Accelerated Stochastic Gradient Descent with NPG to obtain the above-mentioned sample complexity bound. Whereas \citet{fatkhullin2023stochastic} uses (N)-HARPG, a recursive variance-reduction technique that uses Hessian estimates, to obtain their sample complexity bounds. Moreover, they also propose N-PG-IGT, which combines Policy Gradient (PG) with Implicit Gradient Transport (IGT), and show that this algorithm obtains a sample complexity of $\cO(\epsilon^{-2.5})$. Our approach is inspired by these two methods. However, our analysis deviates considerably for several reasons, including the notably different properties of the (biased) policy gradient and Hessian estimates, which emerge due to the lack of access to simulators to restart the MDP (more details in Section \ref{sec:technical-novelty}).

\textbf{Tabular Model-Based/Model-Free Average Reward MDPs:} 
One of the first bounds in this setting was established in \citet{auer2008near}, which introduced the model-based UCRL2 algorithm, achieving a regret of \(\tilde{\mathcal{O}}(\sqrt{T})\). However, model-based approaches become computationally challenging in large state spaces.  

In the realm of model-free algorithms, one approach is purely value-based, where the algorithm learns value functions and extracts the optimal policy only using the learned functions. This approach was adopted in \citet{zhang2023sharper}, achieving an optimal regret bound of \(\mathcal{O}(\sqrt{T})\) for weakly-communicating MDPs. However, such methods often struggle with scalability in large state-action spaces, as they require maximizing the \(Q\)-function over the entire action space at each iteration. In contrast, policy gradient methods are more suitable for large state-action spaces and are widely used in practice.  

Another class of methods updates policies using a learned value function, avoiding action-space maximization at each step \citep{wei2020model,wei2022provably,pmlr-v97-lazic19a,murthy2023performance}. However, these methods assume tabular policy parameterization. More specifically, the transferred function approximation error in \eqref{eq:transfer_error}, and consequently, \(\epsilon_{\mathrm{bias}}\) in Assumption \ref{assump:function_approx_error}, is zero. \citep{wei2020model} introduced an online mirror descent algorithm achieving \(\tilde{\mathcal{O}}(\sqrt{T})\) regret for ergodic MDPs. They also showed that optimistic-Q learning attains \(\tilde{\mathcal{O}}(T^{2/3})\) regret in weakly communicating MDPs, though their approach relies on tabular value function approximation. This assumption was later relaxed to linear value function approximation \citep{pmlr-v97-lazic19a,wei2021learning}. In \citet{murthy2023performance}, policy parameterizations beyond softmax were explored, achieving \(\mathcal{O}(T^{2/3})\) regret. However, their approach still requires tabular policies and value functions.  

In contrast, our paper develops efficient algorithms for MDPs with general parameterized \textit{policies}. Notably, analyzing the general parameterization case is not a straightforward extension of the tabular setting with an added function approximation error; instead, it introduces fundamental differences. One key distinction stems from the additional monotonicity property that ensures per-step improvements in the tabular case, a property that does not naturally extend to the function approximation setting \citep{agarwal2021theory}.

{\bf Parametrized Policies in Average Reward Setup:} It is worth stating that the policy gradient-based approaches with general parametrization are scalable and useful in multiple applications. Although most of the algorithms with general parametrization focus on the discounted reward setup, our work aims to solve the average reward MDPs. The first regret analysis within this framework was undertaken in \citep{bai2023regret}. Using a policy-gradient-based approach, they established a regret bound of $\Tilde{\mathcal{O}}(T^{3/4})$. In another recent study within a similar framework, a global convergence bound of $\Tilde{\mathcal{O}}(T^{-1/4})$ was achieved employing a Multi-level Monte Carlo approach \citep{patel2024global}. However, these works assume that the average reward function, $J$, is $L$-smooth. In this paper, we enhance these findings by presenting two algorithms: one utilizing second-order information and one without. Our results demonstrate improvements over prior work, with the second-order information algorithm achieving the optimal regret order. Table \ref{table1} summarizes the key results and comparisons.

\subsection{Technical Novelty and Contributions}
\label{sec:technical-novelty}

%Variance-reduction methods have been studied extensively within the framework of discounted rewards \citep{shen19d,xu2019sample,liu2020improved, fatkhullin2023stochastic}. However, naively adapting these approaches to the average reward infinite setup presents additional challenges. For instance, in the context of discounted MDPs with general parametrization, it is known that the long term average reward, $J$, is $L$-smooth, with $L=\cO((1-\gamma)^{-2})$ \citep{liu2020improved}. However, this result does not generalize well to average-reward MDPs which corresponds to the case where $\gamma \rightarrow 1$. As a result, previous works \citep{bai2023achieving,patel2024global} assume $J$ to be $L$-smooth.  {\bf \color{red} The last line is unclear - this seems deficiency of the works rather than positive. Are we also assuming the same - or relaxing that? If not, we need to put this in positive light that it is independent of $\gamma$ and how resolved. }

{\bf Difference from Discounted Setup Results: } Infinite-horizon discounted MDPs with general policy parametrization have been extensively studied in the literature \citep{mondal2023improved, fatkhullin2023stochastic, liu2020improved}. However, despite being more aligned toward real-world scenarios, the average-reward setting is less investigated due to several unique challenges. For example, in average-reward MDPs, crucial quantities such as the value function, advantage function, and policy gradient estimator can become unbounded. In contrast, in the discounted case, these quantities are bounded, with the bounds scaling with powers of \(\frac{1}{1-\gamma}\) (where \(\gamma\) is the discount factor, which equals 1 in the average reward setup). Additionally, most works in the discounted setup assume access to a generative model (simulator), which allows them to restart the state distribution at will, a practice we avoid due to its impracticality. Such avoidance also makes the biases of the policy gradient and Hessian estimators non-zero, further complicating the analysis.

%{\bf I am not able to write about exchange of max and limit part - I understand it. But after writing it felt weak - maybe because it does not say why average reward is hard, only that discounted results cannot be translated. However unboundedness is a challenge that averagw reward has and dicounted doesnt.}gamma $\to 1$ does not work . 2. Assume restart/no Markovian. 

{\bf Technical Novelty: } We note that this work improves the state of the art for average reward MDP, by proposing two algorithms and analyzing their regret with the general policy parametrization. The key technical novelties in this work are summarized as follows.

\begin{itemize}[left=5pt, labelwidth=5pt, labelsep=5pt, itemindent=2\parindent, itemsep=0pt, parsep=0pt, topsep=0pt, partopsep=0pt]

\item {\bf Showing the smoothness property of average-reward function: } It is well-known that the average reward function, \( J \), is \( L \)-smooth in the discounted MDP \citep{liu2020improved}. This implies that the Hessian of \( J \) is bounded by \( L \). The proof is straightforward in this context, as the Hessian can be calculated using the policy gradient estimator and its derivative, both of which are bounded. However, as stated before, this does not hold in the average reward case, and prior works have treated this as an unverifiable assumption \citep{bai2023regret}. We address this issue by proving an approximate \( L \)-smoothness result by identifying a function \( \Bar{J} \) that is \( L \)-smooth and works as a close approximation to \( J \) (see Section \ref{subsec:aux-funcs}). We note that while \citet{kumar2024global} provides a smoothness result for average reward MDPs, it considers the tabular setup which uses a substantially different analysis. They show smoothness of $J$ with respect to the policy itself, while we show smoothness with respect to the policy parameter $\theta$. As a result, their analysis crucially relies on the linearity of the policy, reward and transition kernel $P$ with respect to the parameter $\alpha$ (see (13) and (14) of \citet{kumar2024global}).

\item {\bf Efficient Bounds on Gradient Estimator: } In prior research on policy gradient methods in the average-reward parameterized framework, it was observed that even with the exact advantage function \( A^{\pi_{\theta}}(s, a) \), the trajectory sampling distribution introduced an error term proportional to \( \frac{1}{T^{1/4}} \). This error would persist despite employing variance-reduction techniques, thereby hindering the achievement of \( \cO(\sqrt{T}) \) regret. We address this issue by adopting a simple approach that estimates the advantage and the policy gradient at a slightly later point along the trajectory, effectively diminishing this error term without impairing the convergence rate and, consequently, the regret.

\item {\bf Efficient Variance Reduction Analysis in the Average Reward Setup: } In the analysis of variance-reduced methods for discounted settings, the gradient and Hessian estimators are typically unbiased. However, this does not hold in the average reward case. Note that bias accumulation can hinder convergence, so one ought to minimize it for faster convergence. Unfortunately, minimizing the bias of the gradient and Hessian estimators is a difficult endeavor. The approximate function, \(\Bar{J}\) (stated above), helps us in this regard by allowing us to construct its unbiased gradient and Hessian estimators (see Section \ref{subsec:aux-funcs}) which are close to their actual values.

%\item {\bf Bias Guarantees of Hessian estimator: This is contained in the previous point}

\end{itemize}

{\bf Key Contributions}: We provide two Policy Gradient-based approaches with general parametrization and provide regret guarantees under the assumption of an ergodic MDP. More specifically,
\begin{itemize}[left=5pt, labelwidth=5pt, labelsep=5pt, itemindent=2\parindent, itemsep=0pt, parsep=0pt, topsep=0pt, partopsep=0pt]
\item We show that our first approach (Algorithm \ref{alg:PG_IGT_Avg}) has an expected regret of $\tilde{\mathcal{O}}(T^{2/3})$. This method utilizes implicit gradient transport for variance reduction without requiring importance sampling or curvature information (such as Hessian estimates). Moreover, this algorithm only samples a single trajectory per iteration.

\item Our second approach (Algorithm \ref{alg:PG_Hessian_Avg}) uses a Hessian-based approach to obtain an expected regret of $\tilde{\mathcal{O}}(\sqrt{T})$, which is optimal in $T$. Although this algorithm uses Hessian estimates, the required updates can be computed with memory and computational complexities that are similar to that of Hessian-free methods. 

\item Additionally, we provide the first approximate smoothness-type result for the average-reward function using our novel Hessian-based analysis. 
\end{itemize}

For both Algorithms \ref{alg:PG_IGT_Avg} and \ref{alg:PG_Hessian_Avg}, the obtained regret bounds significantly improve the existing state-of-the-art $\Tilde{\cO}(T^{3/4})$ regret of \citet{bai2023regret}.

\section{SETUP}
\label{sec:setup}

%{\bf \color{red} Shorten section moving parts to Appendix. }
In this paper, we explore an infinite-horizon reinforcement learning problem with an average reward criterion, modeled by a Markov Decision Process (MDP) represented as a tuple $\mathcal{M}=(\mathcal{S},\mathcal{A},r, P,\rho)$. Here, $\mathcal{S}$ denotes the state space with size $S$, $\mathcal{A}$ is the action space with a size of $A$, $r:\mathcal{S}\times\mathcal{A}\rightarrow [0,1]$ represents the reward function, $ P:\mathcal{S}\times\mathcal{A}\rightarrow \Delta(\mathcal{S})$ defines the state transition function, where $\Delta(\mathcal{S})$ denotes the probability simplex over $\mathcal{S}$, and $\rho\in \Delta(\mathcal{S})$ signifies the initial distribution of states. A policy $\pi:\mathcal{S}\rightarrow \Delta(\mathcal{A})$ determines the distribution of the action to be taken given the current state. It gives rise to a transition function $P^{\pi}:\mathcal{S}\rightarrow \Delta(\mathcal{S})$ defined as $P^{\pi}(s, s') = \sum_{a\in\mathcal{A}}P(s'|s,a)\pi(a|s)$, for all $s, s'\in\mathcal{S}$. It can be seen that for any given policy $\pi$, the sequence of states produced by the MDP forms a Markov chain. We will assume the following throughout the paper.

\begin{assumption}
    \label{assump:ergodic_mdp}
    The MDP $\mathcal{M}$ is ergodic. In other words, the Markov chain induced by every policy $\pi$ is irreducible and aperiodic.
\end{assumption}

The assumption of ergodicity is frequently employed in the analysis of Markov Decision Processes (MDPs) \citep{pesquerel2022imed, gong2020duality, bai2023regret}. 

We consider a parameterized class of policies $\Pi$, which consists of all policies $\pi_{\theta}$ such that $\theta \in \Theta$, where $\Theta = \R^d$. It is a well-known fact that if $\mathcal{M}$ is ergodic, then for all $\theta\in\Theta$, there exists a unique stationary distribution denoted as $d^{\pi_{\theta}}\in \Delta(\mathcal{S})$, that is independent of the initial distribution $\rho$ and satisfies $(P^{\pi_{\theta}})^\top d^{\pi_{\theta}}=d^{\pi_{\theta}}$. We define the long-term average reward for a given policy $\pi_{\theta}$ as follows:
\begin{align*}
    J_{\rho}^{\pi_{\theta}}&\coloneqq \lim\limits_{T\rightarrow \infty}\frac{1}{T}\E\left[\sum_{t=0}^{T-1}r(s_t,a_t)\bigg|s_0\sim \rho\right] \\
    &= \sum_{s\in \mathcal{S}}\sum_{a\in\mathcal{A}} r(s, a)d^{\pi_\theta}(s)\pi_{\theta}(a|s),
\end{align*}
where the expectation is computed over all state-action trajectories generated by following the action execution process $a_t\sim\pi(\cdot|s_t)$ and the state transition rule $s_{t+1}\sim P(\cdot|s_t, a_t)$ for all $t\in \{0, 1, \cdots\}$. Similar to the stationary distribution, $J_{\rho}^{\pi_{\theta}}$ is also independent of $\rho$. We simplify the notation of $J^{\pi_\theta}$ by expressing it as $J(\theta)$. Our target is to maximize $J(\cdot)$ without any knowledge about the transition function $P$. The maximum value and the corresponding maximizing policy are denoted as $J^*$ and $\pi^*$ respectively.

% The long-term average reward can also be expressed as follows:
% \begin{align}
%     \label{eq_r_pi_theta}
%     \begin{split}
%         &J(\theta) = \E_{s\sim d^{\pi_{\theta}}, a\sim \pi_{\theta}(\cdot|s)}[r(s, a)] = (d^{\pi_{\theta}})^\top  r^{\pi_{\theta}} \text{, where }r^{\pi_{\theta}}(s) \coloneqq \sum_{a\in\mathcal{A}}r(s, a)\pi_{\theta}(a|s), ~\forall s\in \mathcal{S}.
%     \end{split}
% \end{align}

% With this notation in place, our objective can be stated as solving:
% \begin{equation}
%     \argmax_{\theta\in\Theta} J(\theta).
% \end{equation}
% 

%To tackle this optimization problem, we adopt the policy gradient approach, where we update the policy parameter, $\theta$, along the gradient direction, $\nabla_{\theta} J(\theta)$. However, in practical scenarios, obtaining the exact gradient is infeasible\footnote{This is because, as shown in \eqref{eq:grad-expression}, evaluating the exact gradient requires exact knowledge of $d^{\pi_\theta}$ which, in turn, requires an exact knowledge of the transition function, $P$ which is unavailable in most RL applications.}, thus requiring estimation. 

We introduce a few notations before delving into our proposed algorithms to solve this optimization. The state value function $V$ is defined below.
% \begin{equation}
%     \label{eq_bellman}
%     Q^{\pi_{\theta}}(s,a)=r(s,a)-J(\theta)+\E_{s'\sim P(\cdot|s, a)}\left[V^{\pi_{\theta}}(s')\right],
% \end{equation}
% where the state value function, $V^{\pi_{\theta}}:\mathcal{S}\rightarrow \mathbb{R}$ is defined as,
% \begin{align}
%     \label{eq_V_Q}
%     V^{\pi_{\theta}}(s) = \sum_{a\in\mathcal{A}}\pi_{\theta} (a|s)Q^{\pi_{\theta}}(s, a), ~\forall s\in\mathcal{S}.
% \end{align}
% Observe that if $(\ref{eq_bellman})$ is satisfied by $Q^{\pi_{\theta}}$, then it is also satisfied by $Q^{\pi{\theta}}+c$ for any arbitrary constant $c$. To uniquely define these functions, we assume that $\sum_{s\in\mathcal{S}}d^{\pi_{\theta}}(s)V^{\pi_{\theta}}(s)=0$. In this case, $V^{\pi_{\theta}}(s)$ can be expressed as follows for all $s\in\mathcal{S}$:
\begin{align}
    \label{def_v_pi_theta_s}
    \begin{split}
        V^{\pi_{\theta}}(s) =\E_{\theta}\left[\sum_{t=0}^{\infty}(r(s_t,a_t)-J(\theta))\bigg\vert s_0=s\right],
    \end{split}
\end{align}
where $r^{\pi_\theta}(s) \coloneqq \sum_{a}r(s, a)\pi_{\theta}(a|s)$, $\forall s\in \mathcal{S}$, $\forall \theta\in \Theta$, and $\E_{\theta}[\cdot]$ is expectation over all trajectories induced by the policy $\pi_{\theta}$. Similarly, $\forall (s, a)\in\mathcal{S}\times \mathcal{A}$, the state-action value function $Q^{\pi_{\theta}}(s, a)$ can be written as shown below.
\begin{equation}
    Q^{\pi_{\theta}}(s,a)=\E_{\theta}\Bigg[\sum_{t=0}^{\infty}(r(s_t,a_t)-J(\theta))\bigg\vert s_0=s,a_0=a\Bigg].
\end{equation}
 We can now define the advantage function $A^{\pi_{\theta}}:\mathcal{S}\times \mathcal{A}\rightarrow \mathbb{R}$ such that $\forall (s, a)\in\mathcal{S}\times\mathcal{A}$, we have,
\begin{align}
    A^{\pi_{\theta}}(s, a) \coloneqq Q^{\pi_{\theta}}(s, a) - V^{\pi_{\theta}}(s).
\end{align}

%\begin{minipage}{0.5\textwidth}
\begin{algorithm}[ht]
    \caption{Parameterized Policy Gradient with Implicit Gradient Transport}
    \label{alg:PG_IGT_Avg}
    \begin{algorithmic}[1]
        \STATE \textbf{Input:} Initial parameters $\theta_0$ and $\theta_1$, stepsizes $\{\gamma_k\}_{k \geq 1}$, $\{\eta_k\}_{k \geq 1}$, $s_0 \sim \rho(\cdot)$, epoch length $H$, epoch number $K$, $N=7t_{\mathrm{mix}}\log_2T$ \vspace{0.1cm}
        
	\FOR{$k\in\{1, \cdots, K\}$}
            \STATE {$\Tilde{\theta}_k \gets \theta_k + \frac{1-\eta_k}{\eta_k}(\theta_k - \theta_{k-1})$}, ~ $\tilde{\tau}_k\gets \phi$
            
            \FOR{$t\in\{(k-1)H, \cdots, kH-1\}$}
                \STATE Execute $a_t\sim \pi_{\Tilde{\theta}_k}(\cdot|s_t)$, receive reward $r(s_t,a_t) $ and observe $s_{t+1}$
                \STATE $\tilde{\tau}_k\gets \tilde{\tau}_k\cup \{(s_t, a_t)\}$
            \ENDFOR	
            \FOR{$t\in\{(k-1)H+N, \cdots, kH-1\}$}
                \STATE Obtain $\hat{A}^{\pi_{\Tilde{\theta}_k}}(\tilde{\tau}_k, s_t, a_t)$ via Algorithm \ref{alg:estA}
            \ENDFOR
            \vspace{0.1cm}
      
            \STATE {Obtain $g(\tilde{\theta}_k,\tilde{\tau}_k)$ via \eqref{eq:grad_estimate} 
            and $d_k$ via \eqref{eq:dk-IGT}}
           
		\STATE Update policy parameter as
		\begin{equation}
                \label{algorithm-update-IGT}
		    \theta_{k+1}=\theta_k+\gamma_k \frac{d_k}{\|d_k\|}
		\end{equation}
        \ENDFOR
    \end{algorithmic}
\end{algorithm}
%\end{minipage}
% where $(a)$ follows from the fact that $\nabla_{\theta} \log \pi_{\theta} (a|s) = \nabla_{\theta} \pi_{\theta} (a|s) / \pi_{\theta} (a|s)$. Notice that if 
% $b(s)$ is any function of 
% $s$ (commonly called as a \textit{baseline}), then 
% \begin{align}
%    \sum_{s \in \cS} d^{\pi_{\theta}}(s)\sum_{a \in \cA} b(s) \nabla_{\theta} \pi_{\theta}(a|s) = \sum_{s \in \cS} d^{\pi_{\theta}}(s)b(s)\nabla_{\theta} \left(\sum_{a \in \cA}   \pi_{\theta}(a|s)\right) =  \sum_{s \in \cS} d^{\pi_{\theta}}(s)b(s)\nabla_{\theta} 1 = 0.
% \end{align}
% Thus, \eqref{eq:grad-expression} still holds when $Q(s,a)$ is replaced with the advantage function $A(s,a) = Q(s,a)-V(s)$. 
It is known that ergodicity also implies the existence of a finite mixing time. Specifically, if $\mathcal{M}$ is ergodic, then the mixing time can be defined as follows.
\begin{definition}
The mixing time of an MDP $\mathcal{M}$ with respect to a  parameter $\theta$ is given as, $t_{\mathrm{mix}}^{\theta}\coloneqq \min\left\lbrace t\geq 1\bigg| \|(P^{\pi_{\theta}})^t(s, \cdot) - d^{\pi_\theta}\|\leq \dfrac{1}{4}, \forall s\in\mathcal{S}\right\rbrace$. We also define $t_{\mathrm{mix}}\coloneqq \sup_{\theta\in\Theta} t^{\theta}_{\mathrm{mix}} $ as the overall mixing time which is finite due to ergodicity.
\end{definition}

The mixing time of an MDP measures how quickly the MDP approaches its stationary distribution when the same policy is executed repeatedly. Additionally, we define the hitting time below.
\begin{definition}
The hitting time of an MDP $\mathcal{M}$ with respect to a policy parameter, $\theta$ is defined as, $t_{\mathrm{hit}}^{\theta}\coloneqq  \max_{s\in\mathcal{S}} \frac{1}{d^{\pi_{\theta}}(s)}$. We also define $t_{\mathrm{hit}}\coloneqq \sup_{\theta\in\Theta} t^{\theta}_{\mathrm{hit}} $ as the overall hitting time which is finite due to ergodicity.
\end{definition}
For a given MDP $\mathcal{M}$ and a time horizon $T$, the regret of an algorithm $\mathbb{A}$ is defined as follows.
\begin{align}
    \mathrm{Reg}_T(\mathbb{A}, \mathcal{M}) \coloneqq \sum_{t=0}^{T-1} \left(J^*-r(s_t, a_t)\right),
\end{align}
where $J^*$ denotes the optimal long-term average reward and $\{a_t\}_{t \geq 0}$, are selected by the algorithm, $\mathbb{A}$, based on the history up to time, $t$. The state at the next time step is obtained by following the state transition function, $P$. We simplify the notation of regret to $\mathrm{Reg}_{T}$ wherever there is no ambiguity. We are now ready to describe our proposed algorithms. %It can be seen that $J(\cdot)$ can be maximized by designing an algorithm that minimizes the regret.

%%===========================Algorithm==============================%%

%\hfill
%\begin{minipage}{0.47\textwidth}
    \begin{algorithm}[ht]
    \caption{State and state-action value function estimation}
    \label{alg:estA}
    \begin{algorithmic}[1]
        \STATE \textbf{Input:} State $s$, action $a$, policy parameter $\theta$, and trajectory $\tau = \{(s_{t}, a_{t})\}_{t=t_s}^{t_e}$
        \STATE \textbf{Define:} $N=7t_{\mathrm{mix}}\log_2T$
        \STATE \textbf{Initialize:} $i \leftarrow 0$, $\xi\leftarrow t_s$
	\vspace{0.1cm}
	\WHILE{$\xi\leq t_e-N$}
		\IF{$s_{\xi}=s$}
			\STATE $i\leftarrow i+1$,  
            \STATE $\xi_i\gets \xi$
			\STATE $y_i\gets\sum_{t=\xi}^{\xi+N-1}r(s_t, a_t)$.
			\STATE $\xi\leftarrow\xi+2N$.	
		\ELSE
                \STATE {$\xi\leftarrow\xi+1$.}
            \ENDIF
	\ENDWHILE
        \vspace{0.1cm}
        \IF{$i>0$}
            \STATE Obtain $\hat{V}^{\pi_\theta}(\tau, s)$ and $\hat{Q}^{\pi_\theta}(\tau, s, a)$ using Eq. \eqref{eq:V_estimates} and \eqref{eq:Q_estimates}
        \ELSE
            \STATE $\hat{V}^{\pi_\theta}(\tau, s)\gets 0$, ~$\hat{Q}^{\pi_\theta}(\tau, s, a) \gets 0$
        \ENDIF
	\STATE \textbf{return}  $\hat{Q}^{\pi_\theta}(\tau, s, a)$ and $\hat{V}^{\pi_\theta}(\tau, s)$ 
    \end{algorithmic}
\end{algorithm}
%\end{minipage}

\section{PROPOSED ALGORITHMS}
\label{sec:proposed-algo}

This paper proposes two algorithms. The first one (Algorithm \ref{alg:PG_IGT_Avg}) is a PG method that utilizes Implicit Gradient Transport (IGT) for variance reduction. The second one (Algorithm \ref{alg:PG_Hessian_Avg}) applies a Hessian-based approach for variance-reduction. Given a horizon of length $T$, both algorithms run $K=T/H$ number of epochs where $H$ is the duration of each epoch. At the $k$th epoch, both algorithms update the current estimate of the policy parameter, $\theta_k$, along some direction, $d_k$, which is computed as a function of the previous direction, $d_{k-1}$, and the gradient and Hessian estimates of the function $J(\cdot)$ at some appropriate parameter values. Note that the recursive definition of the update direction, $d_k$, is the typical characteristic of a momentum-based algorithm, which contrasts with that of a simple policy gradient algorithm. Below we describe the process to compute gradient and Hessian estimates of $J(\cdot)$ at an arbitrary parameter, $\theta$.

{\bf Gradient Estimation:} The idea of estimating $\nabla_{\theta} J(\theta)$ is mimicking the following result \citep{sutton1999policy}.
\begin{align}
\label{eq:grad-expression}
    \nabla_{\theta} J(\theta) = \E_{s \sim d^{\pi_{\theta}},a \sim \pi_{\theta}(\cdot|s)}[A^{\pi_{\theta}}(s,a) \nabla_{\theta} \log \pi_{\theta}(a|s)].
\end{align}
In particular, if $\tau = \{(s_t, a_t)\}_{t=t_s}^{t_e}$ is a $\pi_\theta$-induced trajectory, and $\hat{A}^{\pi_{\theta}}(\tau, s, a)$ is an estimate of the advantage value $A^{\pi_\theta}(s, a)$ corresponding to $\tau$, $\forall (s, a)$, then the estimate of $\nabla_{\theta}J(\theta)$ corresponding to $\tau$ is the following expression where $|\tau|=t_e-t_s+1$ is the length of $\tau$ and $N=\cO(t_{\mathrm{mix}}\log T)$.
\begin{align}
    \label{eq:grad_estimate}
    g(\theta,\tau) = \dfrac{1}{|\tau|-N}\sum_{t=t_s+N}^{t_e}\hat{A}^{\pi_{\theta}}(\tau, s_{t}, a_{t})\nabla_{\theta}\log \pi_{\theta}(a_{t}|s_{t}),
\end{align}
Note that the first $N$ elements of $\tau$ are discarded to compute the gradient estimate. This is to ensure that the remaining trajectory's initial state distribution is close to the stationary one. It also guarantees that the estimates obtained at different epochs are fairly independent of each other. Algorithm \ref{alg:estA} gives a procedure to calculate the estimate $\hat{A}^{\pi_\theta}(\tau, s, a)$. In particular, for a $\pi_\theta$-induced trajectory $\tau=\{(s_t, a_t)\}_{t=t_s}^{t_{e}}$, Algorithm \ref{alg:estA} first locates its disjoint sub-trajectories of length $N$ that start with a given state $s$ and are at least $N$ step apart from each other. This guarantees that the estimates from each sub-trajectory are fairly independent. If $i$ indicates the number of such sub-trajectories, and the sum of rewards and the initial time of the $j$th such sub-trajectory are $y_j$ and $\xi_j$, respectively, then the estimates $\hat{V}^{\pi_{\theta}}(\tau, s)$ and $\hat{Q}^{\pi_{\theta}}(\tau, s, a)$ are given as follows for $i>0$. 
\begin{align}
\label{eq:V_estimates}
    \hat{V}^{\pi_{\theta}}(\tau, s) = \frac{1}{i}\sum_{j=1}^i y_j  
\end{align}
and
\begin{align}
\label{eq:Q_estimates}
\hat{Q}^{\pi_{\theta}}(\tau, s, a) = \frac{1}{i\pi_{\theta}(a|s)}\sum_{j=1}^i y_j \mathbf{1}(a_{\xi_j}=a),
\end{align}
where $\mathbf{1}(\cdot)$ is an indicator function. For $i=0$, each estimate is taken to be zero. These expressions help obtain the advantage estimate as $\hat{A}^{\pi_{\theta}}(\tau, s, a)= \hat{Q}^{\pi_{\theta}}(\tau, s, a)-\hat{V}^{\pi_{\theta}}(\tau, s)$, $\forall (s, a)$.

{\bf Hessian Estimation:} Let $\tau=\{(s_t, a_t)\}_{t=t_s}^{t_e}$ be a $\pi_\theta$-induced trajectory. Define the following. 
\begin{align}
\label{eq:def_Phi}
    &\Phi(\theta,\tau)\nonumber\\ &\coloneqq \frac{1}{|\tau|-N}\sum_{t=t_s+N}^{t_e} \left\lbrace \Psi_t^{(1)}(\tau) \log \pi_{\theta}(a_t | s_t) + \frac{\Psi_t^{(2)}(\tau)}{\pi_{\theta}(a_t | s_t)}\right\rbrace,
\end{align}
where $|\tau| = t_e-t_s+1$ is the length of the trajectory, $\tau$,
\begin{align*}
   \Psi_t^{(1)}(\tau) \coloneqq -\hat{V}^{\pi_\theta}(\tau, s_t) 
\end{align*}
and
\begin{align*}
    \Psi_t^{(2)}(\tau) \coloneqq -\hat{Q}^{\pi_\theta}(\tau, s_t, a_t)\pi_{\theta}(a_t|s_t)
\end{align*}
where $\hat{V}^{\pi_\theta}(\tau, \cdot)$ and $\hat{Q}^{\pi_\theta}(\tau, \cdot, \cdot)$ are defined by \eqref{eq:V_estimates} and \eqref{eq:Q_estimates}. It is easy to verify that $\Psi_t^{k}(\tau)$, $k\in\{1, 2\}$ depend only on $(\tau, t)$ and not on $\pi_\theta$.

Let $p(\tau, \theta, \bar{\rho}) \coloneq \bar{\rho}(s_{t_s})\prod_{h=t_s+1}^{t_e} \pi_{\theta}(a_h|s_h) P(s_{h+1}|s_h,a_h)$. $p(\tau, \theta, \bar{\rho})$ is the probability of occurrence of $\tau$ that initiates with a $\theta$-independent state distribution $\bar{\rho}$.  We estimate the Hessian as follows:
\begin{align}
\label{eq:Hessian-estimate}
    B(\theta,\tau) \coloneqq \nabla_{\theta}  \Phi (\theta,\tau)\nabla_{\theta}\log p(\tau, \theta, \bar{\rho})^\top  + \nabla_{\theta}^2 \Phi (\theta,\tau).
\end{align}
It is to be emphasized that although $p(\tau, \theta, \bar{\rho})$ requires knowledge of $P$ and $\bar{\rho}$ to be calculated, $\nabla_{\theta} \log p(\tau, \theta, \bar{\rho})$ does not. One can establish (Appendix \ref{subsec:Hessian-details}) that $B(\theta,\tau)$ is an unbiased estimator of $\nabla_{\theta}^2 \Bar{J}(\theta)$ where $\Bar{J}$ is an auxiliary function defined in section \ref{subsec:aux-funcs} that closely approximates $J$. Appendix \ref{subsec:Hessian-details} also exhibits that $\nabla_{\theta} \Phi(\theta,\tau) = g(\theta,\tau)$ and derives an expression to compute $\nabla_\theta^2 \Phi(\theta, \tau)$. These results enable the computation of $B(\theta, \tau)$.

\textbf{Parameterized Policy Gradient with Implicit Gradient Transport (Algorithm \ref{alg:PG_IGT_Avg}):}
The idea of Implicit Gradient Transport was originally proposed in \citet{MomentumImprovesNSGD_Cutkosky_2020} in the context of Stochastic Optimization.  It was later used in discounted RL \citet{fatkhullin2023stochastic}.  Algorithm \ref{alg:PG_IGT_Avg} obtains an auxiliary parameter $\tilde{\theta}_k$, at each epoch $k$, using the current and previous policy parameters $\theta_k, \theta_{k-1}$ as follows.
\begin{align}
    \Tilde{\theta}_k = \theta_k + \frac{1-\eta_k}{\eta_k}(\theta_k - \theta_{k-1})
\end{align}
where $\eta_k$ is an appropriately chosen learning parameter. Next, a $\pi_{\tilde{\theta}_k}$-induced trajectory, $\tilde{\tau}_k$, of length $H$ is obtained. The update direction, $d_k$, is then recursively computed using the following relation.
\begin{align}
\label{eq:dk-IGT}
    d_k = (1-\eta_k) d_{k-1} + \eta_k g(\tilde{\theta}_k,\tilde{\tau}_k) 
\end{align}
where $g(\tilde{\theta}_k, \tilde{\tau}_k)$ is the estimate of the policy gradient given by \eqref{eq:grad_estimate}. Finally, $\theta_k$ is updated following \eqref{algorithm-update-IGT} where $\gamma_k$ is an appropriate learning rate.

The following remarks are worth mentioning. Firstly, $d_k$'s are computed recursively to capture the momentum. Secondly, unlike other momentum-based algorithms, the update equation of $d_k$ involves a gradient estimate computed at the auxiliary parameter $\tilde{\theta}_k$, rather than it being computed at ${\theta_k}$. Note that $\tilde{\theta}_k$ is defined such that it can be interpreted as taking a ``look-ahead" extrapolation from iterates $\theta_k$ and $\theta_{k-1}$ (further insights into this approach can be found in \citep[Sec.~3]{MomentumImprovesNSGD_Cutkosky_2020}). Finally, the update of $\theta_k$ is proportional to the normalized vector $d_k/\Vert d_k\Vert$ rather than it being proportional to $d_k$ itself.

%{\bf \color{red} Unclear what is novel in Algorithm - I am sure there is a lot and cannot be simple combination - while not coming out. }

\begin{algorithm}[h]
    \caption{Parameterized Hessian-aided Policy Gradient}
    \label{alg:PG_Hessian_Avg}
    \begin{algorithmic}[1]
        \STATE \textbf{Input:} Initial parameters $\theta_0$ and $\theta_1$, stepsizes $\{\gamma_k\}_{k \geq 1}$, momentum parameters $\{\eta_k\}_{k \geq 1}$,  initial state $s_0 \sim \rho(\cdot)$, epoch length $H$, number of epochs $K$, $N=7t_{\mathrm{mix}}\log_2T$, $d_0=\mathbf{0}$ \vspace{0.1cm}
        
	\FOR{$k\in\{1, \cdots, K\}$}
            \STATE {$q_k \sim \mathcal{U}([0,1])$},~ {$\Hat{\theta}_k = q_k \theta_k + (1-q_k)\theta_{k-1}$}
            \STATE $\tau_k\gets \phi$, $\Hat{\tau}_k\gets \phi$ 
            
            \FOR{$t\in\{(k-1)H, \cdots, (k-1)H+\frac{H}{2}-1\}$}
                \STATE Execute $a_t\sim \pi_{\theta_k}(\cdot|s_t)$, receive reward $r(s_t,a_t) $ and observe $s_{t+1}$
                \STATE $\tau_k\gets \tau_k\cup \{(s_t, a_t)\}$
            \ENDFOR	

            \FOR{$t\in\{(k-1)H+\frac{H}{2}, \cdots, kH-1\}$}
                \STATE Execute $a_t\sim \pi_{\Hat{\theta}_k}(\cdot|s_t)$, receive reward $r(s_t,a_t) $ and observe $s_{t+1}$
                \STATE $\Hat{\tau}_k\gets \Hat{\tau}_k\cup \{(s_t, a_t)\}$
            \ENDFOR	
            
            \FOR{$t\in\{(k-1)H+N, \cdots, (k-1)H+\frac{H}{2}-1\}$}
                \STATE Using Algorithm \ref{alg:estA}, and $\tau_k$, compute $\hat{A}^{\pi_{\theta_k}}(\tau_k, s_t, a_t)$
            \ENDFOR
            \FOR{$t\in\{(k-1)H+\frac{H}{2}+N, \cdots, kH-1\}$}
                \STATE Using Algorithm \ref{alg:estA}, and $\Hat{\tau}_k$, compute $\hat{Q}^{\pi_{\Hat{\theta}_k}}(\hat{\tau}_k, s_t, a_t)$ and $\hat{V}^{\pi_{\Hat{\theta}_k}}(\hat{\tau}_k, s_t)$
            \ENDFOR
            \vspace{0.1cm}
      
            \STATE Compute $g(\theta_k,\tau_k)$ as in \eqref{eq:grad_estimate}, $v(\hat{\theta}_k,\hat{\tau}_k)$ as in \eqref{eq:vk} and $d_k$ using \eqref{eq:algo_update_d_k_Hessian}
		\STATE Update policy parameter as
		\begin{equation}
                \label{algorithm-update-Hessian}
		    \theta_{k+1}=\theta_k+\gamma_k \frac{d_k}{\|d_k\|}
		\end{equation}
        \ENDFOR
    \end{algorithmic}
\end{algorithm}

\textbf{Parameterized Hessian-aided Policy Gradient (Algorithm \ref{alg:PG_Hessian_Avg}):}
This approach has been explored in the context of Reinforcement Learning (RL) with an infinite-horizon discounted rewards setup in \citep{salehkaleybar-et-al22,fatkhullin2023stochastic}. The algorithm shares similarities with the variance reduction method proposed in \citep{cutkosky-orabona19} but incorporates second-order information as opposed to the difference between consecutive stochastic gradients \citep{BetterSGDUsingSOM_Tran_2021}.

At the $k$th epoch, Algorithm \ref{alg:PG_Hessian_Avg} generates two trajectories, each of length $H/2$. The first one is $\pi_{\theta_k}$ induced where $\theta_k$ is the current estimate of the policy parameter while the second one is $\pi_{\hat{\theta}_k}$ induced where $\hat{\theta}_k = q_k\theta_k + (1-q_k)\theta_{k-1}$, $q_k$ being a uniform random number in $[0, 1]$. The trajectories are named $\tau_k$ and $\hat{\tau}_k$ respectively. The policy parameter $\theta_k$ is updated by following \eqref{algorithm-update-Hessian} where $\gamma_k$ is a learning parameter and the direction $d_k$ is recursively obtained as follows.
\begin{align}
\label{eq:algo_update_d_k_Hessian}
    d_k = (1-\eta_k) (d_{k-1}+v(\hat{\theta}_k,\hat{\tau}_k)) + \eta_k g(\theta_k,\tau_k) 
\end{align}
where $\eta_k$ is a learning rate and the second order correction $v(\hat{\theta}_k, \hat{\tau}_k)$ is computed as shown below.
\begin{align}
\label{eq:vk}
    v(\hat{\theta}_k,\hat{\tau}_k) \coloneqq B(\hat{\theta}_k,\hat{\tau}_k)(\theta_k - \theta_{k-1}). 
\end{align}
where $B(\hat{\theta}_k,\hat{\tau}_k)$ is the Hessian estimate defined by \eqref{eq:Hessian-estimate}. The term $\hat{\theta}_k$ is defined such that it ensures $v(\hat{\theta}_k, \hat{\tau}_k)$ to be an unbiased estimate of $\nabla_{\theta} J(\theta_k) - \nabla_{\theta} J(\theta_{k-1})$. Algorithm \ref{alg:PG_Hessian_Avg} does not need to store the Hessian estimate $B(\hat{\theta}_k, \hat{\tau}_k)$ since it only uses the matrix-vector product $v(\hat{\theta}_k, \hat{\tau}_k)$.

\textbf{Remark:} It is to be mentioned that Algorithm \ref{alg:estA} draws inspiration from Algorithm 2 in \citet{bai2023regret}. A major distinction lies in the length of the input trajectory $|\tau|=H$. While \citet{bai2023regret} utilizes $H=\tilde{\cO}(\sqrt{T})$-length trajectories to estimate the advantage function, we use $H=\cO(T^{1/6}\log^2 T)$ for Algorithm \ref{alg:PG_IGT_Avg} and $H=\cO(\log^2 T)$ for Algorithm \ref{alg:PG_Hessian_Avg} (see Theorem \ref{thm:alg1-regret}, \ref{thm:alg2-regret}). The variance reduction techniques allow us to use a shorter epoch length, which plays a pivotal role in improving the regret bound.

%{\bf \color{red} Unclear what is novel in Algorithm - I am sure there is a lot and cannot be simple combination - while not coming out. }

\section{MAIN RESULTS}
\label{sec:main}

In this section, we establish the global convergence of Algorithms \ref{alg:PG_IGT_Avg} and \ref{alg:PG_Hessian_Avg}. Before delving into the details, we would like to highlight a few assumptions necessary for proving these results.

% Note to self: rewrite upto here.

\begin{assumption}[Policy parametrization regularity]
    \label{assump:score_func_bounds}
    For all $\theta, \theta_1,\theta_2 \in\Theta$ and $(s,a)\in\mathcal{S}\times\mathcal{A}$, the following statements hold:
    \begin{align}
    (a) \text{ }\Vert\nabla_{\theta}\log\pi_\theta(a\vert s)\Vert\leq G \quad (b) \text{ }
        \Vert \nabla_{\theta}^2\log\pi_{\theta_1}(a\vert s)\Vert\leq B.
    \end{align}
\end{assumption}

\begin{remark}
    The Lipschitz and smoothness properties for the log-likelihood are quite common in the field of policy gradient algorithm \citep{Alekh2020, Mengdi2021, liu2020improved}. These properties were shown to hold for various examples recently including Gaussian policies with linearly parameterized means and certain neural parametrizations \citep{liu2020improved, fatkhullin2023stochastic}.
\end{remark}

Define the transferred function approximation error
\begin{align}
    \label{eq:transfer_error}
    \begin{split}
            &L_{d^{\pi^*},\pi^*}(\omega^*_{\theta},\theta ) \\
            &=\E_{s\sim d^{\pi^*} a\sim\pi^*(\cdot\vert s)}\bigg[\bigg(\nabla_\theta\log\pi_{\theta}(a\vert s)\cdot\omega^*_{\theta}-A^{\pi_\theta}(s,a)\bigg)^2\bigg],
\end{split}
\end{align}
 where $\pi^*$ is the optimal policy and $\omega^*_{\theta}$ is given as
    \begin{align}
        \label{eq:NPG_direction}
	\begin{split}
            &\omega^*_{\theta}=\argmin_{\omega\in\mathbb{R}^{d}}\\
            &~\E_{s\sim d^{\pi_{\theta}}a\sim\pi^*(\cdot\vert s)}\bigg[\bigg(\nabla_\theta\log\pi_{\theta}(a\vert s)\cdot\omega-A^{\pi_{\theta}}(s,a)\bigg)^2\bigg].
	\end{split}
    \end{align}

It is worth mentioning that $\omega^*_\theta$ defined in \eqref{eq:NPG_direction} can be alternatively written as,
    \begin{align*}
        \omega^*_{\theta} = F(\theta)^{\dagger} \E_{s\sim d^{\pi_{\theta}}}\E_{a\sim\pi_{\theta}(\cdot\vert s)}\left[\nabla_{\theta}\log\pi_{\theta}(a|s)A^{\pi_{\theta}}(s, a)\right],
    \end{align*}
    where $\dagger$ symbolizes the Moore-Penrose pseudoinverse operation and $F(\theta)$ is the Fisher information matrix as defined below:
    \begin{align}
            F(\theta) = \E_{s\sim d^{\pi_{\theta}} a\sim\pi^*(\cdot\vert s)}\left[
              \nabla_{\theta}\log\pi_{\theta}(a|s)(\nabla_{\theta}\log\pi_{\theta}(a|s))^\top \right].
    \end{align}

\begin{assumption}
    \label{assump:function_approx_error}

    We assume that the error satisfies $L_{d^{\pi^*},\pi^*}(\omega^*_{\theta},\theta)\leq \epsilon_{\mathrm{bias}}$ for any $\theta\in\Theta$ where $\epsilon_{\mathrm{bias}}$ is a positive constant.
\end{assumption}
 
\begin{remark}
    The transferred function approximation error, defined by  \eqref{eq:transfer_error} and \eqref{eq:NPG_direction}, quantifies the expressivity of the policy class in consideration. It has been shown that the softmax parametrization \citet{agarwal2021theory} or linear MDP  \citet{Chi2019} admits $\epsilon_{\mathrm{bias}}=0$. When parameterized by a restricted policy class that does not contain all the policies, $\epsilon_{\mathrm{bias}}$ turns out to be strictly positive. However, for a rich neural network parametrization, the $\epsilon_{\mathrm{bias}}$ is small \citep{Lingxiao2019}. A similar assumption has been adopted in \citet{liu2020improved} and \citet{agarwal2021theory}. 
\end{remark}

\begin{assumption}[Fisher non-degenerate policy]
    \label{assump:FND_policy}
    There exists a constant $\mu>0$ such that $F(\theta)-\mu I_{d}$ is positive semidefinite where $I_{d}$ denotes an identity matrix.
\end{assumption}

Assumption \ref{assump:FND_policy}  requires that the eigenvalues of the Fisher information matrix can be bounded from below. This is commonly used in obtaining global complexity bounds for PG based methods \citep{liu2020improved,Mengdi2021,bai2023achieving,fatkhullin2023stochastic}. Assumptions \ref{assump:score_func_bounds}-\ref{assump:FND_policy} are widely use in PG literature and are shown to hold simultaneously for parametrization including Gaussian policies with linearly parameterized means and certain neural parametrizations  \citet{liu2020improved,fatkhullin2023stochastic}.

%{\color{red} Where are all these assumptions satisfied?}

% We now state our first result, which provides an $L$-smoothness type bound for $J$.

%{\color{red} Remark, and key proof idea/novelty there. }

Next, we provide the regret bounds of Algorithm \ref{alg:PG_IGT_Avg} and \ref{alg:PG_Hessian_Avg} in the following theorems, with their proof outlined in Appendix \ref{sec:proof-outline}.

\begin{theorem}[Regret bound for Algorithm \ref{alg:PG_IGT_Avg}]
    \label{thm:alg1-regret}
    Let $\{\theta_k\}_{k=1}^{K}$ be the outputs of Algorithm \ref{alg:PG_IGT_Avg}. If Assumptions \ref{assump:ergodic_mdp}, \ref{assump:score_func_bounds}, \ref{assump:function_approx_error} and  \ref{assump:FND_policy} hold, $\nabla_{\theta} J(\theta)$ is $L_h$-smooth, $\gamma_k=\frac{6G}{\mu(k+2)}$ and $\eta_k = \left(\frac{2}{k+2}\right)^{4/5}$ then the following inequality holds for $K=T/H$ where $H=63t_{\mathrm{mix}}t_{\mathrm{hit}}(\log_2 T)^2T^{1/6}$ 
    \begin{align*}
   &\E[\mathrm{Reg}_T] \leq T \sqrt{\epsilon_{\mathrm{bias}}} + \cO\Bigg(\bigg(G^2 A t_{\mathrm{mix}}^{7/5}t_{\mathrm{hit}}^{2/5} (\log T)^{9/5} \\
   &+G^3 \mu^{-2}t_{\mathrm{mix}}^{2/5}t_{\mathrm{hit}}^{2/5} L_h (\log T)^{4/5} +\sqrt{A}BGt_{\mathrm{mix}}^{7/5}t_{\mathrm{hit}}^{2/5}(\log T)^{9/5}\\
   &+ \sqrt{A}G^2 t_{\mathrm{mix}}^{12/5}t_{\mathrm{hit}}^{7/5}(\log T)^{14/5}\bigg)\cdot \mu^{-1}T^{2/3}\Bigg).
    \end{align*}
\end{theorem}

%{\color{red} Outline Proof}

Note that the above result requires $\nabla_{\theta} J$ to be $L_h$ smooth. Unlike the discounted case where smoothness follows from the standard assumptions on the score function, this is not the case for the average-reward setting and we assume it to hold instead. However, our Hessian-based algorithm (Algorithm \ref{alg:PG_Hessian_Avg}) does not require any such assumption and obtains an order-optimal regret bound.

\begin{theorem}[Regret bound for Algorithm \ref{alg:PG_Hessian_Avg}]
    \label{thm:alg2-regret}
    Let $\{\theta_k\}_{k=1}^{K}$ be generated from Algorithm \ref{alg:PG_Hessian_Avg}. If Assumptions \ref{assump:ergodic_mdp}, \ref{assump:score_func_bounds}, \ref{assump:function_approx_error} and  \ref{assump:FND_policy} hold, $\gamma_k=\frac{6G}{\mu(k+2)}$ and $\eta_k = \frac{2}{k+2}$ then the following inequality holds for $K=T/H$ where $H=63t_{\mathrm{mix}}t_{\mathrm{hit}}(\log_2 T)^2$ 
    \begin{align*}
   \E[\mathrm{Reg}_T] \leq T \sqrt{\epsilon_{\mathrm{bias}}}& + \cO\bigg(\frac{\sqrt{A} G^2t_{\mathrm{mix}}\log T }{\mu}\cdot \sqrt{T} \\
   &+ \frac{\sqrt{A}G^4 t_{\mathrm{hit}}t^2_{\mathrm{mix}} (\log T)^{3/2}}{\mu^2}\cdot \sqrt{T} \\
   &+\frac{\sqrt{A}(BG+G^3) t_{\mathrm{mix}}\log T}{\mu^2}\cdot\sqrt{T}\bigg).
    \end{align*}
\end{theorem}

Theorem \ref{thm:alg1-regret} and \ref{thm:alg2-regret} demonstrate that the regret accrued by Algorithm \ref{alg:PG_IGT_Avg} and \ref{alg:PG_Hessian_Avg} are $\tilde{\mathcal{O}}(T^{2/3})$ and $\tilde{\mathcal{O}}(\sqrt{T})$ respectively. Note that both of these results are improvements over the current state-of-the-art bound of $\tilde{\mathcal{O}}(T^{3/4})$ \citep{bai2023regret}. Moreover, the regret of Algorithm \ref{alg:PG_Hessian_Avg} is optimal in terms of $T$. Observe the presence of the term $\sqrt{\epsilon_{\mathrm{bias}}}$ in both of these results, which appears due to the incompleteness of parameterized policy class. The proofs of both of these theorems rely on the construction of an auxiliary function, $\bar{J}$, which, intuitively, is a close approximation of $J$. However, unlike $J$, $\bar{J}$ satisfies some desirable analytical properties such as $L$-smoothness for some $L>0$. 

%{\bf \color{red} While significance of regret results is explained, that for last-iterate bounds is not. Is that used in regret or what is their significance?}
%{\color{red} Outline Proof}

\subsection{Construction and Properties of the Auxiliary Function}
\label{subsec:aux-funcs}

We construct the auxiliary function $\Bar{J}(\theta)$ such that our gradient and Hessian estimators are unbiased estimates of $\nabla_{\theta} \Bar{J}$ and $\nabla_{\theta} ^2 \Bar{J}$, respectively. The function $\Bar{J}$ must also be close to $J$. Working with this function is essential, especially for Hessian-based methods, since the expected value of the Hessian estimates can potentially be far from $\nabla_{\theta}^2 J$, unlike the discounted case where the bias of the Hessian estimates is known to decay exponentially in $\gamma$ \citep{masiha}. Moreover, $\Bar{J}$ is also $L$-smooth, for some $L>0$, which helps us prove the approximate smoothness result for $J$. We define $\Bar{J}(\theta)$ as follows.
\begin{align}
\label{eq:bar-J-defn}
   \Bar{J}(\theta) = 
    J(\theta_0) + \int_{0}^{1} f((1-q)\theta_0+q\theta)\cdot (\theta - \theta_0) dq,
\end{align}
where $f(\theta) = \E_{\tau \sim p(\cdot, \theta, \bar{\rho})} g(\theta, \tau)$, $g(\theta, \tau)$ is defined in \eqref{eq:grad_estimate} and $p(\tau, \theta, \bar{\rho})$, as discussed in section \ref{sec:proposed-algo}, is the probability of occurrence of a $\pi_\theta$-induced trajectory, $\tau$, that initiates with state distribution, $\bar{\rho}$ and $\theta_0\in\Theta$ is the initial parameter. Recall that $g(\theta, \tau) = \nabla_{\theta}\Phi(\theta, \tau)$ where $\Phi(\theta, \tau)$ is defined in \eqref{eq:def_Phi}. It follows that (see Appendix \ref{subsec:Hessian-details})
\begin{align}
   \nabla_{\theta} \Bar{J}(\theta) = \E_{\tau \sim p(\cdot, \theta, \bar{\rho})} \left[\nabla_{\theta} \Phi(\theta,\tau)\right] %= \int_{\tau} p(\tau, \theta, \rho)\nabla_{\theta} \Phi (\theta,\tau) d\tau.
\end{align}
 If we sample $\tau$ according to $p(\cdot, \theta, \bar{\rho})$, $g(\theta,\tau)$ will serve as an unbiased estimator for $\nabla_{\theta} \Bar{J}(\theta)$. Since the functions $J$ and $\bar{J}$ are not identical, the above estimate generates a small bias term that diminishes with increasing $T$. We can derive the Hessian of $\Bar{J}$ in the following manner (see Appendix \ref{subsec:Hessian-details}).
\begin{align}
    \nabla^2_{\theta} \Bar{J}(\theta) = \E_{\tau \sim p(\tau, \theta, \bar{\rho})}  \left[ B(\theta, \tau) \right].
\end{align}
where $B(\theta, \tau)$ is defined in \eqref{eq:Hessian-estimate}. This establishes $B(\theta, \tau)$ to be an unbiased estimator of $\nabla_{\theta}^2 \Bar{J}$. From the above expression, we can bound $\|\nabla_{\theta}^2\Bar{J}(\theta)\|$ as (see Appendix \ref{sec:proof-outline} for details)
\begin{align}
\label{eq:norm_hessian_J_bar}
    \|\nabla^2_{\theta} \Bar{J}(\theta)\| &\leq L,
\end{align}
 where $L$ depends on $H$ and $N$. Lemma \ref{lem:grad+hess-est-prop}, \ref{lem:bar-J-J-dist} (Appendix \ref{sec:proof-outline}) show that the following holds for any $\theta$.
 \begin{align}
 \label{eq:dist_J_J_bar}
     &\Vert\nabla_{\theta}J(\theta)-\nabla_{\theta} \bar{J}(\theta)\Vert \leq \tilde{\mathcal{O}}\left(\dfrac{A Gt_{\mathrm{mix}}}{T^4}\right) \nonumber\\
    &|J(\theta)-\Bar{J}(\theta)| 
     \leq  \tilde{\mathcal{O}}\left(\dfrac{A Gt_{\mathrm{mix}}}{T^4}\cdot\|\theta-\theta_0\|\right)
\end{align}
The above relations indicate that $J$, $\bar{J}$, in some sense, are close to each other. Using \eqref{eq:norm_hessian_J_bar}, \eqref{eq:dist_J_J_bar}, and the Taylor expansion, we show the following approximate smoothness result of $J$.
\begin{theorem}
\label{thm:smoothness}
If Assumptions \ref{assump:ergodic_mdp} and \ref{assump:score_func_bounds} hold, then the following bound holds for all $\theta, \bar{\theta} \in \Theta$.     
\begin{align*}
    &(J^* - J(\bar{\theta})) \\
    &\leq (J^* - J(\theta))- \nabla_{\theta} J(\theta)^\top  (\bar{\theta} - \theta) + \frac{L}{2} \|\bar{\theta} -\theta\|^2 + \Delta_L, 
\end{align*}
where $L = \cO(\sqrt{A}G^2t_{\mathrm{mix}}^2t_{\mathrm{hit}}(\log T)^3+Bt_{\mathrm{mix}}\log T)$ and $\Delta_L = \cO \left(\frac{AG t_{\mathrm{mix}}\{\|\theta-\theta_0\| + \|\bar{\theta}-\theta_0\| \}}{T^4}\right)$.
\end{theorem}

The above result is crucial for proving Theorems \ref{thm:alg1-regret} and \ref{thm:alg2-regret}. The proof of Theorem \ref{thm:smoothness} follows by observing that, from \eqref{eq:dist_J_J_bar}, we have
\begin{align*}
    &(a)\quad \bar{J}(\theta) - \epsilon_1 \|\theta-\theta_0\| \leq J(\theta) \leq \bar{J}(\theta) + \epsilon_1 \|\theta-\theta_0\| \\
    &(b) \quad \nabla \bar{J}(\theta)^{\top}(\theta-\bar{\theta})-\epsilon_2 \|\theta-\bar{\theta}\|\leq \nabla J(\theta)^{\top}(\theta-\bar{\theta}) \\
    &\quad \quad \leq \nabla \bar{J}(\theta)^{\top}(\theta-\bar{\theta}) + \epsilon_2 \|\theta-\bar{\theta}\|
\end{align*}

where $\epsilon_1,\epsilon_2$ are $O(AGt_{\mathrm{mix}}T^{-4})$. Separately, since $\bar{J}$ is $L$-smooth,
\begin{align}
\label{eq:intermediate}
    (J^* - \bar{J}(\bar{\theta}))& \leq (J^* - \bar{J}(\theta)) - \nabla \bar{J}(\theta)^{\top} (\bar{\theta}-\theta)\nonumber\\
    &\quad +(L/2)\|\bar{\theta}-\theta\|^2
\end{align}

Substituting $(a)$ and $(b)$ in \eqref{eq:intermediate}, we obtain 
\begin{align*}
    J^* - J(\bar{\theta})& \leq (J^* - J(\theta)) - \nabla J(\theta)^{\top} (\bar{\theta}-\theta)\\
    &+(L/2)\|\bar{\theta}-\theta\|^2 +\epsilon_1(\|\theta-\theta_0\| +\|\bar{\theta}-\theta_0\|)\\
    &+\epsilon_2\|\theta-\bar{\theta}\|.
\end{align*}
Since $\|\bar{\theta}-\theta_0\| \leq \|\theta-\theta_0\|+\|\theta - \bar{\theta}\|$, the result follows.

\section{CONCLUSION}
This paper considers the problem of learning an infinite-horizon average reward MDP with general parametrization. We propose two algorithms. The first, based on the idea of implicit gradient transport, achieves a regret bound of $\tilde{\mathcal{O}}(T^{2/3})$, while the second, a Hessian-based algorithm, achieves a regret of $\tilde{\mathcal{O}}(\sqrt{T})$. Both improve upon the state of the art, with the second achieving the optimal bound. In the process, we also prove an approximate smoothness property of the average-reward function, which may be of independent interest.

Following this work, the authors of \citep{ganesh2024order} show an $\Tilde{\mathcal{O}}(T^{-1/2})$ global convergence rate without requiring knowledge of the mixing time. However, their approach introduces an additional critic error due to the linear approximation of the critic. Extending these algorithms to design parameter-free algorithms is an important future direction.

\section*{ACKNOWLEDGMENT}
This work was supported in part by the Overseas Visiting Doctoral Fellowship from Anusandhan National Research Foundation (ANRF) and U.S. National Science Foundation under grant CCF-2149588.

\bibliography{references}
%%%%%%%%%%%%%%%%%%%%%%%%%%%%%%%%%%%%%%%%%%%%%%%%%%%%%%%%%%%%

\newpage
\clearpage

 \appendix

 \onecolumn

\section{Details of the Hessian estimator}
\label{subsec:Hessian-details}
Observe the following.
\begin{align}
\label{eq:appndx_27}
\begin{split}
    \nabla_{\theta} &\Phi(\theta,\tau) = \frac{1}{|\tau|-N}\sum_{t=t_s+N}^{t_e} \left\lbrace\Psi_t^{(1)}(\tau) \nabla_{\theta} \log \pi_{\theta}(a_t | s_t) + \Psi_t^{(2)}(\tau)\nabla_{\theta} \left(\frac{1}{\pi_{\theta}(a_t | s_t)}\right)\right\rbrace\\
    &= \frac{1}{|\tau|-N}\sum_{t=t_s+N}^{t_e} \left\lbrace \Psi_t^{(1)}(\tau) \nabla_{\theta} \log \pi_{\theta}(a_t | s_t) + \Psi_t^{(2)}(\tau) \left(\frac{-\nabla_{\theta}\log \pi_{\theta}(a_t|s_t)}{\pi_{\theta}(a_t | s_t)}\right)\right\rbrace \\
    &= \frac{1}{|\tau|-N}\sum_{t=t_s+N}^{t_e} (\hat{Q}^{\pi_\theta}(\tau, s_t, a_t) - \hat{V}^{\pi_\theta}(\tau, s_t))\nabla_{\theta}\log \pi_{\theta}(a_t|s_t) \\
    &= \frac{1}{|\tau|-N}\sum_{t=t_s+N}^{t_e} \hat{A}^{\pi_\theta}(\tau, s_t, a_t)\nabla_{\theta}\log \pi_{\theta}(a_t|s_t)= g(\theta, \tau)
\end{split}
\end{align}

Moreover, note that,
\begin{align*}
\begin{split}
&\nabla_{\theta}^2 \Phi(\theta, \tau) 
=  \frac{1}{|\tau|-N}\sum_{t=t_s+N}^{t_e} \nabla_{\theta} \left( \Psi_t^{(1)}(\tau) \nabla_{\theta} \log \pi_{\theta} (a_t | s_t) - \Psi_t^{(2)}(\tau) \left(\frac{\nabla_{\theta}\log \pi_{\theta}(a_t|s_t)}{\pi_{\theta}(a_t | s_t)}\right)\right) \\ 
& = \frac{1}{|\tau|-N}\sum_{t=t_s+N}^{t_e} \left\lbrace \Psi_t^{(1)}(\tau) \nabla_{\theta}^2 \log \pi_{\theta}(a_t | s_t) - \Psi_t^{(2)}(\tau) \nabla_{\theta} \left(\frac{\nabla_{\theta} \log \pi_{\theta}(a_t|s_t)}{\pi_{\theta}(a_t | s_t)}\right)\right\rbrace \\ 
& = \frac{1}{|\tau|-N}\sum_{t=t_s+N}^{t_e} \left\lbrace \Psi_t^{(1)}(\tau) \nabla_{\theta}^2 \log \pi_{\theta}(a_t | s_t) \right.\\
&\left. \hspace{2cm}- \Psi_t^{(2)}(\tau)\left(\frac{\nabla_{\theta}^2\log \pi_{\theta}(a_t|s_t)}{\pi_{\theta}(a_t | s_t)}-\frac{\nabla_\theta \log \pi_{\theta} (a_t|s_t) \nabla_\theta \log \pi_{\theta} (a_t|s_t)^\top }{\pi_{\theta}(a_t | s_t)}\right)\right\rbrace  
\end{split}
\end{align*}
The above expression and \eqref{eq:appndx_27} together outline the process to compute $B(\theta, \tau)$. Recall the definition of $\Bar{J}(\theta)$ as stated below.
\begin{align}
     \Bar{J}(\theta) = 
    J(\theta_0) + \int_{0}^{1} f((1-q)\theta_0+q\theta)\cdot (\theta - \theta_0) dq,
\end{align}
where $f(\theta) = \E_{\tau \sim p(\cdot, \theta, \bar{\rho})}[\nabla \Phi(\theta, \tau)] = \E_{\tau \sim p(\cdot, \theta, \bar{\rho})} [g(\tau,\theta)]$. It follows that 
\begin{align}
   \nabla_{\theta} \Bar{J}(\theta) = f(\theta)=\E_{\tau \sim p(\cdot, \theta, \bar{\rho})} \left[\nabla \Phi(\theta, \tau)\right] = \sum_{\tau} p(\tau, \theta, \bar{\rho}) \nabla_{\theta}\Phi (\theta,\tau)
\end{align}
The above expression establishes that $\nabla_\theta \Phi(\theta, \tau)=g(\theta, \tau)$ is an unbiased estimate of $\nabla_\theta \Bar{J}(\theta)$. The Hessian of $\Bar{J}(\theta)$ can be computed as shown below.
\begin{align}
\begin{split}
    \nabla_{\theta}^2 \Bar{J}(\theta) &= \nabla_{\theta} (\nabla_{\theta} \Bar{J}(\theta)) 
    =\sum_{\tau} \nabla _{\theta} (p(\tau, \theta, \bar{\rho}) \nabla_{\theta} \Phi (\theta, \tau)) \\
     &= \sum_{\tau}  \nabla_{\theta} \Phi (\theta, \tau) \nabla_{\theta} p(\tau, \theta, \bar{\rho})^\top  + p(\tau, \theta, \bar{\rho}) \nabla_{\theta}^2 \Phi (\theta, \tau) \\
     &= \sum_{\tau} p(\tau, \theta, \bar{\rho}) \left[\nabla_{\theta} \Phi (\theta, \tau) \nabla_{\theta} \log p(\tau, \theta, \bar{\rho})^\top  + \nabla_{\theta}^2 \Phi (\theta, \tau)\right] \\
     &= \E_{\tau \sim p(\cdot, \theta, \bar{\rho})}  \left[\nabla_{\theta}  \Phi (\theta, \tau) \nabla_{\theta} \log p(\tau, \theta, \bar{\rho})^\top  + \nabla_{\theta}^2 \Phi (\theta, \tau) \right] =  \E_{\tau \sim p(\cdot, \theta, \bar{\rho})}[B(\theta, \tau)]
\end{split}
\end{align}
The above exercise shows that $B(\theta, \tau)$ is an unbiased estimate of $\nabla_\theta^2 \bar{J}(\theta)$.

\textbf{Computational complexity of Algorithms \ref{alg:PG_IGT_Avg}, \ref{alg:estA} and \ref{alg:PG_Hessian_Avg}:} We note that the computational complexity of Algorithm \ref{alg:estA}, which estimates the value functions has computational complexity of order $O(NH)$. This is because this method uses $H$ iterations and the operations at each iteration are $O(N)$.

We first consider Algorithm \ref{alg:estA}. Line 3 is $O(d)$, lines 4-7 are $O(H)$, lines 8-10 are $O(NH^2)$. Computing the gradient estimate at line 11 is $O(Hd)$ and line 12 is $O(d)$. Since all these steps are done for $K$ iterations, the final complexity of Algorithm \ref{alg:PG_IGT_Avg} is $O(KH^2N+ KHd)$.

For Algorithm \ref{alg:PG_Hessian_Avg}, the computations are similar, except we compute the Hessian estimate as well. The complexity of computing the Hessian estimator in \eqref{eq:Hessian-estimate} is $O(Hd^2)$. To see this, note that the expression for $\nabla_{\theta}^2\Phi(\theta,\tau)$ consists of an average of $H$ matrices of order $d \times d$. The computational complexity of computing these matrices is $O(d^2)$. It follows that the complexity to compute $\nabla_{\theta}^2\Phi(\theta,\tau)$ is $O(Hd^2)$. Separately, the complexity of computing $\nabla_{\theta} \log p(\tau,\theta,\bar{\rho})$ is $O(Hd)$ since it is the sum of $\nabla_{\theta} \log \pi_{\theta}(a_i|s_i)$ for all $(s_i,a_i)\in \tau$. $\nabla_{\theta} \Phi(\theta,\tau)$ is the gradient estimator and taking the outer product with $\nabla_{\theta} \log p(\tau,\theta,\bar{\rho})$ is $O(d^2)$. This gives a final complexity of $O(KH^2N+KHd^2)$ for Algorithm 3.

\section{Proof Outline}
\label{sec:proof-outline}
Before proceeding, we present some useful results related to different estimators used in the paper. The first result is on the first and second-order errors of advantage estimation.

\begin{lemma}
    \label{lemma:advatge_estimate} 
    Consider the advantage estimator $\hat{A}^{\pi}(\tau, s, a)$ (described in section \ref{sec:proposed-algo}) corresponding to a policy $\pi$, a state-action pair $(s, a)$, and a $\pi$-induced trajectory $\tau$. If the length of the trajectory is $|\tau|=H\geq 63t_{\mathrm{mix}}t_{\mathrm{hit}}(\log_2 T)^2$ and $N=7t_{\mathrm{mix}}\log  T$, then the following relations hold.
    \begin{align*}
        \label{eq:appndx_var_advanatge_Estimate}
        &(a)~{\E}\left[\left(\hat{A}^{\pi}(\tau, s, a) - A^{\pi}(s, a)\right)^2\right]\leq \mathcal{O}\left(\dfrac{t_{\mathrm{mix}}^3t_{\mathrm{hit}}(\log T)^4}{H\pi(a|s)}\right)\\
        &(b)\left|{\E}\left[\hat{A}^{\pi}(\tau, s, a) - A^{\pi}(s, a)\right]\right|~\leq \mathcal{O}\left(\dfrac{t_{\mathrm{mix}}\log T}{\sqrt{\pi(a|s)}T^4}\right)
    \end{align*}
    where the expectation is over all such $\pi$-induced $\tau$ with arbitrary initial distribution $\bar{\rho}$.
\end{lemma}

\begin{lemma}
\label{lem:grad+hess-est-prop}
    Consider the gradient and the Hessian estimators given in \eqref{eq:grad_estimate} and \eqref{eq:Hessian-estimate} respectively. Let the $\pi_\theta$-induced trajectory $\tau=\{s_t, a_t\}_{t=t_s}^{t_e}$ used in estimation is such that $|\tau|=H\geq 63t_{\mathrm{mix}}t_{\mathrm{hit}}(\log_2 T)^2$. If assumptions \ref{assump:ergodic_mdp} and \ref{assump:score_func_bounds} hold, the following statements are true $\forall\theta \in \Theta$.
    \begin{enumerate}[label=(\alph*)]
        \item $\E\|g(\theta,\tau) - \nabla_{\theta} J(\theta)\|^2 \leq \sigma^2_g$, where $ \sigma_g^2= \cO\left(\dfrac{A G^2 t_{\mathrm{hit}}t_{\mathrm{mix}}^3(\log T)^4}{H}\right)$
        \item  $\|\E[g(\theta,\tau)] - \nabla_{\theta} J(\theta)\| \leq \beta_g$, where $\beta_g = \mathcal{O}\left(\dfrac{\sqrt{A} Gt_{\mathrm{mix}}\log T}{T^4}\right)$
        \item $\E\|B(\theta,\tau)\|^2 \leq  \mathcal{O}\left(A G^4 H^2  t_{\mathrm{mix}}^2(\log T)^2 + B^2 t_{\mathrm{mix}}^2(\log T)^2  +  A   (B^2+G^4) t_{\mathrm{mix}}^2(\log T)^2\right)$
    \end{enumerate}
    where the expectations are over $\tau\sim p(\cdot, \theta, \bar{\rho})$.
\end{lemma}

Notice that Lemma \ref{lem:grad+hess-est-prop}$(b)$ bounds the bias of the gradient estimate. For defining $\bar{J}(\cdot)$ with $\nabla_\theta \bar{J}(\theta) = \E [g(\theta, \tau)]$, we choose $|\tau|=H= 63t_{\mathrm{mix}}t_{\mathrm{hit}}(\log_2 T)^2$. Then, we can interpret Lemma \ref{lem:grad+hess-est-prop}$(b)$ as an upper bound on $\Vert\nabla_\theta \bar{J}(\theta) - \nabla_\theta J(\theta)\Vert$. Lemma \ref{lem:grad+hess-est-prop}$(c)$ implies that $\Bar{J}$ is $L$-smooth, where $L^2 = \mathcal{O}\left(AB^2t_{\mathrm{mix}}^2(\log T)^2+AG^4t_{\mathrm{mix}}^4t_{\mathrm{hit}}^2(\log T)^4\right)$. To see this, note that,
 \begin{align}
     \|\nabla^2_{\theta} \Bar{J}(\theta) \| = \|\E [B(\theta,\tau)]\| \leq (\E \|B(\theta,\tau)\|^2)^{1/2}. 
 \end{align}

Substituting $H= 63t_{\mathrm{mix}}t_{\mathrm{hit}}(\log_2 T)^2$ in Lemma \ref{lem:grad+hess-est-prop}(c), we obtain $\|\nabla^2_{\theta} \Bar{J}(\theta) \|  \leq L$. Observe that we do not attempt to bound the bias of the Hessian estimate, unlike that for the gradient estimate. This is because, unlike the discounted reward MDP, the expectation of the Hessian estimate might be far from the true Hessian. We instead use the fact that it is an unbiased estimate of $\nabla^2_{\theta} \Bar{J}(\theta)$, where $\nabla_{\theta} \Bar{J}(\theta)$ is close to $\nabla_{\theta} J(\theta)$. Using Lemma \ref{lem:grad+hess-est-prop}, we can bound $J$ and $\Bar{J}$ as follows:

\begin{lemma}
\label{lem:bar-J-J-dist}
Let Assumptions \ref{assump:ergodic_mdp} and \ref{assump:score_func_bounds} hold. Let $\Bar{J}$ be defined by \eqref{eq:bar-J-defn} and $N$ and $|\tau|$ are the same as given in Lemma \ref{lem:grad+hess-est-prop}. Then the following statement holds for all $\theta \in \Theta$:
\begin{align}
    |J(\theta)-\Bar{J}(\theta)| \leq  \mathcal{O}\left(\dfrac{A Gt_{\mathrm{mix}}\log T}{T^4}\|\theta-\theta_0\|\right)
\end{align}
\end{lemma}
\begin{proof}
    The proof of the above result follows from the observation that
\begin{align}
\begin{split}
    |\Bar{J}(\theta)-J(\theta)| &= \bigg| \int_{0}^{1} \big[f((1-q)\theta_0+q\theta)-\nabla_{\theta} J((1-q)\theta_0+q\theta)\big]\cdot (\theta - \theta_0) \mathrm{d}q\bigg|\\
    &\leq \int_0^1\|f((1-q)\theta_0+q\theta)-\nabla_{\theta} J((1-q)\theta_0+q\theta)\| \|\theta - \theta_0\| \mathrm{d}q\\
    &\overset{(a)}{\leq} \mathcal{O}\left(\dfrac{A Gt_{\mathrm{mix}}\log T}{T^4}\|\theta-\theta_0\|\right)
\end{split}
\end{align}
where $f(\theta) = \E_{\tau\sim p(\cdot, \theta, \rho)}[g(\theta,\tau)]$ and $(a)$ follows from Lemma \ref{lem:grad+hess-est-prop}$(b)$.
\end{proof}

\begin{lemma}
\label{lem:exp-regret}
Consider Algorithms \ref{alg:PG_IGT_Avg} and \ref{alg:PG_Hessian_Avg} with $N = 7t_{\mathrm{mix}} \log_2 T$ and let Assumptions \ref{assump:ergodic_mdp} and \ref{assump:score_func_bounds} hold. Then the following statement holds for Algorithm \ref{alg:PG_IGT_Avg} when number of iterations $K = T/H$:
\begin{align}
    \begin{split}
        \E[\mathrm{Reg}_T] \leq \cO\left(H\sum_{k=1}^{K}\left(J^*-\E[J({\tilde{\theta}_k})]\right)+ G^2\mu^{-1}t_{\mathrm{mix}}^2(\log^2 T) K^{4/5}\right).
    \end{split}
\end{align}
Similarly, the following holds for Algorithm \ref{alg:PG_Hessian_Avg} when number of iterations $K = T/H$:
\begin{align}
    \begin{split}
        \E[\mathrm{Reg}_T] \leq \cO\left(H\sum_{k=1}^{K}\left(J^*-\E[J(\theta_k)]\right)+ G^2\mu^{-1}t_{\mathrm{mix}}^2(\log^3 T) \right).
    \end{split}
\end{align}
\end{lemma}

From Lemma \ref{lem:exp-regret}, it can be seen that regret bounds can be obtained by suitably bounding $\sum_{k=1}^{K}\left(J^*-\E J({\theta_k})\right)$. Towards this, we have the following general lemma similar in spirit to [Lemma 5, \citet{fatkhullin2023stochastic}] for bounding $J^*-\E J({\theta_k})$ for all $k \geq 1$.

\begin{lemma}
\label{lem:last_iterate_general}
    Let Assumption~\ref{assump:score_func_bounds}, \ref{assump:function_approx_error} and \ref{assump:FND_policy} hold and $\{\theta_k\}_{k\geq 1}$ be a sequence generated by the following update rule $\forall k\in\{0, 1, \cdots, K-1\}$.
$$
\theta_{k+1} = \theta_k + \gamma_k \frac{d_k}{\norm{d_k}}\,,
$$
where $\gamma_k = \frac{6G}{\mu (k+2)}$, $\{d_k\}_{k\geq 1}$ is any sequence  in $\R^d$ and $\theta_0 \in \R^d$ ($\theta_{k+1} = \theta_k$ if $d_k = 0$) is arbitrary. The following statement holds for every integer $K \leq T$. 
\begin{align*}
    J^*-\E [J(\theta_{K})] \leq \frac{J^*-J(\theta_{0})}{(K+1)^2} +
    \frac{\sum_{k=1}^{K}\nu_k(k+2)^2}{(K+1)^2},
\end{align*}
where $\nu_k \coloneqq \frac{ \mu \gamma_k }{3 G }\sqrt{\epsilon_{\mathrm{bias}}} + \frac{ 8 \gamma_k }{3 } \E \norm{d_k - \nabla_{\theta} J(\theta_k)} + \frac{L \gamma_k^2 }{2}+ \mathcal{O}\left(\frac{A Gt_{\mathrm{mix}}\log T}{T^3}\right)$.
\end{lemma}

The proof of Lemma \ref{lem:last_iterate_general} can be found in Appendix \ref{sec:last-iterate-proof}. Lemma \ref{lem:last_iterate_general} implies that we can obtain a bound for $J^*-\E [J(\theta_{K})]$ by bounding $\E \norm{d_k - \nabla_{\theta} J(\theta_k)}$. In the following lemmas, we establish bounds for this quantity for both algorithms. 
\begin{lemma}
    \label{lem:update_variance_IGT}
    Consider Algorithm \ref{alg:PG_IGT_Avg} and let all the assumptions stated in Theorem \ref{thm:alg1-last-iterate}
    hold. Then for all $K \geq 1$, we have
    \begin{align}
        \E \norm{d_{K} - \nabla_{\theta} J(\theta_{K})  }   \leq \cO\left(\frac{G t_{\mathrm{mix}} \log T}{K^{2/5}}+\frac{G^2 L_h}{\mu^2 K^{2/5}}\right).
    \end{align}
\end{lemma}
Using Lemma \ref{lem:update_variance_IGT}, we obtain the following last-iterate bound for Algorithm \ref{alg:PG_IGT_Avg}, for all $K\in \{1,\cdots,T/H\}$. This is stronger than typical results in PG literature which only provide bounds of form $\frac{1}{K}\sum_{k=1}^K (J^*-\E J(\theta_K))$.
\begin{theorem}[Last-iterate bound for Algorithm \ref{alg:PG_IGT_Avg}]
    \label{thm:alg1-last-iterate}
    Let $\{\theta_k\}_{k=1}^{K}$ be defined as in Algorithm \ref{alg:PG_IGT_Avg}. If Assumptions \ref{assump:ergodic_mdp}, \ref{assump:score_func_bounds}, \ref{assump:function_approx_error} and  \ref{assump:FND_policy} hold, $\nabla J(\theta)$ is $L_h$-smooth, $\gamma_k=\frac{6G}{\mu(k+2)}$ and $\eta_k = \left(\frac{2}{k+2}\right)^{4/5}$ then the following inequality holds for all $K\in \{1,\cdots,T/H\}$ and $H=63t_{\mathrm{mix}}t_{\mathrm{hit}}(\log_2 T)^2$ 
    \begin{align}
        \begin{split}
    J^*-\E [J(\theta_{K})] &\leq \sqrt{\epsilon_{\mathrm{bias}}} + \cO\left(\frac{A G^2 t_{\mathrm{mix}} \log T }{\mu K^{2/5}} +\frac{G^3 L_h }{\mu^3K^{2/5}}\right).
        \end{split}
    \end{align}
\end{theorem}

Similarly, for Algorithm \ref{alg:PG_Hessian_Avg}, we have
\begin{lemma}
    \label{lem:update_variance_Hessian}
    Consider Algorithm \ref{alg:PG_Hessian_Avg} and let all the assumptions stated in Theorem \ref{thm:alg2-last-iterate}
    hold. Then for all $K \geq 1$, we have
    \begin{align}
         \E \norm{d_{K} - \nabla_{\theta} J(\theta_{K})  } 
     \leq \cO\left(\frac{\sigma_g}{\sqrt{K}} + \frac{G M}{\sqrt{K}\mu}\right).
    \end{align}
\end{lemma}
With this, we obtain the following bound
\begin{theorem}[Last-iterate bound for Algorithm \ref{alg:PG_Hessian_Avg}]
    \label{thm:alg2-last-iterate}
    Let $\{\theta_k\}_{k=1}^{K}$ be defined as in Algorithm \ref{alg:PG_IGT_Avg}. If Assumptions \ref{assump:ergodic_mdp}, \ref{assump:score_func_bounds}, \ref{assump:function_approx_error} and  \ref{assump:FND_policy} hold, $\gamma_k=\frac{6G}{\mu(k+2)}$ and $\eta_k = \frac{2}{k+2}$ then the following inequality holds for all $K\in \{1,\cdots,T/H\}$ and $H=63t_{\mathrm{mix}}t_{\mathrm{hit}}(\log_2 T)^2$ 
    \begin{align*}
        \begin{split}
    &J^*-\E [J(\theta_{K})] \\
    &\leq \sqrt{\epsilon_{\mathrm{bias}}} + \cO\left(\frac{ \sqrt{A} G^2t_{\mathrm{mix}}\log T}{\mu \sqrt{K}} + \frac{\sqrt{A}G^4 t_{\mathrm{hit}}t^2_{\mathrm{mix}} (\log T)^{3/2}}{\mu^2\sqrt{K}} + \frac{\sqrt{A}(BG+G^3) t_{\mathrm{mix}}\log T}{\mu^2 \sqrt{K}}\right).
        \end{split}
    \end{align*}
\end{theorem}

The proof of Lemma \ref{lem:update_variance_IGT} is given in Appendix \ref{sec:igt-main-theorems-proof}, while the proof of Lemma \ref{lem:update_variance_Hessian} is given in Appendix \ref{sec:Hessian-main-theorem-proofs}. Their proofs roughly proceed by recursively bounding $\E \sqnorm{d_{K} - \nabla_{\theta} J(\theta_{K})}$. The challenges here include ensuring the accumulation of the bias is not too large and taking care of correlations in the cross-product terms involving previous estimates.

\section{Proof of Lemma \ref{lemma:advatge_estimate}}

Let $i$ denote the number of disjoint sub-trajectories of $\tau$ of length $N$ that start with the given state $s$ and are at least $N$ distance apart. The advantage function estimate can be written as:
    \begin{align}
        \label{def_A_hat_appndx}
        \hat{A}^{\pi}(\tau, s, a) = \begin{cases}
            \dfrac{1}{\pi(a|s)}\left[\dfrac{1}{i}\sum_{j=1}^i y_{j}\mathrm{1}(a_{\xi_j}=a)\right] - \dfrac{1}{i}\sum_{j=1}^i y_{j}~&\text{if}~i>0\\
            0~&\text{if}~i=0,
        \end{cases}
    \end{align}
    where $\xi_j$, $y_{j}$, as stated before, are the starting time and the sum of rewards in the $j$th sub-trajectory respectively. Using the definition of the $Q$ function, one can show that,
    \begin{align}
        \begin{split}
           \E\left[y_{j}\bigg|s_{\xi_j}=s, a_{\xi_j}=a\right]
           = Q^{\pi}(s, a) + NJ^{\pi} - \mathrm{E}^{\pi}_T(s, a),
        \end{split}
    \end{align}
    where $\mathrm{E}^{\pi}_T(s, a) \coloneqq \sum_{s'}P(s'|s, a)\left[\sum_{j=N}^{\infty}(P^{\pi})^j(s', \cdot)-d^{\pi}\right]^\top r^{\pi}$. Observe that,
    \begin{align}
        \label{eq_appndx_47}
        \begin{split}
            &\E\left[\left(\dfrac{1}{\pi(a|s)}y_{j}\mathrm{1}(a_{\xi_j}=a) - y_{j}\right)\bigg| s_{\xi_j}=s\right] \\
            &= \E\left[y_{j}\bigg| s_{\xi_j}=s, a_{\xi_j}=a\right] - \sum_{a'}\pi(a'|s)\E\left[y_{j}\bigg| s_{\xi_j}=s, a_{\xi_j}=a'\right]\\
            &=Q^{\pi}(s, a) + NJ^{\pi} - \mathrm{E}^{\pi}_T(s, a) - \sum_{a'}\pi(a'|s)[Q^{\pi}(s, a') + NJ^{\pi} - \mathrm{E}^{\pi}_T(s, a')]\\
            &=Q^{\pi}(s, a)-V^{\pi}(s)-\left[\mathrm{E}_T(s, a) - \sum_{a'}\pi(a'|s)\mathrm{E}_T^{\pi}(s, a')\right]\\
            &= A^{\pi}(s, a) -\Delta^{\pi}_T(s, a)
        \end{split}
    \end{align}
    where $\Delta^{\pi}_T(s, a)\coloneqq\mathrm{E}_T(s, a) - \sum_{a'}\pi(a'|s)\mathrm{E}_T^{\pi}(s, a')$. It follows from Lemma \ref{lem:aux-sum_N_dist} that $|\mathrm{E}^{\pi}_T(s, a)|\leq \frac{1}{T^6}$ which implies $|\Delta_T^{\pi}(s, a)|\leq \frac{2}{T^6}$. Therefore, one has the following bound.
\begin{align}
        \left|\E\left[\left(\dfrac{1}{\pi(a|s)}y_{j}\mathrm{1}(a_{\xi_j}=a) - y_{j}\right)\bigg| s_{\xi_j}=s\right] - A^{\pi}(s, a)\right|\leq \frac{2}{T^6}.
    \end{align}

   Since $i$ and the reward variables $\{y_{j}\}_{j=1}^i$ are correlated, we cannot directly obtain our desired result from the above bound. We use the methodology given in \citet{wei2020model} to address this challenge. In summary, we first prove bounds under an imaginary MDP where the state distribution 'refreshes' to the stationary distribution $d^{\pi}$ after $N$ time steps after the completion of a sub-trajectory. Under this MDP, $i$ becomes decoupled from $\{y_{j}\}_{j=1}^i$. The resulting bounds from this framework can then be translated into the real MDP given that $N$ is sufficiently large since this effectively makes the imaginary MDP `close' to the real MDP.

More specifically, for a sub-trajectory beginning at $\xi_{j}$ and ending at $\xi_j+N$, the system `rests' for $N$ additional steps before `refreshing' with the state distribution $d^{\pi}$ at $\xi_j+2N$. The wait time between the `refreshing' after the $(j-1)$th sub-trajectory and the onset of the $j$th sub-trajectory is denoted as $w_{j}=\xi_{j}-(\xi_{j-1}+2N)$ for all $j>1$. Additionally, $w_1$ represents the time between the initiation of the trajectory $\tau$ and the commencement of its first sub-trajectory.

It is pertinent to notice the following:
$(a)$ $w_1$ relies solely on the initial state of $\tau$, and the $\pi$-induced transition function, $P^{\pi}$,
$(b)$ For $j>1$, $w_j$ is solely contingent on the stationary distribution, $d^{\pi}$, and the induced transition function, $P^{\pi}$,
$(c)$ $i$ is solely dependent on $\{w_1, w_2, \cdots\}$, as other segments of $\tau$ maintain a fixed length of $2N$, $(d)$ the sequence $\{w_1, w_2, \cdots\}$ (and $i$ consequently) remains independent of $\{y_{1}, y_{2}, \cdots\}$. 

We denote expectation taken in this system as $\E'$ and probability of events similarly as ${\Pr}'$. From Lemma \ref{lemma_aux_bound_advantage}, it follows that $|A^{\pi}(s)|\leq \cO(t_{\mathrm{mix}})$. Now, define the following:
    \begin{align}
    \label{def_delta_i}
        \Delta_j \coloneqq \dfrac{y_{j}\mathrm{1}(a_{\xi_j}=a)}{\pi(a|s)} - y_{j} - A^{\pi}(s, a) + \Delta^{\pi}_T(s, a)
    \end{align}

\subsection{Proof of Lemma \ref{lemma:advatge_estimate}(a)}
Note that $|y_{j}| \leq N$ and as a result ${\E}'[|\Delta_j|^2\big|\{w_j\}]\leq \mathcal{O}(N^2/\pi(a|s))$. With this, we have
    \begin{align}
        \label{eq_appndx_50}
       \begin{split}
           &{\E}'\left[\left(\hat{A}^{\pi}(s, a) - A^{\pi}(s, a)\right)^2\right]\\ 
           &= {\E}'\left[\left(\hat{A}^{\pi}(s, a) - A^{\pi}(s, a)\right)^2\bigg| i>0\right]\times \mathrm{Pr}'(i>0) + \left(A^{\pi}(s, a)\right)^2\times \mathrm{Pr}'(i=0)\\
           & \overset{}{\leq} 2{\E}_{\{w_i\}}'\left[{\E}'\left[\left(\dfrac{1}{i}\sum_{j=1}^i\Delta_j \right)^2\bigg| \{w_j\}\right]\bigg| w_1\leq |\tau|-N\right]\times \mathrm{Pr}'(w_1\leq |\tau|-N) + 2\left(\Delta_T^{\pi}(s, a)\right)^2\\
           &+\left(A^{\pi}(s, a)\right)^2\times \mathrm{Pr}'(i=0)\\
           & \overset{(a)}{\leq} 2{\E}_{\{w_j\}}'\left[\dfrac{1}{i^2}\sum_{j=1}^i {\E}'\left[\Delta_j^2\big|\{w_j\}\right]\bigg| w_1\leq |\tau|-N\right]\times \mathrm{Pr}'(w_1\leq |\tau|-N) + \dfrac{8}{T^{12}} \\
           &+\left(A^{\pi}(s, a)\right)^2\times \mathrm{Pr}'(i=0)\\
           &\overset{(b)}{\leq} 2{\E}'\left[\dfrac{1}{i}\bigg|w_1\leq |\tau|-N\right]\mathcal{O}\left(\dfrac{N^2}{\pi(a|s)}\right)+\dfrac{8}{T^{12}}+\mathcal{O}(t_{\mathrm{mix}}^2)\times \mathrm{Pr}'(i=0),
       \end{split}
    \end{align}
    where $(a)$ utilizes the bound $|\Delta_T^{\pi}(s, a)|\leq \frac{2}{T^6}$ derived earlier, and the fact that $\{\Delta_j\}$ are zero mean random variables that are conditionally independent  given $\{w_j\}$. Inequality $(b)$ is a consequence of Lemma \ref{lemma_aux_bound_advantage}. Notice that $i=0$ is equivalent to $w_1>|\tau|-N$. Using Lemma \ref{lemma_aux_4}, we get:
    \begin{align}
        \label{eq:violation_prob}
        \begin{split}
        \mathrm{Pr}'(w_1>|\tau|-N)&\leq \left(1-\dfrac{3d^{\pi}(s)}{4}\right)^{\frac{|\tau|-N}{N}} \leq \left(1-\dfrac{3d^{\pi}(s)}{4}\right)^{9t_{\mathrm{hit}}\log_2 T-1}\\
        &\overset{(a)}{\leq} \left(1-\dfrac{3}{4t_{\mathrm{hit}}}\right)^{8t_{\mathrm{hit}}\log T}\leq \dfrac{1}{T^6}
        \end{split}
    \end{align}
    where $(a)$ is derived from the definition of $t_{\mathrm{hit}}$ and the inequality that $9t_{\mathrm{hit}}\log_2 T - 1 \geq 8t_{\mathrm{hit}}\log T$. Towards bounding ${\E}'\left[\frac{1}{i}\big|w_1\leq |\tau|-N\right]$, define
    \begin{align}
        i_0\coloneqq \dfrac{|\tau|-N}{2N+ \dfrac{4N\log_2 T}{d^{\pi}(s)}}
    \end{align}
    Note that $i<i_0$ is identical to the event that at least one $w_j$ is greater than $4N\log_2 T/d^{\pi}(s)$. From Lemma \ref{lemma_aux_4}, we thus obtain:
    \begin{align}
    \label{eq:m0-bound}
        \mathrm{Pr}'\left(i<i_0\right) \leq \left(1-\dfrac{3d^{\pi}(s)}{4}\right)^{\frac{4\log_2 T}{d^{\pi(s)}}}\leq \dfrac{1}{T^3}
    \end{align}

    Using this, we obtain
    \begin{align}
        \label{eq_appndx_55_}
        \begin{split}
            {\E}'\left[\dfrac{1}{i}\bigg| i>0\right]=\dfrac{\sum_{l=1}^{\infty}\dfrac{1}{l}\mathrm{Pr}'(i=l)}{\mathrm{Pr}'(i>0)}&\leq \dfrac{1\times \mathrm{Pr}'(i\leq i_0)+\dfrac{1}{i_0}\mathrm{Pr}'(i>i_0)}{\mathrm{Pr}'(i>0)}\\
            &\leq  \dfrac{\dfrac{1}{T^3}+\dfrac{2N+\dfrac{4N \log T}{d^{\pi}(s)}}{|\tau|-N}}{1-\dfrac{1}{T^6}}\leq \mathcal{O}\left(\dfrac{N\log T}{|\tau| d^{\pi}(s)}\right) 
        \end{split}
    \end{align}
    %where the last relation follows from the substitutions $|\tau|=63t_{\mathrm{mix}}t_{\mathrm{hit}}(\log_2 T)^2$, $N=7t_{\mathrm{mix}}\log_2 T$, and the inequality that $[d^{\pi}(s)]^{-1}\leq t_{\mathrm{hit}}$, $\forall s$. 
    Plugging in \eqref{eq:violation_prob} and \eqref{eq_appndx_55_} into \eqref{eq_appndx_50}, we get
    \begin{align}
        \label{eq_appndx_washim_48}{\E}'\left[\left(\hat{A}^{\pi}(s, a) - A^{\pi}(s, a)\right)^2\right] = \mathcal{O}\left(\dfrac{N^3\log T}{Hd^{\pi}(s)\pi(a|s)}\right)
    \end{align}

    We are now left with translating this result to the real MDP. Let $(\hat{A}^{\pi}(s, a)-A^{\pi}(s, a))^2=f(X)$ where $X=(i, \xi_1, \mathcal{T}_1, \cdots, \xi_i, \mathcal{T}_i)$, and $\mathcal{T}_j = (a_{\xi_j}, s_{\xi_j+1}, a_{\xi_j+1}, \cdots, s_{\xi_j+N}, a_{\xi_j+N})$. We have,
    \begin{align}
        \label{eq_appndx_57_}
        \dfrac{{\E}[f(X)]}{{\E}'[f(X)]} = \dfrac{\sum_{X} f(X)\mathrm{Pr}(X)}{\sum_{X} f(X)\mathrm{Pr}'(X)}\leq \max_{X}\dfrac{\mathrm{Pr}(X)}{\mathrm{Pr'}(X)}
    \end{align}

    The last inequality uses the non-negativity of $f(\cdot)$. Observe that, for a fixed $X$, we have,
    \begin{align*}
        \begin{split}
            \mathrm{Pr}(X) = &\mathrm{Pr}(\xi_1)\times \mathrm{Pr}(\mathcal{T}_1|\xi_1)\times \mathrm{Pr}(\xi_2|\xi_1, \mathcal{T}_1)\times \mathrm{Pr}(\mathcal{T}_2|\xi_2) \times \cdots \\
            &\times \mathrm{Pr}(\xi_i|\xi_{i-1}, \mathcal{T}_{i-1})\times \mathrm{Pr}(\mathcal{T}_i|\xi_i)\times \mathrm{Pr}(s_t\neq s, \forall t\in[\xi_i+2N, t_e-N]|\xi_i, \mathcal{T}_i),
        \end{split}\\
        \begin{split}
            \mathrm{Pr}'(X) = &\mathrm{Pr}(\xi_1)\times \mathrm{Pr}(\mathcal{T}_1|\xi_1)\times \mathrm{Pr}'(\xi_2|\xi_1, \mathcal{T}_1)\times \mathrm{Pr}(\mathcal{T}_2|\xi_2) \times \cdots \\
            &\times \mathrm{Pr}'(\xi_i|\xi_{i-1}, \mathcal{T}_{i-1})\times \mathrm{Pr}(\mathcal{T}_i|\xi_i)\times \mathrm{Pr}(s_t\neq s, \forall t\in[\xi_i+2N, t_e-N]|\xi_i, \mathcal{T}_i),
        \end{split}
    \end{align*}

    Thus, the difference between $\mathrm{Pr}(X)$ and $\mathrm{Pr}'(X)$ arises because $\mathrm{Pr}(\xi_{j+1}|\xi_j, \mathcal{T}_j)\neq \mathrm{Pr}'(\xi_{j+1}|\xi_j, \mathcal{T}_j)$, $\forall j\in\{1, \cdots, i-1\}$. Note that the ratio of these two terms can be bounded as follows.
    \begin{align}
    \label{eq:prob-bound-X}
        \begin{split}
            &\dfrac{\mathrm{Pr}(\xi_{j+1}|\xi_j, \mathcal{T}_j)}{\mathrm{Pr}'(\xi_{j+1}|\xi_j, \mathcal{T}_j)}\\
            &=\dfrac{\sum_{s'\neq s} \mathrm{Pr}(s_{\xi_j+2N}=s'|\xi_j, \mathcal{T}_j)\times \mathrm{Pr}(s_t\neq s, \forall t\in [\xi_j+2N, \xi_{j+1}-1], s_{\xi_{j+1}}=s|s_{\xi_j+2N}=s')}{\sum_{s'\neq s} \mathrm{Pr}'(s_{\xi_j+2N}=s'|\xi_j, \mathcal{T}_j)\times \mathrm{Pr}(s_t\neq s, \forall t\in [\xi_j+2N, \xi_{j+1}-1], s_{\xi_{j+1}}=s|s_{\xi_j+2N}=s')}\\
            &\leq \max_{s'}\dfrac{\mathrm{Pr}(s_{\xi_j+2N}=s'|\xi_j, \mathcal{T}_j)}{\mathrm{Pr}'(s_{\xi_j+2N}=s'|\xi_j, \mathcal{T}_j)}\\
            &=\max_{s'}1+\dfrac{\mathrm{Pr}(s_{\xi_j+2N}=s'|\xi_j, \mathcal{T}_j)-d^{\pi}(s')}{d^{\pi}(s')}\overset{(a)}{\leq} \max_{s'}1+\dfrac{1}{T^6 d^{\pi}(s')}\leq 1+\dfrac{t_{\mathrm{hit}}}{T^6}\leq 1+\dfrac{1}{T^5}
        \end{split}
    \end{align}
    where $(a)$ is a consequence of Lemma \ref{lem:aux-sum_N_dist}. We have,
    \begin{align}
        \label{eq_appndx_61_}
        \dfrac{\mathrm{Pr}(X)}{\mathrm{Pr}'(X)}\leq \left(1+\dfrac{1}{T^5}\right)^i\leq e^{\frac{i}{T^5}}\overset{(a)}{\leq} e^{\frac{1}{T^4}}\overset{(b)}{\leq} \left(1+\dfrac{2}{T^4}\right) 
    \end{align}
    where $(a)$ uses the fact that $i\leq T$ while $(b)$ is a consequence of the inequality that $\exp(x)\leq 1+2x$ for sufficiently small $x>0$. Combining $(\ref{eq_appndx_57_})$ and $(\ref{eq_appndx_61_})$, we get,
    \begin{align}
    \label{eq:adv-bound}
        \begin{split}
            {\E}\left[\left(\hat{A}^{\pi}(s, a) - A^{\pi}(s, a)\right)^2\right]&\leq \mathcal{O}\left(1+\dfrac{1}{T^4}\right){\E}'\left[\left(\hat{A}^{\pi}(s, a) - A^{\pi}(s, a)\right)^2\right]\\
            &\overset{(a)}{\leq} \mathcal{O}\left(\dfrac{N^3\log T}{Hd^{\pi}(s)\pi(a|s)}\right)
        \end{split}
    \end{align}
    where $(a)$ follows from $(\ref{eq_appndx_washim_48})$. %and the substitution $N=7t_{\mathrm{mix}}\log_2T$.

\subsection{Proof of Lemma \ref{lemma:advatge_estimate}(b)}
     Observe the following relations.
\begin{align}
    \label{eq:img_adv_bound}
       \begin{split}
           &\bigg|\E'\left[\hat{A}^{\pi}(s, a) - A^{\pi}(s, a)\right]\bigg|\\ 
           &= \bigg|\E'\left[\hat{A}^{\pi}(s, a) - A^{\pi}(s, a)\bigg| i>0\right]\times {\Pr}'(i>0) + A^{\pi}(s, a)\times {\Pr}'(i=0) \bigg|\\
           &= \left|\E'\left[\dfrac{1}{i}\sum_{j=1}^i\Delta_j - \Delta_T^{\pi}(s, a)\bigg| i>0\right]\times {\Pr}'(i>0) + A^{\pi}(s, a)\times {\Pr}'(i=0)\right|\\
           &\overset{(a)}{=} \left|\E'\left[-\Delta_T^{\pi}(s, a)\bigg| i>0\right]\times {\Pr}'(i>0) + A^{\pi}(s, a)\times {\Pr}'(i=0)\right|\\
           &\leq \E'\left[ |\Delta_T^{\pi}(s, a)|\bigg| i>0\right]+ |A^{\pi}(s, a)|\times {\Pr}'(i=0) \overset{(b)}{=} \cO(t_{\mathrm{mix}}T^{-6}),
       \end{split}
    \end{align}

 where $(a)$ follows from the fact that $\{\Delta_j\}$ are zero mean random variables and $(b)$ utilizes Lemma \ref{lemma_aux_bound_advantage} and the bound $|\Delta_T^{\pi}(s, a)|\leq \frac{2}{T^6}$ that was derived earlier. Define $X=(i, \xi_1, \mathcal{T}_1, \cdots, \xi_i, \mathcal{T}_i)$ and $f(X) = \hat{A}^{\pi}(s, a) - A^{\pi}(s, a)$ where $\mathcal{T}_j = (a_{\xi_j}, a_{\xi_j}, s_{\xi_j+1}, a_{\xi_j+1}, \cdots, s_{\xi_j+N}, a_{\xi_j+N})$. For any $f(\cdot)$ (not necessarily non-negative), observe that
\begin{align*}
    |\E [f(X)] - \E'[f(X)]| &= |\sum_X f(X)\Pr(X) - \sum_X f(X){\Pr}'(X) |  \\
    &\leq \sum_X |f(X)| |\Pr(X) - {\Pr}'(X)|. 
\end{align*}

We now focus on bounding $|\Pr(X) - {\Pr}'(X)|$. We have the result stated below from \eqref{eq_appndx_61_}.
    \begin{align}
    \label{eq_appndx_washim_54}
        \Pr(X) \leq \left(1+\frac{2}{T^4}\right){\Pr}'(X)
    \end{align}
    Furthermore, one can derive the following lower bound using Lemma \ref{lem:aux-sum_N_dist}.
  \begin{align}
        \begin{split}
            &\dfrac{{\Pr}(\xi_{j+1}|\xi_j, \mathcal{T}_j)}{{\Pr}'(\xi_{j+1}|\xi_j, \mathcal{T}_j)} \\
            &=\dfrac{\sum_{s'\neq s} {\Pr}(s_{\xi_j+2N}=s'|\xi_j, \mathcal{T}_j)\times {\Pr}(s_t\neq s, \forall t\in [\xi_j+2N, \xi_{j+1}-1], s_{\xi_{j+1}}=s|s_{\xi_j+2N}=s')}{\sum_{s'\neq s} {\Pr}'(s_{\xi_j+2N}=s'|\xi_j, \mathcal{T}_j)\times {\Pr}(s_t\neq s, \forall t\in [\xi_j+2N, \xi_{j+1}-1], s_{\xi_{j+1}}=s|s_{\xi_j+2N}=s')}\\
            &\geq \min_{s'}\dfrac{{\Pr}(s_{\xi_j+2N}=s'|\xi_j, \mathcal{T}_j)}{{\Pr}'(s_{\xi_j+2N}=s'|\xi_j, \mathcal{T}_j)}\\
            &=\min_{s'}1+\dfrac{{\Pr}(s_{\xi_j+2N}=s'|\xi_j, \mathcal{T}_j)-d^{\pi}(s')}{d^{\pi}(s')}\geq \min_{s'}1-\dfrac{1}{T^6 d^{\pi}(s')}\geq 1-\dfrac{t_{\mathrm{hit}}}{T^6}\geq 1-\dfrac{1}{T^5}
        \end{split}
    \end{align}
  By utilising a similar argument used in deriving \eqref{eq_appndx_61_}, we arrive at the following:
\begin{align}
        \dfrac{{\Pr}(X)}{{\Pr}'(X)}\geq \left(1-\dfrac{1}{T^5}\right)^i\overset{(a)}{\geq} e^{-\frac{2i}{T^5}} \geq e^{-\frac{2}{T^4}}\overset{(b)}{\geq} 1-\dfrac{4}{T^4}.
    \end{align}
where $(a)$ and $(b)$ are consequences of the inequality that $1-x\geq \exp(-2x)\geq 1-4x$ for sufficiently small $x>0$. Therefore, we arrive at the following.
\begin{align*}
   \left(1-\dfrac{4}{T^4}\right){\Pr}'(X) \leq {\Pr}(X) \leq  \left(1+\dfrac{2}{T^4}\right){\Pr}'(X)
\end{align*}
which is equivalent to the following relation.
\begin{align*}
   -\dfrac{4}{T^4} \cdot {\Pr}'(X) \leq {\Pr}(X)-{\Pr}'(X) \leq  \dfrac{2}{T^4} \cdot {\Pr}'(X)
\end{align*}
It follows that
\begin{align}
\begin{split}
    \sum_X |f(X)| |\Pr(X) - {\Pr}'(X)| &\leq \frac{4}{T^4} \sum_X |f(X)| {\Pr}'(X) \\
    &= \frac{4}{T^4} \E' [|f(X)|] \\
    &= \frac{4}{T^4} \E' [|\hat{A}^{\pi}(s, a) - A^{\pi}(s, a)|] \\
    &\leq \frac{4}{T^4} \E' [(\hat{A}^{\pi}(s, a) - A^{\pi}(s, a))^2]^{1/2} \overset{(a)}{\leq} \mathcal{O}\left(\dfrac{t_{\mathrm{mix}}\log T}{\sqrt{\pi(a|s)}T^4}\right)
\end{split}
\end{align}
where $(a)$ results from \eqref{eq_appndx_washim_48}. We finally have
\begin{align*}
    \bigg|\E\left[\hat{A}^{\pi}(s, a) - A^{\pi}(s, a)\right]\bigg| = \big| \E\left[f(X)\right]\big| &\leq \big| \E'\left[f(X)\right]\big| +  \sum_X |f(X)| |\Pr(X) - {\Pr}'(X)|\\
    &\leq \mathcal{O}\left(\dfrac{t_{\mathrm{mix}}\log T}{\sqrt{\pi(a|s)}T^4}\right)
\end{align*}
This concludes the proof.

\section{Proof of Lemma \ref{lem:grad+hess-est-prop}}
\label{sec:lemma-grad-hess-proof}

\subsection{Proof of Lemma \ref{lem:grad+hess-est-prop}(a)}
Assume the trajectory being used for estimation is $\tau=\{(s_t, a_t)\}_{t=t_s}^{t_e}$. Observe that
\begin{align}
\label{eq:eq40}
    \begin{split}
        \E\|g(\theta,\tau) - \nabla_{\theta} J(\theta)\|^2 &= \E\|g(\theta,\tau) - \bar{g}(\theta,\tau)+ \bar{g}(\theta,\tau) -\nabla_{\theta} J(\theta)\|^2 \\
        &\leq 2\E\|g(\theta,\tau) - \bar{g}(\theta,\tau)\|^2+ 2\E\|\bar{g}(\theta,\tau) - \nabla_{\theta} J(\theta)\|^2
    \end{split}
\end{align}
where $\bar{g}(\theta,\tau) = \dfrac{1}{|\tau|-N}\sum_{t=t_s+N}^{t_{e}}A^{\pi_{\theta}}(s_{t}, a_{t})\nabla_{\theta}\log \pi_{\theta}(a_{t}|s_{t})$. We apply the following lemma to bound the second term in the RHS of \eqref{eq:eq40}:
\begin{lemma}
    \label{lemma_aux_6}
    \citep[Lemma A.6]{dorfman2022adapting} Fix a trajectory $\tau=\{(s_t, a_t)\}_{t\in\mathbb{N}}$ generated by following $\pi_{\theta}$ starting from some initial state $s_0\sim\bar{\rho}$. Let $\nabla L(\theta)$ be a gradient that we wish to estimate over $\tau$, and $l(\theta, \cdot, \cdot)$ is a function such that $\E_{s\sim d^{\pi_{\theta}}, a\sim \pi_{\theta}(s)}l(\theta, s, a)=\nabla L(\theta)$. Let $\norm{l(\theta, s, a)}, \norm{\nabla L(\theta)}\leq G_L$, $\forall \theta\in\Theta$, $\forall s\in \mathcal{S}, \forall a\in \mathcal{A}$. Define $l^{Q}\triangleq\frac{1}{Q}\sum_{i=1}^Q l(\theta, s_i, a_i)$. If $P=2t_{\mathrm{mix}}\log T$, then the following holds as long as $Q\leq T$,
    \begin{align}
        \E\left[\norm{l^{Q}-\nabla L(\theta)}^2\right]\leq \mathcal{O}\left(G_L^2\log\left(PQ\right)\dfrac{P}{Q}\right).
    \end{align}
\end{lemma}

Applying Lemma \ref{lemma_aux_6}, and \ref{lemma_aux_bound_advantage}, we get
    \begin{align}
        \label{eq_appndx_67_}
        \E\left[\norm{\bar{g}(\theta,\tau)-\nabla_{\theta}J(\theta)}^2\right]\leq \mathcal{O}\left(G^{2}t^2_{\mathrm{mix}}\log T\right)\times \mathcal{O}\left(\dfrac{t_{\mathrm{mix}}\log T}{|\tau|-N}\right)=\mathcal{O}\left(\dfrac{G^2t_{\mathrm{mix}}^3(\log T)^2}{H}\right)
    \end{align}

The first term can be bounded using Assumption \ref{assump:score_func_bounds} as follows.
\begin{align}
    &\E\|g(\theta,\tau) - \bar{g}(\theta,\tau)\|^2 \nonumber\\
    &=\dfrac{1}{(|\tau|-N)^2}\E\sqnorm{\sum_{t=t_s+N}^{t_{e}}\hat{A}^{\pi_{\theta}}(s_{t}, a_{t})\nabla_{\theta}\log \pi_{\theta}(a_{t}|s_{t}) - \sum_{t=t_s+N}^{t_{e}}A^{\pi_{\theta}}(s_{t}, a_{t})\nabla_{\theta}\log \pi_{\theta}(a_{t}|s_{t})} \nonumber\\
    &\leq  \left(\dfrac{G^2}{|\tau|-N}\right)\E\left[\sum_{t=t_s+N}^{t_{e}}\left[\hat{A}^{\pi_{\theta}}(\tau, s_{t}, a_{t})-A^{\pi_{\theta}}(s_{t}, a_{t})\right]^2\right]\nonumber \\
    &= \left(\dfrac{G^2}{|\tau|-N}\right)\E\left[\sum_{t=t_s+N}^{t_{e}} \sum_{a_t \in \cA} \pi_{\theta}(a_t|s_t)\E[(\hat{A}^{\pi_{\theta}}(\tau, s_{t}, a_{t})-A^{\pi_{\theta}}(s_{t}, a_{t}))^2|s_t,a_t]\right]\nonumber\\
    &=\cO\left(\dfrac{A G^2 t_{\mathrm{hit}}N^3\log T}{H}\right)
\end{align}
where the last equality follows from Lemma \ref{lemma:advatge_estimate}.

\subsection{Proof of Lemma \ref{lem:grad+hess-est-prop}(b)}
We begin by observing that, for a $\pi_\theta$-induced trajectory $\tau=\{(s_t, a_t)\}_{t=t_s}^{t_e}$, the following holds.
\begin{align}
        \label{eq:exp_gradient_est-1}
       \begin{split}
           &\|\E [g(\theta,\tau)] - \nabla_{\theta} J(\theta)\|\\ 
            &=\norm{\E\left[\dfrac{1}{|\tau|-N}\sum_{t=t_s+N}^{t_{e}}\hat{A}^{\pi_{\theta}}(\tau, s_{t}, a_{t})\nabla_{\theta}\log \pi_{\theta}(a_{t}|s_{t})\right] - \nabla_{\theta} J(\theta) }\\
        &\leq \dfrac{1}{|\tau|-N}\underbrace{\norm{\E\left[\sum_{t=t_s+N}^{t_{e}}\hat{A}^{\pi_{\theta}}(\tau, s_{t}, a_{t})\nabla_{\theta}\log \pi_{\theta}(a_{t}|s_{t}) - \sum_{t=t_s+N}^{t_{e}}A^{\pi_{\theta}}(s_{t}, a_{t})\nabla_{\theta}\log \pi_{\theta}(a_{t}|s_{t}) \right]}}_{T_1} \\
        &+ \underbrace{\norm{\E\left[\dfrac{1}{|\tau|-N}\sum_{t=t_s+N}^{t_{e}}A^{\pi_{\theta}}(s_{t}, a_{t})\nabla_{\theta}\log \pi_{\theta}(a_{t}|s_{t}) \right] - \nabla_{\theta} J(\theta)}}_{T_2}
       \end{split}
    \end{align}
Note that the term $T_1$ can be bounded as follows.
\begin{align}
    \begin{split}
    T_1
    &\leq  \E\left[\sum_{t=t_s+N}^{t_{e}}\bigg|\E\left[\hat{A}^{\pi_{\theta}}(\tau, s_{t}, a_{t})-A^{\pi_{\theta}}(s_{t}, a_{t})|s_t,a_t\right]\bigg| \left\|\nabla_{\theta}\log \pi_{\theta}(a_{t}|s_{t})\right\|\right] \\
    &\overset{(a)}{\leq}  G \E\left[\sum_{t=t_s+N}^{t_{e}}|\E[\hat{A}^{\pi_{\theta}}(\tau, s_{t}, a_{t})-A^{\pi_{\theta}}(s_{t}, a_{t})|s_t,a_t]|\right]\\
    &\leq  G \E\left[\sum_{t=t_s+N}^{t_{e}} \sum_{a_t \in \cA} \pi_{\theta}(a_t|s_t)|\E[\hat{A}^{\pi_{\theta}}(s_{t}, a_{t})-A^{\pi_{\theta}}(s_{t}, a_{t})|s_t,a_t]|\right]\\
    &\overset{(b)}{\leq} \mathcal{O}\left(\dfrac{Gt_{\mathrm{mix}}\log T}{T^4}\right)\E\left[\sum_{t=t_s+N}^{t_{e}} \sum_{a_t \in \cA} \sqrt{\pi_{\theta}(a_t|s_t)}\right] \leq \mathcal{O}\left(\dfrac{\sqrt{A}Gt_{\mathrm{mix}}\log T}{T^4}\right)(|\tau|-N) 
    \end{split}
\end{align}
where $(a)$, $(b)$ follow from Assumption \ref{assump:score_func_bounds} and Lemma \ref{lemma:advatge_estimate} respectively. To bound $T_2$, define
\begin{align}
    f(s) \coloneqq \dfrac{1}{|\tau|-N}\E\left[\sum_{t=t_s+N}^{t_e}A^{\pi_{\theta}}(s_{t}, a_{t})\nabla_{\theta}\log \pi_{\theta}(a_{t}|s_{t}) \bigg| s_{t_s+N} = s \right], ~\forall s\in\mathcal{S}
\end{align}
Using the stationarity property of $d^{\pi_\theta}$ and the policy gradient expression \eqref{eq:grad-expression}, we deduce
\begin{align}
    \nabla_{\theta} J(\theta) = \sum_{s\in\mathcal{S}} d^{\pi_\theta}(s) f(s)
\end{align}
If $\tau = \{(s_t, a_t)\}_{t=t_s}^{t_e}$ is a $\pi_\theta$-induced trajectory with $s_{t_s}\sim \bar{\rho}$, then $s_{t_s+N}\sim (P^{\pi_\theta})^N\bar{\rho}$. Hence,
\begin{align}
    \begin{split}
     T_2 &=\norm{\sum_{s\in S} \left[((P^{\pi_\theta})^N\bar{\rho}-d^{\pi_{\theta}})(s)\right] f(s)}\\
    &\leq \sum_{s \in S} \left|\left\{\sum_{s'\in\mathcal{S}}(P^{\pi_\theta})^N(s',s)\bar{\rho}(s')\right\}-\left\{d^{\pi_{\theta}}(s)\sum_{s'\in\mathcal{S}}\bar{\rho}(s')\right\}\right| \|f(s)\|\\
    &\leq \sum_{s \in S}\sum_{s' \in S} \left|(P^{\pi_\theta})^N(s',s)-d^{\pi_{\theta}}(s)\right|\bar{\rho}(s') \|f(s)\|\\
    &\leq \sum_{s' \in S} \bar{\rho}(s') \left\Vert (P^{\pi_\theta})^N(s', \cdot)-d^{\pi_{\theta}}\right\Vert_1 \max_{s\in\mathcal{S}} \|f(s)\|\\
    &\overset{(a)}{\leq} \sum_{s' \in S} \bar{\rho}(s') \cdot 2\cdot 2^{-\frac{N}{t_{\mathrm{mix}}}} \mathcal{O}(Gt_{\mathrm{mix}}) = \mathcal{O}\left(\dfrac{Gt_{\mathrm{mix}}}{T^7}\right)
     \end{split}
\end{align}
where $(s)$ uses Lemma \ref{lemma_aux_bound_advantage}, \ref{lemma_aux_mixing_dist} and Assumption \ref{assump:score_func_bounds}. Combining the bounds of $T_1$ and $T_2$, we conclude the result.

\subsection{Proof of Lemma \ref{lem:grad+hess-est-prop}(c)}

Using the expression of $B(\theta, \tau)$ given in \eqref{eq:Hessian-estimate}, we can derive the following.
    \begin{align}
    \begin{split}
         \E\|B(\theta,\tau)\|^2 &= \E \|\nabla_{\theta} \Phi(\theta,\tau)\nabla_{\theta}\log p(\tau, \theta, \bar{\rho})^\top  + \nabla_{\theta}^2 \Phi (\theta,\tau)\|^2 \\
         &\leq 2\E\|\nabla_{\theta} \Phi (\theta,\tau)\|^2\|\nabla_{\theta}\log p(\tau, \theta, \bar{\rho})\|^2 + 2\E\|\nabla_{\theta}^2 \Phi (\theta,\tau)\|^2
    \end{split}
    \end{align}
 Observe the inequality stated below.
    \begin{align}
        \|\nabla_{\theta}\log p(\tau, \theta, \bar{\rho})\|^2 = \sqnorm{\sum_{t=t_s}^{t_e} \nabla_{\theta} \log \pi_{\theta} (a_t|s_t) } \leq |\tau| \sum_{t=t_s}^{t_e} \sqnorm{\nabla_{\theta} \log \pi_{\theta} (a_t|s_t) }
        \leq G^2 |\tau|^2
    \end{align}
    where the last inequality uses Assumption \ref{assump:score_func_bounds}. It follows that
     \begin{align}
         \E\|B(\theta,\tau)\|^2 
         &\leq 2G^2 |\tau|^2 \E\|\nabla  \Phi (\theta,\tau)\|^2 + 2\E\|\nabla^2 \Phi (\theta,\tau)\|^2
    \end{align}

Using \eqref{eq:appndx_27}, one can write
\begin{align}
\begin{split}
    \E\|\nabla  \Phi (\theta,\tau)\|^2  &= \E\sqnorm{\dfrac{1}{|\tau|-N}\sum_{t=t_s+N}^{t_e}\hat{A}^{\pi_{\theta}}(\tau, s_{t}, a_{t})\nabla_{\theta}\log \pi_{\theta}(a_{t}|s_{t})} \\
    &\leq \dfrac{1}{|\tau|-N}\sum_{t=t_s+N}^{t_e}\E\sqnorm{\hat{A}^{\pi_{\theta}}(s_{t}, a_{t})\nabla_{\theta}\log \pi_{\theta}(a_{t}|s_{t})}\\
    &\overset{(a)}{\leq} \dfrac{G^2}{|\tau|-N}\sum_{t=t_s+N}^{t_e}\E|\hat{A}^{\pi_{\theta}}(s_{t}, a_{t})|^2\\
    &\overset{(b)}{\leq} \dfrac{G^2}{|\tau|-N}\sum_{t=t_s+N}^{t_e}\E\left[\frac{N^2}{\pi_{\theta}(a_t|s_t)}\right] \\
    &\leq \dfrac{G^2}{|\tau|-N}\sum_{t=t_s+N}^{t_e}\E\left[\sum_{a\in A}\pi_{\theta}(a|s_t)\E\left[\frac{N^2}{\pi_{\theta}(a|s_t)}\right]\right]\leq AN^2G^2
\end{split}
\end{align}
where $(a)$ uses Assumption \ref{assump:score_func_bounds} and $(b)$ follows from the discussion in the proof of Lemma \ref{lemma:advatge_estimate}(a). Also,
\begin{align}
\begin{split}
&\E \|\nabla_{\theta}^2 \Phi(\theta,\tau) \|^2\\
&=  \E  \left\Vert\dfrac{1}{|\tau|-N}\sum_{t=t_s+N}^{t_e} \left\{\Psi_t^{(1)}(\tau) \nabla_{\theta}^2 \log \pi_{\theta}(a_t | s_t)\right.\right. \\
&\left.\left.\hspace{2cm}- \Psi_t^{(2)}(\tau)\left(\frac{\nabla_{\theta}^2\log \pi_{\theta}(a_t|s_t)}{\pi_{\theta}(a_t | s_t)}-\frac{\nabla_{\theta} \log \pi_{\theta} (a_t|s_t) \nabla_{\theta} \log \pi_{\theta} (a_t|s_t)^\top }{\pi_{\theta}(a_t | s_t)}\right)\right\}\right\Vert^2 \\
&=  \dfrac{1}{|\tau|-N}\sum_{t=t_s+N}^{t_e}\E  \left\Vert \Psi_t^{(1)}(\tau) \nabla_{\theta}^2 \log \pi_{\theta}(a_t | s_t)\right. \\
&\left.\hspace{2cm}- \Psi_t^{(2)}(\tau)\left(\frac{\nabla_{\theta}^2\log \pi_{\theta}(a_t|s_t)}{\pi_{\theta}(a_t | s_t)}-\frac{\nabla_{\theta} \log \pi_{\theta} (a_t|s_t) \nabla_{\theta} \log \pi_{\theta} (a_t|s_t)^\top }{\pi_{\theta}(a_t | s_t)}\right)\right\Vert^2 \\
& \\
&\leq \frac{3}{|\tau|-N}\sum_{t=t_s+N}^{t_e} \left\{ B^2 \E \left[\sqnorm{\Psi_t^{(1)}(\tau)}\right] + \E \left[ \sqnorm{\Psi_t^{(2)}(\tau)}\left(\frac{B^2}{\pi^2_{\theta}(a_t | s_t)}+\frac{G^4}{\pi^2_{\theta}(a_t | s_t)}\right)\right]\right\}
\end{split}
\end{align}

It follows from the definitions of $\{\Psi_t^{(k)}(\tau)\}_{k\in\{0, 1\}}$ that 
\begin{align}
     \E \left[\sqnorm{\Psi_t^{(1)}(\tau)}\right] \leq N^2~\text{and}~~\E \left[\sqnorm{\Psi_t^{(2)}(\tau)}\bigg|s_t, a_t\right] \leq N^2\pi_{\theta}(a_t|s_t)
\end{align}
Using the above inequalities, one can derive the relation stated below. 
\begin{align}
\begin{split}
    \E\|\nabla^2 \Phi(\theta,\tau) \|^2 &\leq    3B^2 N^2 + \frac{3}{|\tau|-N} \sum_{t=t_s+N}^{t_e} N^2  \E\left[\frac{B^2+G^4}{\pi_{\theta}(a_t | s_t)}\right] \\
    &\leq    3 B^2 N^2  + \frac{3}{|\tau|-N} \sum_{t=t_s+N}^{t_e} N^2  \E\left[\sum_{a \in \cA} \pi_{\theta}(a|s_t)\left[\frac{B^2+G^4}{\pi_{\theta}(a | s_t)} \right]\right]\\
     &\leq    3 B^2 N^2  + 3 A N^2  (B^2+G^4) 
\end{split}
\end{align}

Combining the above results, we finally obtain
\begin{align}
    \begin{split}
        \E\|B(\theta,\tau)\|^2 &\leq 2A N^2 G^4 |\tau|^2 + 6B^2 N^2  + 6 A N^2  (B^2+G^4)\\
        %&= \mathcal{O}\left(AB^2t_{\mathrm{mix}}^2(\log T)^2+AG^4t_{\mathrm{mix}}^4t_{\mathrm{hit}}^2(\log T)^4\right)    
    \end{split}
\end{align}

\section{Proof of Lemma \ref{lem:exp-regret}}
\label{sec:regret-decomp-proof}
The regret for Algorithm \ref{alg:PG_IGT_Avg} can be decomposed as:
\begin{align}
    \label{reg_decompose}
        &\mathrm{Reg}_T = \sum_{t=0}^{T-1} \left(J^* - r(s_t, a_t)\right)=H\sum_{k=1}^{K}\left(J^*-J({\tilde{\theta}_k})\right)+\sum_{k=1}^{K}\sum_{t\in\mathcal{I}_k} \left(J(\tilde{\theta}_k)-r(s_t, a_t)\right)
\end{align}
where $\mathcal{I}_k\triangleq \{(k-1)H, \cdots, kH-1\}$. 

Whereas, the regret for Algorithm \ref{alg:PG_Hessian_Avg} is (expressed below) slightly different since the trajectories are sampled using the sequence $\{\theta_1,\hat{\theta}_1,\theta_2,\hat{\theta}_2,\cdots\}$, instead of $\{\tilde{\theta}_1,\tilde{\theta}_2,\cdots\}$:
\begin{align}
\label{eq:decompose_2}
    \begin{split}
        \mathrm{Reg}_T = \sum_{t=0}^{T-1} \left(J^* - r(s_t,a_t)\right)&=H\sum_{k=1}^{K}\left(J^*-J({\theta_k})\right)+\sum_{k=1}^{K}\sum_{t\in\mathcal{I}_k} \left(J(\theta_k)-r(s_t, a_t)\right)\\
        &+H\sum_{k=1}^{K}\left(J^*-J({\hat{\theta}_k})\right)+\sum_{k=1}^{K}\sum_{t\in\mathcal{I}_k} \left(J(\hat{\theta}_k)-r(s_t, a_t)\right)\\
        &\overset{(a)}{=}\cO\left(H\sum_{k=1}^{K}\left(J^*-J({\theta_k})\right)+\sum_{k=1}^{K}\sum_{t\in\mathcal{I}_k} \left(J(\theta_k)-r(s_t, a_t)\right)\right),
        \end{split}
\end{align}
where the proof of $(a)$ in Equation \eqref{eq:decompose_2} is provided in Section \ref{sec:Hessian-main-theorem-proofs}. Thus, we can focus on the decomposition of the form in \eqref{reg_decompose}. The expectation of the second term in \eqref{reg_decompose} can be expressed as follows,
\begin{align}
    \label{eq_38}
    \begin{split}
        \E\left[\sum_{k=1}^{K}\sum_{t\in\mathcal{I}_k} \left(J(\tilde{\theta}_k)-r(s_t, a_t)\right)\right]&\overset{(a)}{=}\E\left[\sum_{k=1}^{K}\sum_{t\in\mathcal{I}_k} \E_{s'\sim P(\cdot|s_t, a_t)}[V^{\pi_{\tilde{\theta}_k}}(s')]-Q^{\pi_{\tilde{\theta}_k}}(s_t, a_t)\right]\\
        &\overset{(b)}{=}\E\left[\sum_{k=}^{K}\sum_{t\in\mathcal{I}_k} V^{\pi_{\tilde{\theta}_k}}(s_{t+1})-V^{\pi_{\tilde{\theta}_k}}(s_t)\right]\\
        &=\E\left[\sum_{k=1}^{K} V^{\pi_{\tilde{\theta}_k}}(s_{kH})-V^{\pi_{\tilde{\theta}_k}}(s_{(k-1)H})\right]\\
        &=\underbrace{\E\left[\sum_{k=1}^{K-1} V^{\pi_{\tilde{\theta}_{k+1}}}(s_{kH})-V^{\pi_{\tilde{\theta}_k}}(s_{kH})\right]}_{T_3}\\
        &+\underbrace{\E\left[ V^{\pi_{\tilde{\theta}_K}}(s_{T})-V^{\pi_{\tilde{\theta}_0}}(s_{0})\right]}_{T_4}
    \end{split}
\end{align}
where $(a)$ follows from Bellman equation and $(b)$ follows from the fact that $\E[V^{\pi_{\tilde{\theta}_k}}(s_{t+1})] = \E_{s'\sim P(\cdot|s_t, a_t)}[V^{\pi_{\tilde{\theta}_k}}(s')]$ and $\E[V^{\pi_{\tilde{\theta}_k}}(s_{t})]=\E[Q^{\pi_{\tilde{\theta}_k}}(s_t, a_t)]$. The term $T_3$ in \eqref{eq_38} can be bounded as in \citet{wei2020model,bai2023regret}. This is given in Lemma \ref{lem:stability} (as stated below). Furthermore, the term $T_4$ can be upper-bounded as $\mathcal{O}(t_{\mathrm{mix}})$ since $V^{\pi}(\cdot)\leq 5t_{\mathrm{mix}}$ \citet{wei2020model}.

\begin{lemma}
    \label{lem:stability}
    If Assumptions \ref{assump:ergodic_mdp} and \ref{assump:score_func_bounds} hold, then the following inequalities are true $\forall k$, $\forall (s, a)\in\mathcal{S}\times \mathcal{A}$
    \begin{enumerate}[label=(\alph*)]
        \item $ |\pi_{\theta_{k+1}}(a|s)-\pi_{\theta_{k}}(a|s)|\leq \cO \left(G\pi_{\bar{\theta}_k}(a|s)\|\theta_{k+1}-\theta_k\|\right)$
      
        \item $ \sum_{k=1}^{K-1}\E|J(\theta_{k+1})-J(\theta_k)|\leq  \cO\left(Gt_{\mathrm{mix}} \sum_{k=1}^{K}\|\theta_{k+1}-\theta_k\|\right)$
       
        \item$\sum_{k=1}^K\E|V^{\pi_{\theta_{k+1}}}(s) - V^{\pi_{\theta_{k}}}(s)| \leq \cO\left(Gt_{\mathrm{mix}}^2\log^2 T\sum_{k=1}^K\|\theta_{k+1}-\theta_k\|\right),$
    \end{enumerate}
    where $\bar{\theta}_k$ is some convex combination of $\theta_k$ and $\theta_{k+1}$.
\end{lemma}
Lemma \ref{lem:stability} can be understood as providing stability insights into our algorithm. It is similar to Lemma 7 in \citet{wei2020model}. Essentially, it asserts that the policy parameters undergo updates in a manner that reduces the average difference between consecutive average reward and value functions as the number of iterations, denoted by $k$, increases. 

\begin{proof}[Proof of Lemma \ref{lem:stability}]
    Using Taylor's expansion, we can write the following $\forall (s, a)\in \mathcal{S}\times \mathcal{A}$, $\forall k$.
    \begin{align}
        \label{eq_pi_lipschitz}
        \begin{split}
            |\pi_{\theta_{k+1}}(a|s)-\pi_{\theta_{k}}(a|s)|&=\left|(\theta_{k+1}-\theta_k)^\top \nabla_{\theta}\pi_{\bar\theta}(a|s) \right| \\&=\pi_{\bar{\theta}_k}(a|s)\left|(\theta_{k+1}-\theta_k)^\top \nabla_{\theta}\log \pi_{\bar{\theta}_k}(a|s) \right|\\
            &\leq \pi_{\bar{\theta}_k}(a|s) \norm{\theta_{k+1}-\theta_k}\norm{\nabla_{\theta}\log \pi_{\bar{\theta}_k}(a|s)}\overset{(a)}{\leq} G\pi_{\bar{\theta}_k}(a|s)\norm{\theta_{k+1}-\theta_k}
        \end{split} 
    \end{align}
    where $\bar{\theta}_k$ is some convex combination of $\theta_{k}$ and $\theta_{k+1}$ and $(a)$ follows from Assumption \ref{assump:score_func_bounds}. This concludes the first statement. 
    
    \begin{lemma}
    \label{lemma_aux_5}
    \citep[Lemma 15]{wei2020model} For two policies $\pi$ and $\pi'$, the difference of the objective function $J$ is
    \begin{equation}
        J^{\pi}-J^{\pi'}=\sum_{s}\sum_{a}d^{\pi}(s)(\pi(a|s)-\pi'(a|s))Q^{\pi'}(s,a)
    \end{equation}
\end{lemma}
    Applying \eqref{eq_pi_lipschitz} and Lemma \ref{lemma_aux_5}, we obtain,

    \begin{align}
        \label{eq_long_49}
        \begin{split}
            \sum_{k=1}^{K}\E\left|J(\theta_{k+1}) - J(\theta_{k})\right| &= \sum_{k=1}^{K}\E\left|\sum_{s,a}d^{\pi_{\theta_{k+1}}}(s)(\pi_{\theta_{k+1}}(a|s)-\pi_{\theta_{k}}(a|s))Q^{\pi_{\theta_{k}}}(s, a)\right|\\
            &\leq \sum_{k=1}^{K}\E\left[\sum_{s,a}d^{\pi_{\theta_{k+1}}}(s)\left|\pi_{\theta_{k+1}}(a|s)-\pi_{\theta_{k}}(a|s)\right|\left|Q^{\pi_{\theta_{k}}}(s, a)\right|\right]\\
            &\leq G\sum_{k=1}^{K}\E\left[\sum_{s,a}d^{\pi_{\theta_{k+1}}}(s)\pi_{\bar{\theta}_k}(a|s)\Vert\theta_{k+1}-\theta_{k}\Vert|Q^{\pi_{\theta_{k-1}}}(s, a)|\right]\\
            &\leq G\sum_{k=1}^{K} \E\left[\underbrace{\sum_{s,a}d^{\pi_{\theta_{k+1}}}(s)\pi_{\bar{\theta}_k}(a|s)}_{=1}\gamma_k\cdot 6t_{\mathrm{mix}}\right]\\
            &\overset{}{=} 6Gt_{\mathrm{mix}} \sum_{k=1}^{K} \|\theta_{k+1}-\theta_k\|.
        \end{split}
    \end{align}
 
   Next, recall that for any policy $\pi_{\theta}$,  $r^{\pi_{\theta}}(s) \coloneqq \sum_a\pi_{\theta}(a|s)r(s, a)$.
    Note that, for any policy parameter $\theta$, and any state $s\in\mathcal{S}$, the following holds.
    \begin{align}
        \label{eq_49}
        V^{\pi_{\theta}}(s)&=\sum_{t=0}^{\infty}\left<(P^{\pi_{\theta}})^t(s,\cdot) - d^{\pi_{\theta}}, r^{\pi_{\theta}}\right> = \sum_{t=0}^{N-1}\left<(P^{\pi_{\theta}})^t({s,\cdot}),r^{\pi_{\theta}}\right> - NJ(\theta) \nonumber\\
        &+ \sum_{t=N}^{\infty}\left<(P^{\pi_{\theta}})^t({s,\cdot}) - d^{\pi_{\theta}},r^{\pi_{\theta}}\right>.
    \end{align}

    Define the following quantity
    \begin{align}
        \label{def_error}
        \delta^{\pi_{\theta}}(s, N) \coloneqq \sum_{t=N}^{\infty}\norm{(P^{\pi_{\theta}})^t({s,\cdot}) - d^{\pi_{\theta}}}_1.
    \end{align}

    for any policy $\pi_{\theta}$ and state $s$. With $N=7t_{\mathrm{mix}}(\log_2 T)$, $\delta^{\pi_{\theta}}(s, T) \leq \frac{1}{T^6}$. Combining this result with the fact that the reward function is bounded in $[0, 1]$, we obtain,
    \begin{align}
        \label{eq_exp_diff_v}
	\begin{split}    
            &\sum_{k=1}^K\E|V^{\pi_{\theta_{k+1}}}(s) - V^{\pi_{\theta_{k}}}(s)|\\
            &\leq \sum_{k=1}^K\E\left|\sum_{t=0}^{N-1}\left<(P^{\pi_{\theta_{k+1}}})^t({s,\cdot}) - (P^{\pi_{\theta_k}})^t({s,\cdot}), r^{\pi_{\theta_{k+1}}}\right>\right| + \sum_{k=1}^K\E\left|\sum_{t=0}^{N-1}\left<(P^{\pi_{\theta_k}})^t({s,\cdot}), r^{\pi_{\theta_{k+1}}}-r^{\pi_{\theta_k}}\right>\right| \\
            &+ N\sum_{k=1}^K\E|J(\theta_{k+1}) - J(\theta_{k})| + \frac{2K}{T^5}\\
            &\overset{(a)}{\leq} \sum_{k=1}^K\sum_{t=0}^{N-1}\E\norm{ (P^{\pi_{\theta_{k+1}}})^t - (P^{\pi_{\theta_k}})^t)r^{\pi_{\theta_{k+1}}} }_{\infty} + \sum_{k=1}^K\sum_{t=0}^{N-1}\E\norm{r^{\pi_{\theta_{k+1}}}-r^{\pi_{\theta_k}}}_{\infty} \\
            &+ 6Gt_{\mathrm{mix}} \sum_{k=1}^{K}\gamma_k + \frac{2K}{T^5}
        \end{split}
    \end{align}
    where $(a)$ follows from \eqref{eq_long_49} and substituting $N=7t_{\mathrm{mix}}(\log_2 T)$. For the first term, note that,
    \begin{align}
        \label{eq_long_recursion}
        \begin{split}
            &\norm{ ((P^{\pi_{\theta_{k+1}}})^t - (P^{\pi_{\theta_k}})^t)r^{\pi_{\theta_{k+1}}} }_{\infty}\\ &\leq \norm{ P^{\pi_{\theta_{k+1}}}((P^{\pi_{\theta_{k+1}}})^{t-1} - (P^{\pi_{\theta_k}})^{t-1})r^{\pi_{\theta_{k+1}}} }_{\infty} + \norm{ (P^{\pi_{\theta_{k+1}}} - P^{\pi_{\theta_k}})(P^{\pi_{\theta_k}})^{t-1}r^{\pi_{\theta_{k+1}}} }_{\infty}\\
            &\overset{(a)}{\leq} \norm{ ((P^{\pi_{\theta_{k+1}}})^{t-1} - (P^{\pi_{\theta_k}})^{t-1})r^{\pi_{\theta_{k+1}}} }_{\infty} + \max_s\norm{P^{\pi_{\theta_{k+1}}}({s,\cdot})-P^{\pi_{\theta_k}}({s,\cdot})}_1
        \end{split}
    \end{align}

    Inequality $(a)$ holds since every row of $P^{\pi_{\theta_k}}$ sums to $1$ and $\norm{(P^{\pi_{\theta_k}})^{t-1}r^{\pi_{\theta_{k+1}}}}_{\infty}\leq 1$. Moreover, invoking \eqref{eq_pi_lipschitz}, we get,
    \begin{align*}
        \max_s\Vert P^{\pi_{\theta_{k+1}}}({s,\cdot})-P^{\pi_{\theta_k}}({s,\cdot})\Vert_1 &= \max_s\left| \sum_{s'}\sum_a(\pi_{\theta_{k+1}}(a|s)-\pi_{\theta_k}(a|s))P(s'|s, a)\right| \\
        &\leq G \norm{\theta_{k+1}-\theta_k}\max_s\left| \sum_{s'}\sum_a \pi_{\bar{\theta}_k}(a|s)P(s'|s, a)\right| \\
        &\leq G\|\theta_{k+1}-\theta_k\|.
    \end{align*}
 
    Plugging the above result into $(\ref{eq_long_recursion})$ and using a recursive argument, we get,
    \begin{align*}
        \norm{ ((P^{\pi_{\theta_{k+1}}})^t - (P^{\pi_{\theta_k}})^t)r^{\pi_{\theta_{k+1}}} }_{\infty} &\leq \sum_{t'=1}^{t} \max_s\norm{P^{\pi_{\theta_{k+1}}}({s,\cdot})-P^{\pi_{\theta_k}}({s,\cdot})}_1\\
        &\leq \sum_{t'=1}^{t} G\|\theta_{k+1}-\theta_k\|  \leq t  G\|\theta_{k+1}-\theta_k\|
    \end{align*}
    
    Finally, we have
    \begin{align}
        \label{eq_app_54}
        \begin{split}
            \sum_{k=1}^K\sum_{t=0}^{N-1} \E\norm{ ((P^{\pi_{\theta_{k+1}}})^t - (P^{\pi_{\theta_k}})^t)r^{\pi_{\theta_{k+1}}} }_{\infty} &\leq \sum_{k=1}^K\sum_{t=0}^{N-1} t G\gamma_k \leq \cO\left(GN^2\sum_{k=1}^K\|\theta_{k+1}-\theta_k\|\right).
        \end{split}
    \end{align}
 Moreover, notice that,
    \begin{equation}
        \label{eq_app_55}
        \begin{aligned}
            \sum_{k=1}^{K}\sum_{t=0}^{N-1}\E\norm{r^{\pi_{\theta_{k+1}}}-r^{\pi_{\theta_{k}}}}_{\infty}
            &\leq \sum_{k=1}^{K}\sum_{t=0}^{N-1}\E\left[\max_s\left|\sum_a(\pi_{\theta_{k+1}}(a|s)-\pi_{\theta_{k}}(a|s))r(s,a)\right|\right]\\
            &\leq GN \sum_{k=1}^{K} \|\theta_{k+1}-\theta_k\|.
        \end{aligned}
    \end{equation}  

Note that for Algorithm \ref{alg:PG_IGT_Avg}, $\|\tilde{\theta}_{k+1}-\tilde{\theta}_{k}\|\leq \cO(\gamma_{k}/\eta_k)=\cO\left(\frac{G}{\mu k^{1/5}}\right)$. Substituting in Lemma \ref{lem:stability} with the regret decomposition in \eqref{reg_decompose} yields
\begin{align*}
&\E[\mathrm{Reg}_T] =H\sum_{k=1}^{K}\left(J^*-\E[J({\tilde{\theta}_k})]\right)+ \cO\left(G^2\mu^{-1}t_{\mathrm{mix}}^2(\log^2 T) K^{4/5}\right).
\end{align*}
Whereas, for Algorithm \ref{alg:PG_Hessian_Avg}, $\|\theta_{k+1}-\theta_k\|=\gamma_k=\frac{6G}{\mu (k+2)}$. It similarly follows from the decomposition in \eqref{eq:decompose_2} that
\begin{align}
    \E[\mathrm{Reg}_T] =H\sum_{k=1}^{K}\left(J^*-\E[J({\theta_k})]\right)+ \cO\left(G^2\mu^{-1}t_{\mathrm{mix}}^2(\log^3 T) \right).
\end{align}
\end{proof}

\section{Proof of Lemma \ref{lem:last_iterate_general}}
\label{sec:last-iterate-proof}
By the $L$-smoothness of $\Bar{J}$ (where $L$ is defined in Lemma \ref{lem:grad+hess-est-prop}) and using the update rule for $\theta_k$, we get
\begin{eqnarray}
    -\Bar{J}(\theta_{k+1}) &\leq& - \Bar{J}(\theta_{k}) - \langle \nabla_{\theta} \Bar{J}(\theta_{k}), \theta_{k+1} - \theta_k \rangle + \frac{L}{2} \sqnorm{ \theta_{k+1} - \theta_k } \notag \\
    &=& - \Bar{J}(\theta_{k}) - \gamma_k \frac{\langle  \nabla_{\theta} \Bar{J}(\theta_{k}) , d_k \rangle}{\norm{d_k}} + \frac{L \gamma_k^2}{2} \notag
\end{eqnarray}
Now let us bound the second term in the above inequality. Define $\hat{e}_k = d_k -  \nabla_{\theta} \Bar{J}(\theta_{k})$. We consider two cases. First, if $\norm{\hat{e}_k} \leq \frac{1}{2} \norm{  \nabla_{\theta} \Bar{J}(\theta_{k})  }$, then 
\begin{eqnarray}
    - \frac{\langle  \nabla_{\theta} \Bar{J}(\theta_{k}), d_k \rangle}{\norm{d_k}} &=& \frac{ -\sqnorm{ \nabla_{\theta} \Bar{J}(\theta_{k})} - \langle  \nabla_{\theta} \Bar{J}(\theta_{k}), \hat{e}_k \rangle}{\norm{d_k}} \notag \\
    &\leq& \frac{ -\sqnorm{ \nabla_{\theta} \Bar{J}(\theta_{k})} + \norm{  \nabla_{\theta} \Bar{J}(\theta_{k})} \norm{ \hat{e}_k }}{\norm{d_k}} \notag \\
    &\leq& \frac{ -\sqnorm{ \nabla_{\theta} \Bar{J}(\theta_{k})} + \frac{1}{2}\sqnorm{  \nabla_{\theta} \Bar{J}(\theta_{k}) } }{\norm{d_k}} \notag \\
    &\leq& -\frac{ \sqnorm{ \nabla_{\theta} \Bar{J}(\theta_{k})} }{2 \rb{ \norm{ \nabla_{\theta} \Bar{J}(\theta_{k}) } + \norm{\hat{e}_k} }  } \notag \\
    &\leq& -\frac{ 1 }{3 } \norm{ \nabla_{\theta} \Bar{J}(\theta_{k}) }  \notag 
\end{eqnarray}
Otherwise, if $\norm{\hat{e}_k} \geq \frac{1}{2} \norm{  \nabla_{\theta} \Bar{J}(\theta_{k}) }$, we have
\begin{eqnarray}
    - \frac{\langle  \nabla_{\theta} \Bar{J}(\theta_{k}), d_k \rangle}{\norm{d_k}} &\leq & \norm{  \nabla_{\theta} \Bar{J}(\theta_{k}) } \notag \\
    & = & -\frac{ 1 }{3 } \norm{ \nabla_{\theta} \Bar{J}(\theta_{k})} + \frac{ 4 }{3 } \norm{ \nabla_{\theta} \Bar{J}(\theta_{k})} \notag \\
    & \leq & -\frac{ 1 }{3 } \norm{ \nabla_{\theta} \Bar{J}(\theta_{k})} + \frac{8}{3} \norm{\hat{e}_k} \notag
\end{eqnarray}
Combining the two cases gives
\begin{eqnarray}
    -\Bar{J}(\theta_{k+1}) &\leq&  - \Bar{J}(\theta_{k}) - \frac{ \gamma_k }{3 } \norm{ \nabla_{\theta} \Bar{J}(\theta_{k})} + \frac{ 8 \gamma_k }{3 } \norm{\hat{e}_k} + \frac{L \gamma_k^2 }{2}\label{eq:NSGD_descent_FOSP}
\end{eqnarray}
Let $e_k\coloneqq d_k-\nabla_{\theta} J(\theta_k)$. Using Lemma \ref{lem:grad+hess-est-prop}, \ref{lem:bar-J-J-dist} and the fact that $\|\theta_k-\theta_0\| \leq T$, $\forall k \leq K$, we obtain
\begin{eqnarray}
\label{eq:j-bound-mid}
    -J(\theta_{k+1})
    &\leq&  - J(\theta_{k}) - \frac{ \gamma_k }{3 } \norm{ \nabla_{\theta} J(\theta_{k})} + \frac{ 8 \gamma_k }{3 } \norm{e_k} + \frac{L \gamma_k^2 }{2} + \mathcal{O}\left(\dfrac{A Gt_{\mathrm{mix}}\log T}{T^3}\right)
\end{eqnarray}
Now adding $J^*$ and taking expectation on both sides and using Gradient Domination Lemma gives
\begin{align*}
    &J^*-\E [J(\theta_{k+1})] \\
    &\leq \E\left[J^* - J(\theta_{k}) - \frac{ \mu \gamma_k }{3 G } (J^*-J(\theta_k) -\sqrt{\epsilon_{\mathrm{bias}}}) + \frac{ 8 \gamma_k }{3 } \norm{e_k} + \frac{L \gamma_k^2 }{2}+ \mathcal{O}\left(\dfrac{A Gt_{\mathrm{mix}}\log T}{T^3}\right)\right] \\
    &=\left(1-\frac{ \mu \gamma_k }{3 G }\right) (J^* - \E[ J(\theta_{k})]) +\frac{ \mu \gamma_k }{3 G }\sqrt{\epsilon_{\mathrm{bias}}} + \frac{8\gamma_k}{3} \E \norm{e_k} + \frac{L\gamma_k^2}{2} + \mathcal{O}\left(\dfrac{A Gt_{\mathrm{mix}}\log T}{T^3}\right)
\end{align*}

We unroll the above recursion with the help of the following lemma.
\begin{lemma}\label{le:aux_rec0}
\citep[Lemma 12]{fatkhullin2023stochastic}
Let $a$ be a positive real, $\xi$ a positive integer and let $\{r_t\}_{t\geq 0}$ be a non-negative sequence satisfying for every integer~$t \geq 0$
$$
r_{t+1} - r_t \leq - a \alpha_t r_t + \nu_t,
$$
where $\cb{\alpha_t}_{t\geq 0}, \cb{\beta_t}_{t\geq 0}$ are non-negative sequences and $ a \alpha_t \leq 1$ for all $t$. Then for $\alpha_t = \frac{2 }{ a (t+\xi) }$ we have for every integers~$t_0, T \geq 1$
$$
r_T \leq \frac{ (t_0 + \xi-1)^2 r_{t_0} }{(T+\xi-1)^2} + \frac{  \sum_{t=0}^{T-1}\nu_t (t + \xi )^2 }{(T+\xi-1)^2} .
$$
\end{lemma}

Since $\gamma_k$ is chosen such that $\frac{ \mu \gamma_k }{3 G } = \frac{2}{(k+2)}$, we can invoke Lemma \ref{le:aux_rec0} to obtain
\begin{align*}
    J^*-\E [J(\theta_{K})] \leq \frac{J^*-J(\theta_{0})}{(K+1)^2} +
    \frac{\sum_{k=0}^{K-1}\nu_k(k+2)^2}{(K+1)^2},
\end{align*}
where $\nu_k \coloneqq \frac{\mu\gamma_k}{3 G }\sqrt{\epsilon_{\mathrm{bias}}} + \frac{ 8 \gamma_k }{3} \E \norm{e_k} + \frac{L \gamma_k^2 }{2}+ \mathcal{O}\left(\dfrac{A Gt_{\mathrm{mix}}\log T}{T^3}\right)$. 

\section{Proof of Theorems \ref{thm:alg1-last-iterate} and \ref{thm:alg1-regret}}
\label{sec:igt-main-theorems-proof}
We first provide our results assuming that the initial distribution for trajectories sampled at each iteration is $\rho$. Under this setting, the proof methods used are inspired by the Lemma 7 of \citet{fatkhullin2023stochastic}. However, here the gradient is biased unlike the discounted case and the resulting terms arising from the bias must be bounded. 

For ease of exposition, we let $\hat{e}_k = d_k - \nabla_{\theta} J(\theta_k)$, $e_k \coloneqq  g(\theta_k,\tau_k) -\nabla_{\theta} J(\tilde{\theta}_k)$, 

    \begin{align*}
        S_{k} \coloneqq \nabla_{\theta} J({\theta}_{k-1}) - \nabla_{\theta} J({\theta}_k) + \nabla^2 J(\theta_{k}) (\theta_{k-1} - \theta_{k})
    \end{align*}
and 
    \begin{align*}
         Z_k \coloneqq \nabla_{\theta} J(\tilde{\theta}_{k}) - \nabla_{\theta} J({\theta}_k) + \nabla^2 J(\theta_{k}) (\tilde{\theta}_k - \theta_{k}).
    \end{align*}
    We have
    \begin{align}
    \label{eq:St_bound}
        \|S_{k}\| \leq L_h \|\theta_k - \theta_{k-1}\|^2 = L_h \gamma_{k-1}^2 
    \end{align}
    and
    \begin{align}
    \label{eq:Zt_bound}
    \|Z_k\| \leq L_h \|\tilde{\theta}_k - \theta_{k}\|^2 = L_h \frac{(1-\eta_k)^2}{\eta_k^2} \|\theta_k - \theta_{k-1}\|^2  \leq L_h \frac{(1-\eta_k)^2}{\eta_k^2} \gamma_{k-1}^2.  
    \end{align}
    Using the update rule of the sequence $\{d_k\}_{k \geq 1}$ in Algorithm \ref{alg:PG_IGT_Avg}, we obtain the following recursion
\begin{align}
\label{eq:recursion_e_hat}
    \hat{e}_{k} &= d_{k} - \nabla_{\theta} J(\theta_k) \nonumber \\
    &= (1 - \eta_k) d_{k-1} + \eta_k  g(\theta_k,\tau_k) - \nabla_{\theta} J(\theta_k) \nonumber \\
    &= (1 - \eta_k) (d_{k-1}-\nabla_{\theta} J(\theta_{k-1})) + (1 - \eta_k) \nabla_{\theta} J(\theta_{k-1}) + \eta_k  g(\theta_k,\tau_k) - \nabla_{\theta} J(\theta_k) \nonumber \\
    &= (1 - \eta_k) \hat{e}_{k-1}+ (1-\eta_k) \nabla_{\theta} J(\theta_{k-1}) + \eta_k  g(\theta_k,\tau_k) - \nabla_{\theta} J(\theta_k) \nonumber \\
    &= (1 - \eta_k) \hat{e}_{k-1} + \eta_k  (g(\theta_k,\tau_k)-\nabla_{\theta} J(\Tilde{\theta}_k)) + \eta_k\nabla_{\theta} J(\Tilde{\theta}_k)+ (1-\eta_k) \nabla_{\theta} J(\theta_{k-1})- \nabla_{\theta} J(\theta_k) \nonumber \\
    &= (1 - \eta_k) \hat{e}_{k-1} + \eta_k  e_k + \eta_k\nabla_{\theta} J(\Tilde{\theta}_k)+ (1-\eta_k) \nabla_{\theta} J(\theta_{k-1})- \nabla_{\theta} J(\theta_k) \nonumber \\
    &= (1 - \eta_k) \hat{e}_{k-1} + \eta_k  e_k + \eta_k (\nabla_{\theta} J(\Tilde{\theta}_k)- \nabla_{\theta} J(\theta_k))+ (1-\eta_k) (\nabla_{\theta} J(\theta_{k-1})- \nabla_{\theta} J(\theta_k)) \nonumber \\
    &= (1 - \eta_k) \hat{e}_{k-1} + \eta_k  e_k + \eta_k (S_k - \nabla^2 J(\theta_{k}) (\theta_{k-1} - \theta_{k}))+ (1-\eta_k) (Z_k - \nabla^2 J(\theta_{k}) (\Tilde{\theta}_{k} - \theta_{k})) \nonumber \\
    &= (1 - \eta_k) \hat{e}_{k-1}  + \eta_t e_{k} + (1 - \eta_k) S_k +  \eta_t Z_k - (\eta_k \nabla^2 J(\theta_{k}) (\theta_{k-1} - \theta_{k}) \nonumber\\
    &+ (1-\eta_k)  \nabla^2 J(\theta_{k})(\Tilde{\theta}_{k} - \theta_{k})) \\
    &= (1 - \eta_k) \hat{e}_{k-1}  + \eta_t e_{k} + (1 - \eta_k) S_k +  \eta_t Z_k \nonumber ,
\end{align}
where the last term of \eqref{eq:recursion_e_hat} is $0$ since $\Tilde{\theta}_k = \theta_k + \frac{1-\eta_k}{\eta_k}(\theta_k - \theta_{k-1})$. This simplification provides the main motivation behind the update rule of $\Tilde{\theta}_k$. 

Let $\zeta_{k,K} \coloneqq \prod_{j = k}^{K-1} (1 - \eta_{j+1})$ (with $\zeta_{K,K} = 1$). Unrolling the above recursion yields
\begin{equation*}
    \hat{e}_{K} = \zeta_{0,K} \hat{e}_0 + \sum_{k = 0}^{K-1} \eta_{k+1} \zeta_{k+1,K} e_{k+1} + \sum_{k = 0}^{K-1} (1-\eta_{k+1}) \zeta_{k+1,K} S_{k+1} 
     + \sum_{k = 0}^{K-1} \eta_{k+1} \zeta_{k+1,K} Z_{k+1}.
\end{equation*}

Taking the norm and expectation on both sides give
\begin{align*}
    &\E\|\hat{e}_{K}\| \\
    &= \E \norm{\zeta_{0,K} \hat{e}_0 + \sum_{k = 0}^{K-1} \eta_{k+1} \zeta_{k+1,K} e_{k+1} + \sum_{k = 0}^{K-1} (1-\eta_{k+1}) \zeta_{k+1,K} S_{k+1} + \sum_{k = 0}^{K-1} \eta_{k+1} \zeta_{k+1,K} Z_{k+1}}\\
     &\leq \E \|\zeta_{0,K} \hat{e}_0\| + \E \norm{\sum_{k = 0}^{K-1} \eta_{k+1} \zeta_{k+1,K} e_{k+1}} + \sum_{k = 0}^{K-1} (1-\eta_{k+1}) \zeta_{k+1,K} \E\| S_{k+1} \|\\
     &+ \sum_{k = 0}^{K-1} \eta_{k+1} \zeta_{k+1,K} \E\|Z_{k+1}\|\\
     &\leq \E \|\zeta_{0,K} \hat{e}_0\| + \left(\E \norm{\sum_{k = 0}^{K-1} \eta_{k+1} \zeta_{k+1,K} e_{k+1}}^2 \right)^{1/2} + \sum_{k = 0}^{K-1} (1-\eta_{k+1}) \zeta_{k+1,K}L_h \gamma_{k-1}^2 \\
     &+ \sum_{k = 0}^{K-1} \zeta_{k+1,K} L_h \frac{(1-\eta_k)^2}{\eta_k} \gamma_{k-1}^2\\
     &\leq \E \|\zeta_{0,K} \hat{e}_0\| + \left(\E \norm{\sum_{k = 0}^{K-1} \eta_{k+1} \zeta_{k+1,K} e_{k+1}}^2 \right)^{1/2} +  2L_h \sum_{k = 0}^{K-1} \zeta_{k+1,K} \frac{\gamma_{k-1}^2}{\eta_k}.
\end{align*}

We now focus on the second term in the RHS of the above line.
\begin{align*}
    &\E \norm{\sum_{k = 0}^{K-1} \eta_{k+1} \zeta_{k+1,K} e_{k+1}}^2 \\
    &=  \sum_{k = 0}^{K-1} \E \|\eta_{k+1} \zeta_{k+1,K} e_{k+1} \|^2 + 2 \sum_{i=0}^{K-1} \sum_{j=0}^{i} \E \langle \eta_{i+1} \zeta_{i+1,K} e_{i+1} , \eta_{j+1} \zeta_{j+1,K} e_{j+1}  \rangle \\
    &=  \sum_{k = 0}^{K-1} (\eta_{k+1} \zeta_{k+1,K})^2 \E \| e_{k+1} \|^2 + 2 \sum_{i=0}^{K-1} \sum_{j=0}^{i} (\eta_{i+1} \zeta_{i+1,K} \eta_{j+1} \zeta_{j+1,K} )\E \langle e_{i+1} , e_{j+1} \rangle.
\end{align*}
 Define for every integer~$t \geq 1$ the $\sigma$-field $\mathcal{F}_{t} \coloneqq \sigma(\{\tilde{\theta}_0, \tilde{\xi}_0, \dots, \tilde{\xi}_{t-1} \})$ where~$\tilde{\xi}_s \sim p(\cdot|\pi_{\tilde{\theta}_s})$ for every~$0 \leq s \leq t-1$. Notice that for any integers~$t_2 > t_1 \geq 1$  we have 
\begin{align}
\begin{split}
|\E{ \langle e_{t_1}, e_{t_2} \rangle }| &= |\E [\E [\langle e_{t_1} ,e_{t_2}\rangle | \mathcal{F}_{t_2}] ]|  = |\E [ \langle e_{t_1}, \E [ e_{t_2}  | \mathcal{F}_{t_2} ]  \rangle]| \\
&\leq   \E [\| e_{t_1} \|\|\E [ e_{t_2}  | \mathcal{F}_{t_2} ] \|]  \leq  \sigma_g \bias.
\end{split}
\end{align}

We state some useful lemmas for bounding terms arising from the step-sizes:

\begin{lemma}\label{le:prod_bound}
\citep[Lemma 14]{fatkhullin2023stochastic}
   Let $q \in [0,1]$ and let $\eta_t = \left( \frac{2}{t+2} \right)^{q}$ for every integer~$t$. Then for every integer~$t$ and any integer~$T \geq 1$ we have
    \begin{eqnarray}\label{eq:tech1}
    \eta_t (1 - \eta_{t+1}) \leq \eta_{t+1} \text{\quad and \quad} \prod_{t = 0}^{T-1} (1 - \eta_{t+1}) \leq \eta_T.
    \end{eqnarray}
\end{lemma}

\begin{lemma}\label{le:sum_prod_bound1}
\citep[Lemma 15]{fatkhullin2023stochastic}
Let $q \in [0, 1)$, $p \geq 0$, $\gamma_0 > 0$ and let $\eta_t =  \left( \frac{2}{t+2} \right)^q$, $\gamma_t = \gamma_0 \left( \frac{2}{t+2} \right)^p$ for every integer~$t$. Then for any integers~$t$ and~$T \geq 1$, it holds 
    $$
     \sum_{t = 0}^{T-1} \gamma_{t} \prod_{\xi = t+1 }^{T-1} (1 - \eta_{\xi}) \leq C \gamma_T \eta_T^{-1},
    $$
     where $C = C(p,q) \coloneqq 2^{p-q} (1-q)^{-1}  t_0 \exp \left(  2^q (1-q)  t_0^{1-q} \right)  + 2^{ 2 p + 1 - q} (1-q)^{-2}$ and $ t_0 \coloneqq \max\left\{ \left(\frac{p}{(1-q)2^q} \right)^{\frac{1}{1-q}}, 2 \left(\frac{p-q}{(1-q)^2}\right)^{\frac{1}{1-q}}\right\} $.
\end{lemma}

We then obtain the following
\begin{align*}
    &\E \norm{\sum_{k = 0}^{K-1} \eta_{k+1} \zeta_{k+1,K} e_{k+1}}^2 \\
    &\leq  \sum_{k = 0}^{K-1} (\eta_{k+1} \zeta_{k+1,K})^2 \sigma_g^2 + 2 \sum_{i=0}^{K-1} \sum_{j=0}^{i} (\eta_{i+1} \zeta_{i+1,K} \eta_{j+1} \zeta_{j+1,K} )\cdot \sigma_g \bias \\
    &\leq \sigma_g^2 \cdot  \sum_{k = 0}^{K-1} \eta_{k+1}^2 \zeta_{k+1,K}  + 2  \sigma_g \bias \cdot \sum_{i=0}^{K-1} \sum_{j=0}^{i} \eta_K^2 \\
    &\overset{(a)}{=} \cO \left( \sigma_g^2 \cdot \eta_K + \sigma_g \beta_g \cdot K^2 \eta_K^2\right) = \cO\left(\frac{AG^2t_{\mathrm{mix}}^2(\log T)^2}{K^{4/5}}+\frac{A^2 G^2t_{\mathrm{mix}}^2(\log T)^2\cdot K^{2/5}}{T^4}\right).
\end{align*}
where $(a)$ follows from Lemmas \ref{le:prod_bound} and \ref{le:sum_prod_bound1}. 
We have
\begin{align*}
    \E\|d_K-\nabla_{\theta} J(\theta_K)\| &= \E\|\hat{e}_{K}\| \\
    &\leq \E \|\zeta_{0,K} \hat{e}_0\| + \left(\E \norm{\sum_{k = 0}^{K-1} \eta_{k+1} \zeta_{k+1,K} e_{k+1}}^2 \right)^{1/2} +  2L_h \sum_{k = 0}^{K-1} \zeta_{k+1,K} \frac{\gamma_{k-1}^2}{\eta_k} \\
    &\leq \eta_K \E \|\hat{e}_0\| + \left(\E \norm{\sum_{k = 0}^{K-1} \eta_{k+1} \zeta_{k+1,K} e_{k+1}}^2 \right)^{1/2} +  2L_h \cdot C\gamma_K^2\eta_K^{-2} \\
    &= \cO\left(\frac{A G t_{\mathrm{mix}} \log T}{K^{2/5}}+\frac{G^2 L_h}{\mu^2 K^{2/5}}\right).
\end{align*}

Note that
\begin{align*}
  \frac{\sum_{k=0}^{K-1} \frac{ \mu \gamma_k }{3 G }\cdot\sqrt{\epsilon_{\mathrm{bias}}}\cdot(k+2)^2}{(K+1)^2} &=   \frac{\sum_{k=0}^{K-1} 2\sqrt{\epsilon_{\mathrm{bias}}}\cdot(k+2)}{(K+1)^2} \\ 
  &=   \frac{(K^2+3K)\cdot\sqrt{\epsilon_{\mathrm{bias}}}}{(K+1)^2} \\
  &\leq   \sqrt{\epsilon_{\mathrm{bias}}}.
\end{align*}

Also,
\begin{align*}
  \frac{\sum_{k=0}^{K-1} \frac{L \gamma_k^2}{2} \cdot(k+2)^2}{(K+1)^2}  &= \frac{\sum_{k=0}^{K-1} \frac{L (6G)^2}{2\mu^2 (k+2)^2} \cdot(k+2)^2}{(K+1)^2} \\
  &= \frac{\sum_{k=0}^{K-1} \frac{18 G^2 L}{\mu^2}}{(K+1)^2}\\
  &\leq  \frac{18 G^2 L}{\mu^2(K+1)}.
\end{align*}

Since $\E \norm{\hat{e}_k} = \cO\left(\frac{G t_{\mathrm{mix}} \log T}{k^{2/5}}+\frac{G^2 L_h}{\mu^2 k^{2/5}}\right)$, we have
\begin{align*}
    \frac{\sum_{k=0}^{K-1} \frac{ 8 \gamma_k }{3 } \E \norm{\hat{e}_k} \cdot(k+2)^2}{(K+1)^2} &= \frac{\sum_{k=0}^{K-1} \frac{ 16 G }{\mu} \E \norm{\hat{e}_k} \cdot(k+2)}{(K+1)^2}\\
    &\leq \frac{\sum_{k=0}^{K-1} \cO\left(\frac{A G^2 t_{\mathrm{mix}} \log T \cdot k^{2/5}}{\mu} +\frac{G^3 L_h \cdot k^{2/5}}{\mu^3}\right)}{(K+1)^2}\\
    &\leq \cO\left(\frac{A G^2 t_{\mathrm{mix}} \log T \cdot K^{-2/5}}{\mu} +\frac{G^3 L_h \cdot K^{-2/5}}{\mu^3}\right).
\end{align*}
It follows from Lemma \ref{lem:last_iterate_general} that for all $K \geq 1$
\begin{align*}
    J^*-\E [J(\theta_{K})] &\leq \sqrt{\epsilon_{\mathrm{bias}}} + \cO\left(\frac{A G^2 t_{\mathrm{mix}} \log T \cdot K^{-2/5}}{\mu} +\frac{G^3 L_h \cdot K^{-2/5}}{\mu^3}\right).
\end{align*}

Note that $\tilde{\theta}_{k}-\theta_k = \frac{(1-\eta_k)\gamma_{k-1}}{\eta_{k}} \frac{d_{k-1}}{\|d_{k-1}\|}$.
Using Theorem \ref{thm:smoothness}, we obtain
\begin{align}
\label{eq:new-eq-tilde}
    &J^* - J(\tilde{\theta}_k) \nonumber \\
    &\leq (J^*-J(\theta_k)) + \cO \left( - \frac{\gamma_{k-1}}{\eta_k}\|\nabla_{\theta} J(\theta_k)\|+\frac{\gamma_{k-1}}{\eta_k}\|d_{k-1}-\nabla_{\theta}J(\theta_k)\|+\frac{L\gamma_{k-1}^2}{\eta_k^2}+\frac{AGt_{\mathrm{mix}}\log T}{T^3}\right) \nonumber\\
    &\leq (J^*-J(\theta_k)) + \cO \left( \frac{\gamma_{k-1}}{\eta_k}(\|e_{k-1}\|+\|\nabla_{\theta}J(\theta_{k-1})-\nabla_{\theta}J(\theta_k)\|)+\frac{L\gamma_{k-1}^2}{\eta_k^2}+\frac{AGt_{\mathrm{mix}}\log T}{T^3}\right) %\\
    %&\leq (J^*-J(\theta_k)) + \cO \left( \frac{L}{k^{2/5}}+\frac{AGt_{\mathrm{mix}}\log T}{T^3}\right).
\end{align}
Recall that $\E \norm{\hat{e}_k} = \cO\left(\frac{G t_{\mathrm{mix}} \log T}{k^{2/5}}+\frac{G^2 L_h}{\mu^2 k^{2/5}}\right)$. Separately, invoking Lemmas \ref{lem:grad+hess-est-prop}(b) and (c), we obtain $\|\nabla_{\theta}J(\theta_{k-1})-\nabla_{\theta}J(\theta_k)\| \leq \cO \left(L\|\theta_{k-1}-\theta_k\|+\dfrac{\sqrt{A} Gt_{\mathrm{mix}}\log T}{T^4}\right) \leq \cO \left(L\gamma_{k-1}+\dfrac{\sqrt{A} Gt_{\mathrm{mix}}\log T}{T^4}\right)$. Taking expectation in \eqref{eq:new-eq-tilde} and substituting these bounds yields
\begin{align}
    J^* - &\E [J(\tilde{\theta}_k)] \leq (J^*-\E [J(\theta_k)]) \nonumber \\
    &+\cO \left( \frac{\gamma_{k-1}G t_{\mathrm{mix}} \log T}{\eta_kk^{2/5}}+\frac{\gamma_{k-1}G^2 L_h}{\eta_k\mu^2 k^{2/5}}+ \frac{L\gamma^2_{k-1}}{\eta_k}+\dfrac{\gamma_{k-1}\sqrt{A} Gt_{\mathrm{mix}}\log T}{\eta_kT^4} + \frac{L\gamma_k^2}{\eta^2_k}+\frac{AGt_{\mathrm{mix}}\log T}{T^3}\right).
\end{align}
It follows that
\begin{align*}
    &J^*-\E [J(\tilde\theta_{K})] \\
    &\leq \sqrt{\epsilon_{\mathrm{bias}}} + \cO\left(\frac{A G^2 t_{\mathrm{mix}} \log T \cdot K^{-2/5}}{\mu} +\frac{G^3 L_h \cdot K^{-2/5}}{\mu^3}+\frac{(\sqrt{A}BGt_{\mathrm{mix}}\log T+ \sqrt{A}G^2 t_{\mathrm{mix}}^2t_{\mathrm{hit}}(\log T)^2) K^{-2/5}}{\mu}\right).
\end{align*}
%\mathcal{O}\left(AB^2t_{\mathrm{mix}}^2(\log T)^2+AG^4t_{\mathrm{mix}}^4t_{\mathrm{hit}}^2(\log T)^4\right)

It follows that
\begin{align*}
   H\sum_{K=1}^{T/H} (J^*-\E [J(\tilde{\theta}_{K})]) \leq& \quad T \sqrt{\epsilon_{\mathrm{bias}}} + \cO\Bigg(\frac{HG^2 A t_{\mathrm{mix}} \log T \cdot (T/H)^{3/5}}{\mu} +\frac{HG^3 L_h \cdot (T/H)^{3/5}}{\mu^3}\\
   &\quad +\frac{H(\sqrt{A}BGt_{\mathrm{mix}}\log T+ \sqrt{A}G^2 t_{\mathrm{mix}}^2t_{\mathrm{hit}}(\log T)^2) (T/H)^{3/5}}{\mu}\Bigg).
   %&\leq T \sqrt{\epsilon_{\mathrm{bias}}} + \cO\left(\frac{A G^2 t_{\mathrm{mix}}^{7/5} t_{\mathrm{hit}}^{2/5} (\log T)^{9/5} \cdot T^{2/3}}{\mu} +\frac{L_h G^3 t_{\mathrm{mix}}^{2/5} t_{\mathrm{hit}}^{2/5} (\log T)^{4/5} \cdot T^{2/3}}{\mu^3}\right).
\end{align*}

Ignoring constants and logarithmic terms, the above term is $\tilde\cO(H^{2/5}T^{3/5})$. Using the regret bound in Lemma \ref{lem:exp-regret}, it follows that 
\begin{align}
     \E[\mathrm{Reg}_T] \leq \cO (H^{2/5}T^{3/5}+T^{4/5}H^{-4/5}).
\end{align}

It can be seen that $H=\tilde\cO(T^{1/6})$ balances the two terms to provide a regret of order $\tilde\cO(T^{2/3})$. Thus, the regret bound follows for the constructed MDP. To convert the above results to our actual setting, where the starting state of each trajectory is determined by the previous trajectory, we use a similar argument used in Lemma \ref{lem:grad+hess-est-prop} where we obtain results under an imaginary MDP first and then translate it into the real MDP. Here, we will consider an imaginary setup, where the state distribution becomes $(P^{\pi_{\theta_{i}}})^N\rho$ at every iteration $i$ after completion of the buffer trajectory of length $N$. We also let $f(X):= \sum_{k=1}^{K} (J^* - J(\theta_k))$, where $X\coloneqq (\theta_0,\theta_1,\tau_1,\tau_2,\cdots,\tau_{K})$. 

Let $\E'$ and $\Pr'$ denote expectation and probability under this setup. Since $f(X)$ is non-negative, notice that:
 \begin{align}
        \dfrac{\E[f(X)]}{\E'[f(X)]} = \dfrac{\sum_{X} f(X){\Pr}(X)}{\sum_{X} f(X){\Pr}'(X)}\leq \max_{X}\dfrac{{\Pr}(X)}{{\Pr}'(X)}.
\end{align}

Observe that given $\theta_0$ and $\theta_1$
\begin{align}
    \dfrac{{\Pr}(X)}{{\Pr}'(X)} &= \dfrac{{\Pr}(\tau_1|\theta_0,\theta_1)\cdot {\Pr}(\tau_2|\theta_0,\theta_1,\tau_1)\cdots {\Pr}(\tau_K|\theta_0,\theta_1,\tau_1,\cdots,\tau_{K-1})}{{\Pr}'(\tau_1|\theta_0,\theta_1)\cdot {\Pr}'(\tau_2|\theta_0,\theta_1,\tau_1)\cdots {\Pr}'(\tau_K|\theta_0,\theta_1,\tau_1,\cdots,\tau_{K-1})} \\
    &\overset{(a)}{=} \dfrac{{\Pr}(\tau_1|\theta_0,\theta_1)\cdot {\Pr}(\tau_2|\theta_1,\tau_1)\cdot {\Pr}(\tau_3|\theta_2,\tau_2)\cdots {\Pr}(\tau_K|\theta_{K-1},\tau_{K-1})}{{\Pr}'(\tau_1|\theta_0,\theta_1)\cdot {\Pr}'(\tau_2|\theta_1,\tau_1)\cdot {\Pr}'(\tau_3|\theta_2,\tau_2)\cdots {\Pr}'(\tau_K|\theta_{K-1},\tau_{K-1})},
\end{align}
where $(a)$ follows the observation that $\tau_i$ is only a function of $\theta_i$ and $\tau_{i-1}$, while $\theta_i$ is completely determined by $\theta_{i-1}$ and $\tau_{i-1}$.
We have
\begin{align}
    \frac{{\Pr}(\tau_i|\theta_{i-1},\tau_{i-1})}{{\Pr}'(\tau_i|\theta_{i-1},\tau_{i-1})} = \frac{(P^{\pi_{\theta_{i}}})^N(s,s')}{(P^{\pi_{\theta_{i}}})^N\rho(s')} &\leq \max_{s'}1+\dfrac{(P^{\pi_{\theta_{i}}})^N(s,s')-(P^{\pi_{\theta_{i}}})^N\rho(s')}{(P^{\pi_{\theta_{i}}})^N\rho(s')}\\
    &\overset{(a)}{\leq} \max_{s'}1+\dfrac{2}{T^6d^{\pi_{\theta_{i}}}(s')}\leq 1+\dfrac{2t_{\mathrm{hit}}}{T^6}\leq 1+\dfrac{2}{T^5},
\end{align}
where $(a)$ follows from Lemma \ref{lem:aux-sum_N_dist} and the fact that $T \geq 2t_{\mathrm{hit}}$. We then obtain
\begin{align}
    \dfrac{{\Pr}(X)}{{\Pr}'(X)} \leq \left(1+\frac{2}{T^5}\right)^K \leq e^{\frac{2K}{T^5}} \leq e^{\frac{2}{T^4}}\leq  \left(1+\frac{4}{T^4}\right),
\end{align}
and the result follows. The same argument also holds for Algorithm \ref{alg:PG_Hessian_Avg}.

\section{Proof of Theorems \ref{thm:alg2-last-iterate} and \ref{thm:alg2-regret}}
\label{sec:Hessian-main-theorem-proofs}

Let $\cV_k \coloneqq g(\tau_k, \theta_k)  -\nabla_{\theta} \Bar{J}(\theta_{k})$ and $\cW_k \coloneqq\nabla_{\theta} \Bar{J}(\theta_{k-1}) -\nabla_{\theta} \Bar{J}(\theta_{k}) + B(\hat{\tau}_k, \hat{\theta}_{k})(\theta_k - \theta_{k-1})$ .
These quantities are defined in such a way that they are both zero-mean with bounded variance, which was also the motivating factor for our choice of $\Bar{J}$. It is easy to see that $\E [\cV_k] = 0$ (from the definition of $\Bar{J}$) and $\E [\sqnorm{\cV_k}] \leq \sigma_g^2$ (from Lemma \ref{lem:grad+hess-est-prop}). To see that $\E [\cW_k] = 0$, observe that
\begin{align}
\begin{split}
\E[\cW_k]  &= \E[\nabla_{\theta} \Bar{J}(\theta_{k-1}) -\nabla_{\theta} \Bar{J}(\theta_{k}) + B(\hat{\tau}_k, \hat{\theta}_{k})(\theta_k - \theta_{k-1}) ] \\
&=  \E[ \E [\nabla_{\theta} \Bar{J}(\theta_{k-1}) -\nabla_{\theta} \Bar{J}(\theta_{k}) + B(\hat{\tau}_k, \hat{\theta}_{k})(\theta_k - \theta_{k-1}) \mid \theta_{k-1}, \theta_{k}, \hat{\theta}_{k} ]  ] \\
&= \E [\nabla_{\theta} \Bar{J}(\theta_{k-1}) -\nabla_{\theta} \Bar{J}(\theta_{k}) + \nabla^2 \Bar{J}(\hat{\theta}_{k}) ( \theta_{k} - \theta_{k-1} )] \\
&= \E[\nabla_{\theta} \Bar{J}(\theta_{k-1}) -\nabla_{\theta} \Bar{J}(\theta_{k}) ] + \E\left[\int_0^1 \nabla^2 \Bar{J}( q \theta_{k} + (1 - q) \theta_{k-1} ) (\theta_{k} - \theta_{k-1} ) dq \right] = 0. 
\end{split}
\end{align}

The variance bound for $\cW_k$ can be obtained as: 
\begin{align}
\begin{split}
\E [\sqnorm{ \cW_k } ] &=  \E  \sqnorm{ \nabla_{\theta} \Bar{J}(\theta_{k-1}) -\nabla_{\theta} \Bar{J}(\theta_{k}) + B(\hat{\tau}_k, \hat{\theta}_{k})(\theta_k - \theta_{k-1}) }   \\
&\leq  2 \E  \sqnorm{\nabla_{\theta} \Bar{J}(\theta_{k-1}) -\nabla_{\theta} \Bar{J}(\theta_{k}) } 
+  2  \E  \sqnorm{B(\hat{\tau}_k, \hat{\theta}_{k}) (\theta_k - \theta_{k-1}) }   \\
&\overset{(a)}{\leq}  4 L^2 \gamma_{k-1}^2,  
\end{split}
\end{align}
where $(a)$ follows using the fact that $\Bar{J}$ is $L$-smooth and the bound on the Hessian estimate variance (both implied by Lemma \ref{lem:grad+hess-est-prop}(c)). 

We can now obtain a recursion for $d_k -\nabla_{\theta} \Bar{J}(\theta_k)$ using the update rule of the sequence~$(d_k)$ in terms of $\cV_k$ and $\cW_k$ introduced earlier. We have  
\begin{eqnarray}
    d_{k} -\nabla_{\theta} \Bar{J}(\theta_{k})  &= & (1 - \eta_k) \rb{ d_{k-1} +  B(\hat{\tau}_k, \hat{\theta}_{k})(\theta_k - \theta_{k-1}) } + \eta_k g(\tau_k, \theta_k)  -\nabla_{\theta} \Bar{J}(\theta_{k}) \notag \\
    &= & (1 - \eta_k)   d_{k-1}   + \eta_k g(\tau_k, \theta_k)
    + (1 - \eta_k)  B(\hat{\tau}_k, \hat{\theta}_{k})(\theta_k - \theta_{k-1}) -\nabla_{\theta} \Bar{J}(\theta_{k})  \notag \\
    &= & (1 - \eta_k)  \rb{ d_{k-1} - \nabla_{\theta} \Bar{J}(\theta_{k-1}) }   + \eta_k g(\tau_k, \theta_k) \notag \\
    && \qquad +\, (1 - \eta_k)  B(\hat{\tau}_k, \hat{\theta}_{k})(\theta_k - \theta_{k-1}) - \nabla_{\theta} \Bar{J}(\theta_{k}) + (1 - \eta_k)\nabla_{\theta} \Bar{J}(\theta_{k-1}) \notag \\
    &= & (1 - \eta_k)  \rb{ d_{k-1} - \nabla_{\theta} \Bar{J}(\theta_{k-1}) }   + \eta_k \rb{ g(\tau_k, \theta_k) - \nabla_{\theta} \Bar{J}(\theta_{k}) } \notag \\
    && \qquad +\, (1 - \eta_k)  \rb{ \nabla_{\theta} \Bar{J}(\theta_{k-1}) -  \nabla_{\theta} \Bar{J}(\theta_{k}) +  B(\hat{\tau}_k, \hat{\theta}_{k})(\theta_k - \theta_{k-1}) }   \notag \\
    &= & (1 - \eta_k) \rb{ d_{k-1} - \nabla_{\theta} \Bar{J}(\theta_{k-1}) } + \eta_k \cV_{k} + (1 - \eta_k) \cW_k    \notag.
\end{eqnarray}

With the decomposition mentioned above, we can derive a recursive upper bound on the norm of $V_k$ as follows:
\begin{align}
\begin{split}
    V_{k} &= \E \sqnorm{d_{k} -\nabla_{\theta} \Bar{J}(\theta_{k})  }  \notag \\
    & =  \E \sqnorm{ (1 - \eta_k) \rb{ d_{k-1} - \nabla_{\theta} \Bar{J}(\theta_{k-1}) } + \eta_k \cV_{k} + (1 - \eta_k) \cW_k }   \notag \\
    & \overset{(a)}{=}   (1 - \eta_k)^2 \E \sqnorm{ d_{k-1} - \nabla_{\theta} \Bar{J}(\theta_{k-1}) }   + \E  \sqnorm{ \eta_k \cV_{k} + (1 - \eta_k) \cW_k }   \notag \\
    & \leq   (1 - \eta_k)V_{k-1} + 2 \eta_k^2 \E  \sqnorm{  \cV_{k} }   + 2 \E  \sqnorm{ \cW_k }  \notag \\
     & \leq   (1 - \eta_k)V_{k-1} + 2 \eta_k^2 \sigma_g^2 + 8 \gamma_{k-1}^2 L^2 \notag 
\end{split}
\end{align}
where $(a)$ can be inferred by noticing that the conditional expectation of the random variable $\E[\eta_k \cV_k + (1-\eta_k) \cW_k|\theta_{k-1}]=0$. 
Unrolling this recursion using Lemma \ref{le:aux_rec0} with $t_0 = 0$, $a=1$ and $\xi = 2$, we obtain
\begin{align}
    V_K = \E \sqnorm{d_{K} -\nabla_{\theta} \Bar{J}(\theta_{K})  }   \leq \frac{ V_{0} }{(K+1)^2} + \frac{  \sum_{k=0}^{K-1}\nu_k (k + 2 )^2 }{(K+1)^2}, 
\end{align}
where $\nu_k = 2 \eta_k^2 \sigma_g^2 + 8 \gamma_{k-1}^2 L^2$. Since $\gamma_k = \frac{6G}{\mu(k+2)}$ and $\eta_k = \frac{2}{k+2}$, we have $\nu_k (k+2)^2 = \cO\left(\sigma_g^2 + \frac{G^2 L^2}{\mu^2}\right)$ and with this we have
\begin{align}
     \frac{  \sum_{k=0}^{K-1}\nu_k (k + 2 )^2 }{(K+1)^2} \leq \cO\left(\frac{\sigma_g^2}{K} + \frac{G^2 L^2}{K\mu^2}\right).
\end{align}

Thus, for all $K\geq 1$
\begin{align}
    \E \sqnorm{d_{K} -\nabla_{\theta} \Bar{J}(\theta_{K})  } \leq \cO\left(\frac{ \E\sqnorm{d_{0} -\nabla_{\theta} \Bar{J}(\theta_{0})  } }{K^2} +\frac{\sigma_g^2}{K} + \frac{G^2 L^2}{K\mu^2}\right).
\end{align}

It follows that
\begin{align}
\label{eq:Hessian-d-norm}
    \E \norm{d_{K} -\nabla_{\theta} \Bar{J}(\theta_{K})  } &\leq \left(\E \sqnorm{d_{K} -\nabla_{\theta} \Bar{J}(\theta_{K})  }\right)^{1/2} \nonumber \\
    &\leq \cO\left(\frac{ \E\norm{d_{0} -\nabla_{\theta} \Bar{J}(\theta_{0})}}{K} +\frac{\sigma_g}{\sqrt{K}} + \frac{G L}{\sqrt{K}\mu}\right).
\end{align}

Note that
\begin{align}
\begin{split}
    \E \norm{d_{K} - \nabla_{\theta} J(\theta_{K})  } &= \E \norm{d_{K} -\nabla_{\theta} \Bar{J}(\theta_{K}) +\nabla_{\theta} \Bar{J}(\theta_{K}) - \nabla_{\theta} J(\theta_{K})  } \\
    &\leq \E \norm{d_{K} - \nabla_{\theta} \Bar{J}(\theta_{K})} + \norm{\nabla_{\theta} \Bar{J}(\theta_{K}) - \nabla_{\theta} J(\theta_{K})  } \\
     &= \E \norm{d_{K} - \nabla_{\theta} \Bar{J}(\theta_{K})} + \norm{ \E[g(\theta_{K},\tau)] - \nabla_{\theta} J(\theta_{K})  } \\
     &\overset{(a)}{\leq} \cO\left(\frac{\sigma_g}{\sqrt{K}} + \frac{G L}{\sqrt{K}\mu}+\frac{AGt_{\mathrm{mix}}\log T}{T^4}\right), 
\end{split}
\end{align}
where $(a)$ follows from \eqref{eq:Hessian-d-norm} and Lemma \ref{lem:grad+hess-est-prop}. 

From Lemma \ref{lem:last_iterate_general}, we have for every integer $K \geq 1$ : 
\begin{align*}
    J^*-\E [J(\theta_{K})] \leq \frac{J^*-J(\theta_{0})}{(K+1)^2} +
    \frac{\sum_{k=1}^{K}\nu_k(k+2)^2}{(K+1)^2},
\end{align*}
where $\nu_k \coloneqq \frac{ \mu \gamma_k }{3 G }\cdot\sqrt{\epsilon_{\mathrm{bias}}} + \frac{ 8 \gamma_k }{3 } \E \norm{d_k - \nabla_{\theta} J(\theta_k)} + \frac{L \gamma_k^2 }{2}$.

Observe that since $\gamma_k = \frac{6G}{\mu (k+2)}$ we have
\begin{align}
    &\frac{\sum_{k=1}^{K}\frac{ 8 \gamma_k }{3 } \E \norm{d_k - \nabla_{\theta} J(\theta_k)} (k+2)^2}{(K+1)^2} = \cO\left(\frac{G \sigma_g}{\mu \sqrt{K}} + \frac{G^2 L}{\mu^2\sqrt{K}}+\frac{AG^2t_{\mathrm{mix}}\log T}{\mu T^4}\right).
\end{align}
It follows that
\begin{align}
\label{eq:middle-bound-hess}
     J^*-\E [J(\theta_{K})] \leq \cO\left(\frac{G \sigma_g}{\mu \sqrt{K}} + \frac{G^2 L}{\mu^2\sqrt{K}}+\frac{AG^2t_{\mathrm{mix}}\log T}{\mu T^4}\right).
\end{align}

\textbf{Proof of Equation \eqref{eq:decompose_2}:} In order to get obtain the final regret decomposition, we show that
\begin{align}
\label{eq:reg-mid-part-hess}
    H\sum_{k=1}^{K}\left(J^*-J({\hat{\theta}_k})\right)= \cO\left(H\sum_{k=1}^{K}\left(J^*-J({\theta_k})\right)\right),
\end{align}
while the rest of the proof follows from Section \ref{sec:regret-decomp-proof}. From the $L$-smoothness of $\Bar{J}$ , we have
\begin{align}
\label{eq:mid-ra-Hessian}
   -\Bar{J}(\hat{\theta}_{k+1}) \leq - \Bar{J}(\theta_{k}) - \langle \nabla_{\theta} \Bar{J}(\theta_{k}), \hat{\theta}_{k+1} - \theta_k \rangle + \frac{L}{2} \sqnorm{ \hat{\theta}_{k+1} - \theta_k }. 
\end{align}
Note that $\hat{\theta}_{k+1} = q_{k+1} \theta_{k+1} + (1-q_{k+1}) \theta_{k}$, which implies $ \hat{\theta}_{k+1} - \theta_k = q_{k+1}(\theta_{k+1}-\theta_k)$. Substituting this in \eqref{eq:mid-ra-Hessian}, adding $J^*$ to both sides and taking expectation conditioned on $\theta_k$ and $\theta_{k+1}$ yields
\begin{align*}
   \E[J^*-\Bar{J}(\hat{\theta}_{k+1})\mid\theta_k,\theta_{k+1}] \leq J^*-  \Bar{J}(\theta_{k}) - \frac{1}{2}\langle \nabla_{\theta} \Bar{J}(\theta_{k}), \theta_{k+1}-\theta_k \rangle + \frac{L}{6} \sqnorm{ \theta_{k+1}-\theta_k}. 
\end{align*}

Utilizing arguments similar to \eqref{eq:j-bound-mid}, we obtain
\begin{align}
   J^*-\E J(\hat{\theta}_{k+1})
    &\leq  J^* - \E J(\theta_{k}) - \frac{ \gamma_k }{3 } \E \norm{ \nabla_{\theta} J(\theta_{k})} + \frac{ 8 \gamma_k }{3 } \E \norm{\hat{e}_k} + \frac{L \gamma_k^2 }{2} \\
    &+ \mathcal{O}\left(\dfrac{A Gt_{\mathrm{mix}}\log T}{T^3}\right)\,.
\end{align}
Taking into account the above bound and the inequality $\|\theta_{k+1}-\hat{\theta}_k\| \leq \|\theta_{k+1}-{\theta}_k\|+\| \theta_{k}-\hat{\theta}_k\| \leq 2\gamma_k$, and by replacing the bounds for $\sigma_g$ and $L$ in \eqref{eq:middle-bound-hess}, we derive the subsequent bound for the expected regret of Algorithm \ref{alg:PG_Hessian_Avg}: 
 \begin{align}
   \E[\mathrm{Reg}_T] \leq T \sqrt{\epsilon_{\mathrm{bias}}} &+ \cO\bigg(\frac{\sqrt{A} G^2t_{\mathrm{mix}}\log T }{\mu}\cdot \sqrt{T} + \frac{\sqrt{A}G^4 t_{\mathrm{hit}}t^2_{\mathrm{mix}} (\log T)^{3/2}}{\mu^2}\cdot \sqrt{T} \nonumber\\
   &+\frac{\sqrt{A}(BG+G^3) t_{\mathrm{mix}}\log T}{\mu^2}\cdot\sqrt{T}\bigg).
    \end{align}

\section{Auxillary lemmas}

\begin{lemma}
    \label{lemma_aux_bound_advantage}
    \cite[Lemma 14]{wei2020model} For ergodic MDPs with mixing time $t_{\mathrm{mix}}$, the following holds $\forall (s, a)\in\mathcal{S}\times \mathcal{A}$, any policy $\pi$ and $\forall g\in\{r, c\}$.
    \begin{align*}
        (a) |V_g^{\pi}(s)|\leq 5 t_{\mathrm{mix}},~~
        (b) |Q_g^{\pi}(s, a)|\leq 6 t_{\mathrm{mix}}
    \end{align*}
\end{lemma}

\begin{lemma}
    \label{lemma_aux_mixing_dist}
    \citep[Corollary 13.1]{wei2020model} For an ergodic MDP with mixing time $t_{\mathrm{mix}}$, we have 
    \begin{align}
        \label{def_error_aux_1}
        \|(P^{\pi})^t(s,\cdot) - d^{\pi}\|_1 \leq 2\cdot 2^{-t/t_{\mathrm{mix}}},
    \end{align}
    
   for all $\pi$, $s \in \cS$ and $t \geq 2t_{\mathrm{mix}}$.
\end{lemma}

\begin{lemma}
\label{lem:aux-sum_N_dist}
Let $N = 7t_{\mathrm{mix}}\log_2 T$. For an ergodic MDP with mixing time $t_{\mathrm{mix}} < T/4$, define 
    \begin{align}
        \label{def_error_1}
        \delta^{\pi}(s, T) \coloneqq \sum_{t=N}^{\infty}\norm{(P^{\pi})^t({s,\cdot}) - d^{\pi}}_1.
    \end{align}
We have the following bound $\forall \pi \in \Pi, s \in \cS$.
\begin{align}
    \delta^{\pi}(s, T) \leq  \frac{1}{T^6}
\end{align}
\end{lemma}
\begin{proof}
         From Lemma \ref{lemma_aux_mixing_dist}, we have 
    \begin{align}
        \delta^{\pi}(s, T) \leq \sum_{t=N}^{\infty} 2\cdot 2^{-t/t_{\mathrm{mix}}} \leq \frac{2\cdot 2^{-N/t_{\mathrm{mix}}}}{1- 2^{-1/t_{\mathrm{mix}}}} \leq \frac{4t_{\mathrm{mix}}}{\ln 2}\cdot2^{-N/t_{\mathrm{mix}}} \leq \frac{1}{T^6}.
    \end{align}
    The second inequality follows from the fact that $1-2^{-x}\geq \frac{1}{2}x\ln 2 $, $\forall x\in(0, 1)$.
\end{proof}

\begin{lemma}
    \label{lemma_aux_4}
    \citep[Lemma 16]{wei2020model} Let $\mathcal{I}=\{t_1+1,t_1+2,\cdots,t_2\}$ be a certain period of an epoch $k$ of Algorithm \ref{alg:estA} with length $N$. Then for any $s$, the probability that the algorithm never visits $s$ in $\mathcal{I}$ is upper bounded by
    \begin{equation}
        \left(1-\frac{3d^{\pi_{\theta_k}}(s)}{4}\right)^{\left\lfloor\frac{\lfloor \mathcal{I}\rfloor}{N}\right\rfloor}
    \end{equation}
\end{lemma}

\begin{lemma}[Lemma 4, \citep{bai2023regret}]
    \label{lem_performance_diff}
    The difference in the performance for policies $\pi$ and $\pi'$ is 
    \begin{equation}
        J^{\pi}-J^{\pi'}= \E_{s\sim d^{\pi}}\E_{a\sim\pi(\cdot\vert s)}\big[A^{\pi'}(s,a)\big]
    \end{equation}
\end{lemma}

\begin{lemma}[Gradient domination lemma]
 \label{lem:gradient_domination}
 Let Assumptions \ref{assump:score_func_bounds}, \ref{assump:function_approx_error}, and \ref{assump:FND_policy} hold. The following inequality holds for any $\theta \in \Theta$.
 \begin{align}
        J^*-J(\theta) &\leq \sqrt{\epsilon_{\mathrm{bias}}}+\frac{G}{\mu}\|\nabla_{\theta} J(\theta)\|
\end{align}
\end{lemma}
\begin{proof}
From Lemma \ref{lem_performance_diff}, we arrive at
\begin{align}
\label{eq:gd_avg_1}
    J^*-J(\theta)= \E_{s\sim d^{\pi^*}}\E_{a\sim\pi^*(\cdot\vert s)}\big[A^{\pi_{\theta}}(s,a)\big]
\end{align}
Moreover, we obtain the following from Assumption \ref{assump:function_approx_error}.
\begin{align}
\label{eq:gd_avg_2}
    \epsilon_{\mathrm{bias}} &\geq \E_{s\sim d^{\pi^*}}\E_{a\sim\pi^*(\cdot\vert s)}\bigg[\bigg(A^{\pi_\theta}(s,a)-\nabla_\theta\log\pi_{\theta}(a\vert s)\cdot\omega^*_{\theta}\bigg)^2\bigg] \nonumber\\
    &\geq \bigg(\E_{s\sim d^{\pi^*}}\E_{a\sim\pi^*(\cdot\vert s)}\bigg[A^{\pi_\theta}(s,a)-\nabla_\theta\log\pi_{\theta}(a\vert s)\cdot\omega^*_{\theta}\bigg]\bigg)^2
\end{align}
From \eqref{eq:gd_avg_1} and \eqref{eq:gd_avg_2}, one can deduce
\begin{align}
        \sqrt{\epsilon_{\mathrm{bias}}} &\geq \E_{s\sim d^{\pi^*}}\E_{a\sim\pi^*(\cdot\vert s)}\bigg[A^{\pi_\theta}(s,a)-\nabla_\theta\log\pi_{\theta}(a\vert s)\cdot\omega^*_{\theta}\bigg] \nonumber\\
        &= (J^*-J(\theta))-\E_{s\sim d^{\pi^*}}\E_{a\sim\pi^*(\cdot\vert s)}\bigg[\nabla_\theta\log\pi_{\theta}(a\vert s)\cdot\omega^*_{\theta}\bigg]
\end{align}

Rearranging the above inequality yields
\begin{align}
        J^*-J(\theta) &\leq \sqrt{\epsilon_{\mathrm{bias}}}+\E_{s\sim d^{\pi^*}}\E_{a\sim\pi^*(\cdot\vert s)}\bigg[\nabla_\theta\log\pi_{\theta}(a\vert s)\cdot\omega^*_{\theta}\bigg]\nonumber \\
        &\leq \sqrt{\epsilon_{\mathrm{bias}}}+\E_{s\sim d^{\pi^*}}\E_{a\sim\pi^*(\cdot\vert s)}\bigg[\|\nabla_\theta\log\pi_{\theta}(a\vert s)\| \|\omega^*_{\theta}\|\bigg] \nonumber\\
        &\leq \sqrt{\epsilon_{\mathrm{bias}}}+G\|\omega^*_{\theta}\|
\end{align}
where the last inequality follows from Assumption \ref{assump:score_func_bounds}. Note that
\begin{align}
    \|\omega^*_{\theta}\| &= \|F(\theta)^{\dagger} \nabla_{\theta} J(\theta)\| \overset{(a)}{\leq} \mu^{-1}\|\nabla_{\theta} J(\theta)\|\nonumber
\end{align}
where $(a)$ is a consequence of Assumption \ref{assump:FND_policy}. It follows that
\begin{align}
        J^*-J(\theta) &\leq \sqrt{\epsilon_{\mathrm{bias}}}+\frac{G}{\mu}\|\nabla_{\theta} J(\theta)\|
\end{align}   
This concludes the proof of Lemma \ref{lem:gradient_domination}.
\end{proof}

\end{document}